\documentclass[twoside]{article}
\usepackage{etoc}

\usepackage[square,sort,comma,numbers]{natbib}
\usepackage[accepted]{aistats2022}
\usepackage{appendix}

\usepackage[utf8]{inputenc} 
\usepackage[T1]{fontenc}    
\usepackage{hyperref}       
\hypersetup{colorlinks=true, linkcolor=blue, citecolor=blue}
\usepackage{url}            
\usepackage{booktabs}       
\usepackage{amsfonts}       
\usepackage{nicefrac}       
\usepackage{microtype}      
\usepackage{wrapfig}
\usepackage{subfig}
\usepackage{graphicx}
\usepackage{caption}
\usepackage{subfig}
\usepackage{mwe}
\usepackage{algorithmic}
\usepackage{algorithm}
\usepackage{bm}
\usepackage{algorithm}
\usepackage{algorithmic}
\usepackage[english]{babel}
\usepackage{amsmath,amsthm,amssymb}
\usepackage{mathtools}
\usepackage{bbm}
\usepackage{xcolor}
\usepackage{amsmath}
\usepackage{xfrac}

\usepackage{framed, color}
\definecolor{shadecolor}{gray}{.95}
\usepackage{multirow}

\newcommand{\nocontentsline}[3]{}
\newcommand{\tocless}[2]{\bgroup\let\addcontentsline=\nocontentsline#1{#2}\egroup}

\newtheorem{theorem}{Theorem}
\newtheorem{assumption}{Assumption}
\newtheorem{lemma}{Lemma}

\begin{document}

\twocolumn[

\aistatstitle{High-Dimensional Bayesian Optimisation with Variational Autoencoders and Deep Metric Learning}

%


  

\aistatsauthor{%
  Antoine Grosnit*
  \\
  Huawei Noah's Ark Lab
\And
  Alexandre Max Maraval*
  \\
  Huawei Noah's Ark Lab
  \And
  Rasul Tutunov*
  \\
  Huawei Noah's Ark Lab
  \AND
  Ryan-Rhys Griffiths* \\
  Huawei Noah's Ark Lab,\\
  University of Cambridge
  \And
  Alexander I.~Cowen-Rivers\\
  Huawei Noah's Ark Lab\\
  Technische Universität Darmstadt\\
  \And
  Lin Yang\\
  Huawei Noah's Ark Lab
  \AND
  Lin Zhu\\
  Huawei Noah's Ark Lab
  \And
  Wenlong Lyu\\ 
  Huawei Noah's Ark Lab
  \And
  Zhitang Chen\\
  Huawei Noah's Ark Lab
  \AND
  Jun Wang\\
  Huawei Noah's Ark Lab\\
  University College London\\
  \And
  Jan Peters\\
  Technische Universität Darmstadt\\
  \And
  Haitham Bou Ammar\\
  Huawei Noah's Ark Lab\\
  University College London
}

\aistatsaddress{
}
]

\begin{abstract}
We introduce a method combining variational autoencoders (VAEs) and deep metric learning to perform Bayesian optimisation (BO) over high-dimensional and structured input spaces. By adapting ideas from deep metric learning, we use label guidance from the black-box function to structure the VAE latent space, facilitating the Gaussian process fit and yielding improved BO performance. Importantly for BO problem settings, our method operates in semi-supervised regimes where only few labelled data points are available. We run experiments on three real-world tasks, achieving state-of-the-art results on the penalised logP molecule generation benchmark using just 3\% of the labelled data required by previous approaches. As a theoretical contribution, we present a proof of vanishing regret for VAE BO.
\end{abstract}

\etocsettocdepth.toc {chapter} 

\section{Introduction}
While Bayesian optimisation (BO) is a promising solution method for black-box optimisation problems~\citep{2016_Shahriari, 2018_Frazier, 2021_Turner}, scaling the approach to high-dimensional settings has proved challenging. Variational autoencoders (VAEs) have emerged as a powerful scaling strategy based on learning low-dimensional, nonlinear manifolds on which to perform BO \citep{2018_Gomez, 2020_Griffiths, 2018_Lu, 2018_Eissman, 2020_Moriconi, 2020_Siivola, 2020_Antonova}. VAE-based approaches are particularly suited to structured (i.e. graphs, strings or images) input spaces whereby projection to the VAE latent space enables continuous optimisation. Indeed, structured input spaces encompass a broadening spectrum of real-world tasks including, but not limited to, molecule generation \citep{2020_Korovina}, chemical reaction optimisation \citep{2021_Shields}, human motion prediction \citep{2019_Wei, 2020_Bourached} and neural architecture search \citep{2018_Kandasamy, 2019_Zhou, 2020_Ru}. 


The outstanding question for VAE BO however, is how best to leverage the black-box function in learning the latent space.  The first approaches to use VAEs for BO learned the VAE in a purely unsupervised fashion \citep{2018_Gomez, 2020_Griffiths} giving rise to pathological behaviour such as invalid decoder outputs. Purely unsupervised learning of the VAE entails that the learned latent space is not \textit{discriminative} \citep{2007_Urtasun} in the sense that it is not constructed using the black-box function labels. Such a strategy has long been noted to be sub-optimal for discriminative tasks in autoencoders \citep{2012_Snoek} and hence by analogy will be sub-optimal for VAE BO. As such, recent approaches \citep{2018_Eissman, 2020_Siivola, 2020_Tripp} have utilised ideas based on label guidance \citep{2007_Urtasun, 2012_Snoek} to construct discriminative VAE latent spaces that are more amenable to BO. 

Label guidance approaches may be categorised according to how the VAE and the surrogate model are trained. Joint training facilitates label guidance by propagating signal from the black-box function through the Gaussian process (GP) surrogate to the weights of the VAE networks. Joint training has been found to exhibit overfitting on real-world problems however \citep{2020_Siivola}. The leading approach to affect label guidance in disjoint training \citep{2020_Tripp} utilises a weighted retraining mechanism, assigning more influence to regions of the latent space with favourable black-box function values in subsequent retrainings of the VAE. Nevertheless, this approach may not produce an optimally discriminative latent space because latent points are not grouped according to their function values, thus hindering the GP fit.




In this paper we take a new approach to construct discriminative latent spaces for VAE BO using ideas from deep metric learning (DML) \citep{2002_Xing}. Metric learning has been observed to improve generalisation performance in discriminative tasks when applied as a preprocessing step \citep{2020_Khosla} and additionally, metric learning encourages points with similar function values to be close in latent space. Mechanistically, we integrate DML into the VAE by including metric loss terms (e.g. contrastive \citep{2020_Khosla} or triplet \citep{hoffer2015, wu2019}) in the evidence lower bound (ELBO). To achieve synergy with the downstream task of BO, we argue that these losses should be smooth and continuous. We interpret the losses variationally through weighted likelihoods, yielding a new ELBO through which previous approaches are recovered as special cases. Our contributions may be summarised as:








\begin{enumerate}
    \item A demonstration that DML structures the VAE latent space according to function values, yielding improved performance in downstream BO tasks.
    \item A demonstration that such a DML scheme is operational in the semi-supervised setting where state-of-the-art performance is achieved on the penalised logP molecule generation benchmark using 3\% of the labelled data required by previous approaches.
    \item A proof of sublinear regret for VAE BO.
\end{enumerate}

Additionally, we open-source all models\footnote{
\url{https://github.com/huawei-noah/HEBO/tree/master/T-LBO}
}.

\section{Background}
\label{backgr}


\subsection{Bayesian Optimisation}
In this paper, we wish to solve the optimisation problem
\begin{equation}
    \label{Eq:ProbOne}
    \bm{x}^{\star} = \arg\max_{\bm{x} \in \mathcal{X}} f(\bm{x}),
\end{equation}
where $f(\cdot):\mathcal{X} \rightarrow\mathbb{R}$ is an expensive black-box function over a high-dimensional and structured input domain $\mathcal{X}$. {BO}~\citep{1964_Kushner, 1975_Mockus, 1975_Zhilinskas, jones1998, 2010_Brochu, 2020_Grosnit} is a data-efficient methodology for determining $\bm{x}^{\star}$. There are two core components in BO; a surrogate model and an acquisition function. GPs~\citep{rasmussen2006} are the surrogate model of choice for $f(\cdot)$ as they maintain calibrated uncertainties to guide exploration. The acquisition function is responsible for suggesting new input points $\bm{x}$ to query at each iteration of BO and is designed to trade off exploration and exploitation in the black-box objective. Upon completion of each iteration, the queried points are appended to the dataset of the surrogate model which is then retrained. This process continues \textit{ad libitum} until a solution is obtained. In this paper we use the expected improvement (EI) \citep{1975_Mockus, jones1998} acquisition to facilitate comparison against recent approaches to VAE BO \citep{2020_Tripp} although we note that in general our framework is agnostic to the choice of acquisition. 


\subsection{High-Dimensional BO with VAEs}\label{Sec:LBO}
Although many disparate attempts have been made to extend BO to high dimensions cf. Section \ref{Sec:RelatedWork}, in this paper we focus on VAE-based approaches (also known as latent space optimisation (LSO) \cite{2020_Tripp}). The VAE is used to map between $\mathcal{X}$, a structured input space (e.g. graphs) and $\mathcal{Z}\subseteq\mathbb{R}^d$, a low-dimensional latent space. The model's encoder $q_{\bm{\phi}}(\cdot|\bm{x}): \mathcal{X} \rightarrow \mathcal{P}(\mathcal{Z})$ induces a probability distribution over $\mathcal{Z}$ conditioned on $\bm{x} \in \mathcal{X}$, while the decoder $g_{\bm{\theta}}(\cdot|\bm{z}): \mathcal{Z} \rightarrow \mathcal{P}(\mathcal{X})$ is a stochastic inverse map from $\mathcal{Z}$ to $\mathcal{X}$. The weights $\bm{\phi}$ and $\bm{\theta}$ are obtained by maximising the ELBO which contains a reconstruction error term and a regularisation term that encourages the approximate posterior to be close to the prior $p(\bm{z})$, i.e. $
\textbf{ELBO}(\bm{\theta}, \bm{\phi}) =\mathbb{E}_{q_{\bm{\phi}}(\bm{z}|\bm{x})}[\log g_{\bm{\theta}}(\bm{x}|\bm{z})] - \text{KL}(q_{\bm{\phi}}(\bm{z}|\bm{x})||p(\bm{z}))$. The VAE is typically pre-trained using a set of unlabelled data.

The problem formulation of VAE BO bears notable differences to standard BO. We seek to determine $\bm{z}^{\star}$ such that the expected function value evaluated on $x^{\star} \sim g_{\bm{\theta}^{\star}}(\cdot|\bm{z}^{\star})$ is maximised i.e. $\arg\max_{\bm{z}\in\mathcal{Z}} \mathbb{E}_{\bm{x}\sim g_{\bm{\theta}^{\star}}(\cdot|\bm{z})}[f(\bm{x})]$. As such, we assume that the trained decoder possesses support over $\bm{x}^{\star}$ i.e. $\exists\bm{z}\in\mathcal{Z}: \mathbb{P}\text{r}\left[\bm{x}^*\in g_{\bm{\theta}^*}(\cdot|\bm{z})\right] > 0$, an assumption that we verify empirically in Section~\ref{Sec:Theory}. This formulation may be regarded as a generalisation of standard BO, whereby we aim to acquire an optimal conditional distribution from which $\bm{x}^{\star}$ is sampled. Thus, when $g_{\bm{\theta}^{\star}}(\cdot|\bm{z})$ follows a Dirac distribution, one recovers the solution of the optimisation problem in Equation~\ref{Eq:ProbOne}. Given that the input is stochastic, $\bm{z} \sim q_{\bm{\phi}^{\star}}(\cdot | \bm{x})$, the surrogate may be viewed as a Gaussian process latent variable model (GPLVM)~\citep{2020_Siivola}.


\textbf{Label Guidance in Latent Space:} BO solves a regression subproblem in $\mathcal{Z}$. To be informative for regression tasks, $\mathcal{Z}$ can be constructed using the black-box function labels. Inspired by the finding that mild supervision can be beneficial when initialising discriminative deep networks~\citep{2007_Urtasun, 2006_Bengio}, a plethora of models have been proposed which facilitate label guidance by jointly training GPLVMs together with the autoencoder~\citep{2018_Eissman, 2020_Siivola, 2012_Snoek, moriconi2020highdimensional, liu2019}. Though successful in isolated instances, the recent findings of~\citep{2020_Siivola} suggest that \emph{disjoint training with label guidance} is preferable to avoid overfitting, yielding improved BO performance. As such, we follow the disjoint training approach detailed in~\citep{2020_Tripp} which has enjoyed success in solving a range of high-dimensional optimisation problems over structured input spaces. The technique of \citep{2020_Tripp} couples BO with the VAE through a weighted retraining scheme based on ranking evaluated function values. 


\subsection{Deep Metric Learning}\label{Sec:DML}


The goal of deep metric learning (DML) may be loosely stated as the identification and extraction of good features for downstream tasks~\cite{2020_Elezi}. Metric learning is termed “deep” when used in conjunction with deep networks which in the case of the VAE constitute the encoder-decoder networks. In this paper we wish to use deep metric learning to construct discriminative latent spaces for VAE BO. In our experiments we use a variety of metric losses which we detail in Section~\ref{Sec:Exps}. We introduce one of the most widely-used metric losses, the triplet loss~\cite{hoffer2015, wu2019} here to serve as \emph{our running example}. 




\textbf{Triplet Loss:} The triplet loss $\mathcal{L}_{\text{triple}}(\cdot)$,  frequently encountered in classification settings, measures distances between input triplets. To define $\mathcal{L}_{\text{triple}}(\cdot)$, an anchor/base input (e.g., an image of a dog) $\bm{x}^{(\text{b})}$, a positive input (e.g., a rotated image of a dog) $\bm{x}^{(\text{p})}$ and a negative input (e.g., an image of a cat) $\bm{x}^{(\text{n})}$ are required. The aim of this loss is to minimise the distance between the anchor and the positive point while maximising the distance between the anchor and the negative point. More precisely, given a separation margin $\rho$, the triplet is encoded to $\bm{z}^{(\text{b})}$, $\bm{z}^{(\text{p})}$ and $\bm{z}^{(\text{n})}$ such that: $||\bm{z}^{(\text{b})} - \bm{z}^{(\text{p})} ||_{q} + \rho\leq ||\bm{z}^{(\text{b})} - \bm{z}^{(\text{n})} ||_{q}$. Consequently, minimising  $\mathcal{L}_{\text{triple}}(\cdot)= \max\left\{0, ||\bm{z}^{(\text{b})} - \bm{z}^{(\text{p})} ||_{q} + \rho - ||\bm{z}^{(\text{b})} - \bm{z}^{(\text{n})} ||_{q} \right\}$ yields a structured space where positive and negative pairs cluster together subject to separation by a margin $\rho$. 

\section{High-Dimensional BO with VAEs and DML}\label{Sec:Sol}


Deep metric learning has been shown to be highly effective in constructing discriminative features for downstream classification tasks in computer vision \citep{2020_Khosla} and natural language processing \citep{fang2020cert}. These successes point towards deep metric learning being a promising approach for affecting discriminative latent spaces in VAE BO. Deep metric learning and VAEs are typically combined by including an additional metric loss term in the ELBO of the VAE \citep{2018_Ishfaq, 2021_Koge}. In Section \ref{Subsec:continuous-DML} and Section \ref{Sec:NewVAE} we discuss two design considerations that are unique to metric learning applied to VAE BO: 1) continuity and smoothness and 2) sample efficiency. In Section \ref{Sec:Theory} we present a proof of sublinear regret for the VAE BO scheme.

\subsection{Smooth \& Continuous Metric Losses}\label{Subsec:continuous-DML}



From Section~\ref{Sec:DML}, we note that the triplet loss, $\mathcal{L}_{\text{triple}}(\cdot)$, requires a triplet coupling constructed using label information. To extend this idea beyond classification, we introduce a threshold $\eta$ and execute triplet matching based on differences in black-box function values such that for a base input $\bm{x}^{\text{(b)}}$, we create the relative set of positive points $\mathcal{D}_{\text{p}}(\bm{x}^{\text{(b)}}; \eta)=\langle\bm{x} \in \mathcal{D}: |f(\bm{x}^{(\text{b})}) - f(\bm{x})| < \eta\rangle$ and the relative set of negative points $\mathcal{D}_{\text{n}}(\bm{x}^{\text{(b)}}; \eta)=\langle\bm{x} \in \mathcal{D}: |f(\bm{x}^{(\text{b})}) - f(\bm{x})| \geq \eta\rangle$. At this stage, we can apply  $\mathcal{L}_{\text{triple}}(\cdot)$ during the training phase of the VAE to induce a useful metric in $\mathcal{Z}$. 

In doing so, however, we observed that using $\mathcal{L}_{\text{triple}}(\cdot)$ as is yielded unstable behaviour attributed to an absence of differentiablity across the domain of valid triplets. This problem can be remedied by applying a soft-plus smooth approximation to the $\text{ReLU}(\cdot)$ leading to\footnote{We also set $\rho = 0$ as it has been observed to lead to faster convergence by sampling semi-hard triplets~\citep{schroff2015}.}: 
\begin{equation}
\label{Eq:triple}
    \mathcal{L}_{\text{triple}}^{(\text{BO})}(\cdot) \propto \log(1+\exp(\Delta_{\bm{z}}^{+} - \Delta_{\bm{z}}^{-})) 
\end{equation}
with $\Delta_{\bm{z}}^{+} = ||\bm{z}^{(\text{b})} - \bm{z}^{(\text{p})}||_{q}$ and $\Delta_{\bm{z}}^{-} = ||\bm{z}^{(\text{b})} - \bm{z}^{(\text{n})}||_{q}$, such that $\bm{z}^{(\text{p})} \sim q_{\bm{\phi}}(\cdot|\bm{x}^{(\text{p})})$ and $\bm{z}^{(\text{n})} \sim q_{\bm{\phi}}(\cdot|\bm{x}^{(\text{n})}) \ \forall \bm{x}^{(\text{p})}\in \mathcal{D}_{\text{p}}(\bm{x}^{\text{(b)}};\eta)$ and $\forall \bm{x}^{(\text{n})}\in \mathcal{D}_{\text{n}}(\bm{x}^{\text{(b)}};\eta)$.

\textbf{Softening the Triplet Loss:} Although Equation~\ref{Eq:triple} facilitates the application of metric learning in BO, it is important to note that penalisation \emph{magnitudes} are independent of the black-box function values; see Figure~\ref{Fig:SoftTriplet} in Appendix~\ref{App: Triplet_loss_section}. Such a factor can influence feature discrimination when used in conjunction with GP regressors since VAEs are not \emph{directly} ensuring an increase in similarity of function values as $\bm{z} \rightarrow \bm{z}^{\prime}$. Hence, we introduce a simple yet effective modification to $\mathcal{L}_{\text{triple}}^{(\text{BO})}(\cdot)$ that incorporates weightings for positive and negative pairs $w^{(\text{p})} \propto \eta - |f(\bm{x}^{(\text{b})}) - f(\bm{x}^{(\text{p})})|$ and $w^{(\text{n})} \propto |f(\bm{x}^{(\text{b})}) - f(\bm{x}^{(\text{n})})| - \eta$ leading us to $\mathcal{L}_{\text{s-triple}}^{(\text{BO})}(\cdot)\propto w^{(\text{p})}w^{(\text{n})} \mathcal{L}_{\text{triple}}^{(\text{BO})}(\cdot)$. Clearly, $w^{(\text{p})}$ increases penalisation magnitudes as function value differences decrease; encouraging closer latent points. Analogously, $w^{(\text{n})}$ promotes latent space separation as function values grow farther apart. In our experiments we also use the continuous log-ratio metric loss \citep{2019_Kim}.

\subsection{Sample Efficiency with Semi-Supervised Metric-Regularised VAEs}\label{Sec:NewVAE}
Deep metric learning in its general form presumes access to vast quantities of data. To compute $\mathcal{L}_{\text{triple}}(\cdot)$ (our running example) the data must also admit a categorisation between positive and negative input pairs\footnote{Other forms of metric losses we present in Section~\ref{Sec:Exps} might not require such a categorisation. They, however, still assume access to black-box evaluations.}. Generally, this dichotomisation requires access to class labels, which are either readily available in supervised settings or constructed implicitly during data augmentation in self-supervised learning~\citep{yang2006distance, jing2020self}. 
When it comes to high-dimensional BO on the other hand, determining a categorisation rule using either direct supervision or data augmentation is difficult; access to abundant function evaluations is incongruous with sample-efficient optimisation and data augmentation requires significant prior knowledge of downstream tasks contrary to typical settings for black-box optimisation. Therefore, in addition to constructing a suitable deep metric loss for GP regression, we also require that our solution method learns a discriminative latent space with few queries of the black-box.

To enable sample efficiency, we propose to pre-train a VAE in an unsupervised fashion followed by supervised fine-tuning based on BO-derived function evaluations \emph{and} deep metric learning. Such a hybrid semi-supervised framework combining both labelled and unlabelled data presents a solution for a limited black-box evaluation budget. During pre-training, we assume access to large amounts of unlabelled structured data $\mathcal{D}_{\mathbb{U}}=\langle \bm{x}_{m}^{(\text{u})} \rangle_{m=1}^{M}$ and train a standard VAE \emph{without} any label guidance as originally proposed in~\cite{kingma2014}. 

Having pre-trained the VAE, we then utilise function evaluations (i.e. label information) from the BO loop to induce a latent space that facilitates the fit of the GP surrogate. To this end, we derive a new ELBO combining that of~\cite{2020_Tripp} with a deep metric regularisation term of the form $\textbf{ELBO}_{\textbf{DML}}(\bm{\theta}, \bm{\phi}) = \text{Com}_{\text{label}}(\bm{\theta}, \bm{\phi}) - \text{Com}_{\text{metric}}(\bm{\theta}, \bm{\phi})$ (where $\text{Com}$ is an abbreviation for component). Both parts of our ELBO are computed using a labeled dataset $\mathcal{D}_{\mathbb{L}}=\{\bm{x}_{i}, f(\bm{x}_{i})\}_{i=1}^{N}$ representing the points acquired by BO in $N$ rounds. The first component $\text{Com}_{\text{label}}(\bm{\theta}, \bm{\phi})$ is from~\cite{2020_Tripp} and is defined through a set of weights $w(\bm{x}_{i})\propto f(\bm{x}_{i})$ for an input $\bm{x}_i \in \mathcal{D}_{\mathbb{L}}$ as 
$\text{Com}_{\text{label}}(\bm{\theta}, \bm{\phi}) = w(\bm{x}_{i})[\mathbb{E}_{q_{\bm{\phi}}(\bm{z}_{i}|\bm{x}_{i})}[\log g_{\bm{\theta}}(\bm{x}_{i}|\bm{z}_{i})] - \text{KL}(q_{\bm{\phi}}(\bm{z}_{i}|\bm{x}_{i})||p(\bm{z}))]$. 

The second component, however, is unique to this work and acts as a regulariser to shape latent spaces to be favourable for GP modelling (cf. Section~\ref{Subsec:continuous-DML}). In general, we can adapt \emph{any} deep metric loss to our formulation. Due to space constraints, we now instantiate our framework with soft-triplets and refer the reader to Appendix~\ref{App:contrastive} for other forms. Given $\mathcal{D}_{\mathbb{L}}$, we first construct all \emph{valid triplets} according to the threshold $\eta$ as introduced in Section~\ref{Subsec:continuous-DML}. Following a similar reasoning to~\cite{2020_Tripp}, we adopt a weighting scheme such that for any valid triplet $\left\langle \bm{x}_{i}, \bm{x}_{j}, \bm{x}_{k}\right\rangle$, we compute a weighting factor $w_{i,j,k}\propto w(\bm{x}_{i})w(\bm{x}_{j})w(\bm{x}_{k})\propto f(\bm{x}_i)f(\bm{x}_j)f(\bm{x}_k)$ that premultiplies the metric regulariser yielding
\begin{equation*}
    \text{Com}_{\text{metric}=\text{s-triple}} (\bm{\theta}, \bm{\phi}) = w_{ijk} \mathbb{E}_{q_{\bm{\phi}}(\underline{\bm{z}}|\underline{\bm{x}})}[\mathcal{L}_{\text{s-triple}}^{(\text{BO})}(\underline{\bm{z}})],
\end{equation*}
where we use $\underline{\bm{x}} = \left\langle \bm{x}_{i}, \bm{x}_{j}, \bm{x}_{k}\right\rangle$ and $\underline{\bm{z}} = \left\langle \bm{z}_{i}, \bm{z}_{j}, \bm{z}_{k}\right\rangle$ such that $q_{\bm{\phi}}(\underline{\bm{z}}|\underline{\bm{x}}) =  q_{\bm{\phi}}(\bm{z}_i|\bm{x}_i)q_{\bm{\phi}}(\bm{z}_j|\bm{x}_j)q_{\bm{\phi}}(\bm{z}_k|\bm{x}_k)$. 

Enumerating all possible triplets from $\mathcal{D}_{\mathbb{L}}$ can quickly become infeasible at a cost of $\mathcal{O}(N^3)$ and so in practice, we compute the above components over mini-batches\footnote{When mini-batching, weights need to be normalised. We absorb normalising constants into the learning rate. } of size $N_{\text{mini}} << N$.

{\textbf{Weighted Target Prediction VAEs:}} Whilst conducting our experiments, we observed that weighted ELBOs that simultaneously reconstruct and predict function values~\citep{2018_Eissman} performed strongly in the molecule generation tasks we consider. As such, we introduce a novel baseline that extends weighted retraining from~\cite{2020_Tripp} to include target prediction such that $\text{Com}^{(\text{TP})}_{\text{label}}(\bm{\theta}, \bm{\phi}) = \text{Com}_{\text{label}}(\bm{\theta}, \bm{\phi}) + w(\bm{x}_{i}) \mathbb{E}_{q_{\bm{\phi}}(\bm{z}_{i}|\bm{x}_i)}[\log h_{\bm{\theta}}(f(\bm{x}_{i})|\bm{z}_{i})]$ with $h_{\bm{\theta}}(f(\bm{x}_{i})|\bm{z}_{i})$ being an additional decoder network (sharing parameters with $g_{\bm{\theta}}(\cdot)$) geared towards reconstructing $f(\bm{x}_i)$ in $\mathcal{D}_{\mathbb{L}}$.

\subsection{Algorithm \& Theoretical Guarantees}\label{Sec:Theory}
The pseudocode in Algorithm~\ref{Algo:Overall} summarises our approach which consists of two main loops. In the first, a VAE is trained by optimising the ELBO derived in Section~\ref{Sec:NewVAE} arriving at optimal encoder and decoder parameters $\bm{\theta}_{\ell}^{\star}$ and $\bm{\phi}_{\ell}^{\star}$ (line 3). 

\begin{algorithm}[H]
        \caption{High-D BO with VAEs and Deep Metric Learning}
        \label{Algo:Overall}
        \begin{algorithmic}[1]
          \STATE \textbf{Inputs:} Budget $B$, frequency $q$, pre-trained VAE, $\mathcal{D}_{\mathbb{L}}$, and stopping criteria $\tau$
          \STATE \textbf{for} $\ell=1$ to $L \equiv \lceil B/q \rceil$:
          \STATE \hspace{1em} Solve $\bm{\theta}_{\ell}^{\star}, \bm{\phi}_{\ell}^{\star} = \arg\max_{\bm{\theta}, \bm{\phi}} \textbf{ELBO}_{\textbf{DML}}(\bm{\theta}, \bm{\phi})$  
          \STATE \hspace{1em} Compute $\mathcal{D}_{\mathbb{Z}}=\langle\bm{z}_{i}, f(\bm{x}_{i}) \rangle_{i=1}^{N}$ using the encoder \\
          \STATE \hspace{1em} \textbf{for} $k=0$ to $q-1$ \text{and} $\text{EI}(\hat{\bm{z}}_{\ell, k+1}) \geq \tau$:
          \STATE \hspace{2em} Fit surrogate GP on $\langle \bm{z}_{i}, f(\bm{x}_{i})\rangle_{i=1}^{N}$ 
          \STATE \hspace{2em} Optimise EI for $\hat{\bm{z}}_{\ell, k+1}$ 
          \STATE \hspace{2em} Compute $\hat{\bm{x}} = g_{\bm{\theta}_{l}^{\star}}(\cdot|\hat{\bm{z}}_{\ell, k+1})$
          \STATE \hspace{2em} Evaluate $f(\hat{\bm{x}})$ and augment data $\mathcal{D}_{\mathbb{Z}}$ and $\mathcal{D}_{\mathbb{L}}$
          \STATE \hspace{1em} \textbf{end for}
          \STATE \hspace{-.1em} \textbf{end for}
          \STATE \textbf{Output:} $\bm{x}^{\star} = \arg\max_{\bm{x}\in \mathcal{D}_{\mathbb{L}}}f(\bm{x})$
        \end{algorithmic}
      \end{algorithm}
  Given $\bm{\phi}_{\ell}^{\star}$, we compute a dataset $\mathcal{D}_{\mathbb{Z}} = \langle \bm{z}_{i} = \mathbb{E}_{q_{\bm{\phi}_{\ell}^{\star}}(\bm{z}|\bm{x}_{i})}[ \bm{z}], f(\bm{x}_i)\rangle_{i=1}^{N}$ and execute a standard BO loop (lines 5-9) to determine new query points $\hat{\bm{z}}_{\ell, k+1}$ for evaluation. Decoding $\hat{\bm{z}}_{\ell, k+1}$, we then evaluate $\hat{\bm{x}}$ to obtain black-box values which are appended to the dataset. The process above runs for a total of $q$ iterations or until a stopping criterion ($\text{EI}(\hat{\bm{z}}_{\ell, k+1}) < \tau$) is met. After the termination of both loops, Algorithm~\ref{Algo:Overall} outputs the best candidate acquired so far (line 11). \\
\textbf{Theoretical Guarantees:} The remainder of this section is dedicated to providing vanishing regret guarantees for Algorithm~\ref{Algo:Overall}. Such results challenge standard notions of regret analysis in BO due to two coupled loops affecting feasibility sets. To illustrate, imagine that under a fixed $\ell$, $g_{\bm{\theta}^{\star}_\ell}(\cdot)$ does not possess the capacity to recover any input in $\mathcal{X}$. In such a case, although the BO loop (lines 5-9) can arrive at an optimum $\bm{z}^{\star}_{\ell}$, this point when decoded does not necessarily correspond to the true $\bm{x}^{\star} = \arg\max_{\bm{x} \in \mathcal{X}} f(\bm{x})$. To shed light on such behaviour, we define a new notion of cumulative regret that encompasses both $\ell$ and $k$: 
\begin{equation*}
\text{Reg}_{L, q}=
\sum_{\ell=1}^{L}\sum_{k=1}^{q} \left(f(\bm{x}^{\star}) -
\mathbb{E}_{\bm{x}_{\ell, k}^{\star} \sim g_{\bm{\theta}^{\star}_{\ell}}(\cdot |\hat{\bm{z}}_{\ell, k})}\left[f(\bm{x}^{\star}_{\ell, k})\right]\right).
\end{equation*}
\textbf{Domain Recovery Assumptions for VAEs:} In analysing asymptotic regret, we impose three assumptions (see Appendix~\ref{App:Theory}) two of which are standard practice in BO~\citep{RegretEI2017} bounding norms and posterior variances, while the third is new and corresponds to handling the representational power of $g_{\bm{\theta}^{\star}_\ell}(\cdot)$. Here we assume that as the outer loop progresses (i.e. as we gather more data), the VAE improves its ability to recover around $\bm{x}^{\star}$ such that for any $\ell \ge \ell^{\prime}$, there exists a $\bar{\bm{z}}(\ell) \in \mathcal{Z}$ where $\mathbb{P}\text{r}\left[\bm{x}^{\star}\sim g_{\bm{\theta}_{\ell}^{\star}}(\cdot | \bar{\bm{z}}(\ell))\right] \geq 1 - \gamma(\ell)$ with $\gamma(\ell)$ being a decreasing function. Although this assumption is less restrictive than having the VAE reconstruct $\bm{x}^{\star}$ at any iteration, it requires further analysis due to its relation to the generalisation properties of generative models. Providing a PAC Bayes generalisation bound of VAEs in the context of high-D BO undoubtedly constitutes an exciting direction for future work. Here, we instead motivate our assumption through a tightness analysis of necessary and sufficient conditions (see Appendix~\ref{App:Theory}) and run a dedicated experiment (see Appendix~\ref{App:AssumptionV}) showing its validity in one of our empirical settings. With this, we prove sub-linear convergence in non-convex black-box optimisation:
\begin{theorem}\label{Th: main_theorem}
Algorithm \ref{Algo:Overall} with $q = \lceil B^{\frac{2}{3}}\rceil$, $L = \lceil B^{\frac{1}{3}}\rceil$ and under the assumptions in Appendix~\ref{App:Theory} admits sub-linear regrets, i.e., 
$\lim_{B\to\infty}\frac{1}{B} \text{Reg}_{L, q} \rightarrow 0$, with a probability of at least $1-\delta$ for $\delta\in(0,1)$.
\end{theorem}


\section{Experiments \& Results}\label{Sec:Exps}
We apply Algorithm~\ref{Algo:Overall} to three high-dimensional, structured BO tasks (see Appendix~\ref{App:Exps} for full details).

\textbf{Topology Shape Fitting:} As a new toy problem, we employ the topology dataset from~\citep{sosnovik2017neural} and formulate an optimisation problem that seeks to generate a $40\times 40$ image (representing a mechanical design) such that the cosine similarity $\cos(\bm{x},\bm{x}')=\bm{x}\bm{x'}^T/\Vert \bm{x} \Vert \Vert \bm{x}' \Vert$ to a pre-defined target image is minimised. We use a VAE with latent space of dimension 2 and 10'000 data points (see code for exact architecture).\\
\textbf{Expression Reconstruction:} Following~\citep{2020_Tripp, kusner2017grammar}, we consider generating single-variable expressions from a formal grammar (e.g. \texttt{3*sin(2+x)}) and minimising a distance (based on Mean Squared Error of expression evaluated at fixed points) to a target equation \texttt{x*sin(x*x)}. We allow access to 40,000 data points and use the grammar VAE from~\citep{kusner2017grammar} with a latent space of dimension 25.\\
\textbf{Chemical Design:} Similar to~\citep{2020_Tripp}, we optimise the penalised water-octanol partition coefficient (PlogP) objective of molecules using the ZINC250K dataset~\citep{sterling2015zinc}. Each molecule is represented as a unique SMILES sequence and we utilise a Junction-Tree VAE~\citep{jin2018} with a latent space dimension of 56 for encoding and generating novel and valid molecules. 


\textbf{Choice of Metric Loss:} In our exposition, we introduced the triplet loss as a running example. Since our goal is to investigate the performance of general deep metric learning in conjunction with VAE BO, we implement three additional metric losses. We describe these losses in brief here and refer the reader to Appendix~\ref{App:AdditionalMetrics} for more details.\\
{\textbf{Simple Loss:}} For a pair of inputs $\left\langle \bm{x}_{i}, \bm{x}_{j}\right \rangle$, we  regularise the VAE using $\Delta f_{ij}= f(\bm{x}_{i}) - f(\bm{x}_j)$ as: $\text{Com}_{\text{metric}=\text{simple}} (\cdot, \cdot) \propto w_{ij} \mathbb{E}_{q_{\bm{\phi}}(\cdot)}\left[| \ ||\Delta \bm{z}_{ij}||-|\Delta f_{ij}| \ |\right]$ with $\Delta \bm{z}_{ij}= \bm{z}_{i} - \bm{z}_{j}$.\\
{\textbf{Contrastive Loss:}} The contrastive loss is another widely used deep metric~\cite{Contrastive_2006}. It operates on input pairs and separates latent encodings based on class label information. Like the triplet loss, the contrastive loss is not directly applicable to BO and requires similar modifications to those in Section~\ref{Subsec:continuous-DML} that we describe in Appendix~\ref{App:contrastive} leading us to: $\text{Com}_{\text{metric}=\text{s-cont}} (\cdot, \cdot) = w_{ij} \mathbb{E}_{q_{\bm{\phi}}(\cdot)}[\mathcal{L}_{\text{s-cont}}^{(\text{BO})}(\cdot)]$. \\
{\textbf{Log-Ratio Loss:}} The log-ratio loss~\citep{2019_Kim} is a continuous triplet loss that may also be applied to VAE BO. Given a triplet of inputs $\left\langle\bm{x}_{i}, \bm{x}_{j}, \bm{x}_{k}\right\rangle$ with $\bm{x}_{i}$ being the anchor, we define: 
$\text{Com}_{\text{metric}=\text{log-ratio}} (\cdot, \cdot) = w_{ijk} \mathbb{E}_{q_{\bm{\phi}}(\cdot)}[(\log \sfrac{||\Delta \bm{z}_{ij}||}{||\Delta \bm{z}_{ik}||} - \log \sfrac{|\Delta f_{ij}|}{|\Delta f_{ik}|})^{2}]$. 

We next highlight the findings of our experiments.

\subsection{DML Induces Discriminative Latent Spaces}\label{Sec:ExpDisc}
We assess the capability of DML to construct useful discriminative latent spaces for GPs by conducting modelling experiments across all three tasks. Utilising the same weight design as~\citep{2020_Tripp}, we add $\text{Com}_{\text{metric}}(\cdot)$ to the weighted ELBO and use the combined loss to train the VAE. Equipped with the trained encoder, we map points in the original space $\mathcal{D}_{\mathbb{L}}$ to a latent dataset $\mathcal{D}_{\mathbb{Z}}$ on which we fit a GP, using the labels of the original points. We then assess the latent space in terms of separation and ability for the GP to generalise.\\
\textbf{Separation in Latent Encodings:} In Section~\ref{Sec:Sol}, we noted that deep metric learning induces a discriminative latent space for regression by encouraging encoded inputs with similar function values to cluster together as well as encoded inputs with different function values to be separated. To confirm this behaviour, we study the the distribution of distances between points in the latent space of the molecule task (results on all other tasks can be found in Appendix~\ref{App:Exps}). 



\begin{figure*}[t!]
    \centering
    \includegraphics[scale=0.21]{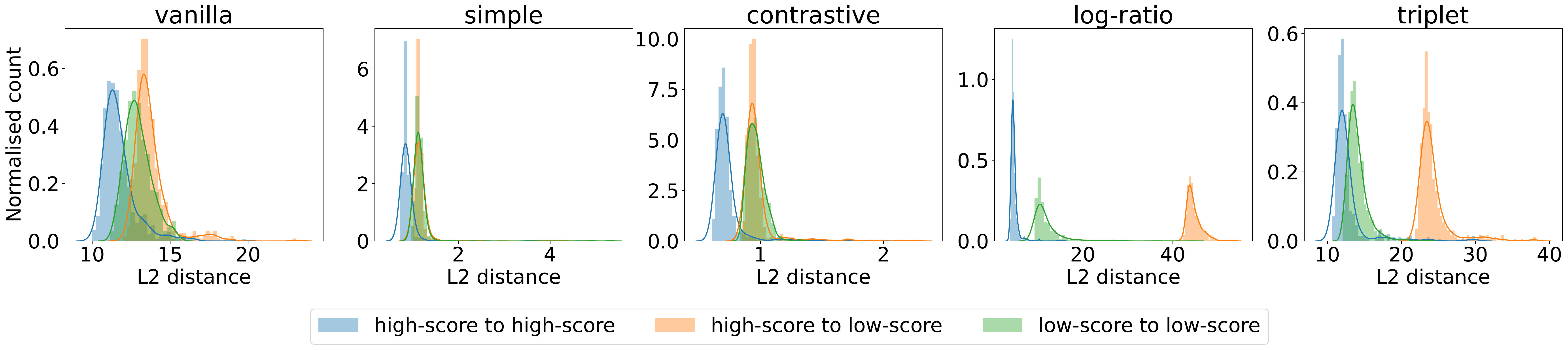}
    \caption{Distance distribution between latent representations in the molecule task. From left to right: Vanilla VAE, VAE trained with simple metric loss, VAE with contrastive loss, VAE trained with log-ratio loss and Triplet-VAE. The $x$ axis is the $L_2$ distance between points in latent space while the $y$ axis is the normalised count.}
    \label{fig:latent_dist_histogram}
\end{figure*}

To do so, we rank the inputs in $\mathcal{D}_{\mathbb{L}}$ according to their black-box function values and split them equally into two parts: high and low-scoring inputs. Mapping these inputs onto $\mathcal{Z}$ using our trained encoder, we compute the distances between all pairs of latent encodings. Consequently, attaining three groups of distances: 1) between high-scoring inputs and other high-scoring inputs, 2) low-scoring to low-scoring, and 3) high-scoring to low-scoring points. Figure~\ref{fig:latent_dist_histogram} shows the distribution of such distances when using different metric regularisers. It is clear, that augmenting VAEs with $\text{Com}_{\text{metric}}(\cdot)$ achieves the goal of latent space clustering in accordance with function values. Importantly, soft-triplets and log-ratio losses yield the best results where low-scoring and high-scoring inputs are tightly clustered together but inter-cluster separation is large. \\
\textbf{Generalisation:} To verify if clustered latent inputs improve GP generalisation, we unify the experimental setting across all tasks and utilise $80\%$ of $\mathcal{D}_{\mathbb{Z}}$ for training a sparse GP with 500 inducing points. Table~\ref{table:gp_pred_mse} reports the predictive log likelihood on $20\%$ held-out validation sets. This is repeated over 5 random splits. Although this data differs from that which would iteratively be acquired through BO, this experiment serves as a study of the effect of clustered inputs in GP regression\footnote{Note that we are interested in standard GP regression that is the most widely used in BO. In our experiments we also baseline against warped-input GPs~\cite{2020_Cowen}.} (an essential component inside a BO loop). The GP fit on a latent space induced by a Vanilla VAE is outperformed on all tasks demonstrating that metric learning-induced separation in the latent space aids GP generalisation. 

\begin{table}[h!]\caption{GP predictive log-likelihood $\pm$ 1-standard deviation on the validation set.} 
\centering
\begin{tabular}{c | c c c } 
\toprule
   & Top. & Expr. & Mol.  \\ 
\hline
Vanilla  & -1.87 (0.06) & -2.99 (0.06) & -1.79 (0.21) \\ 
Simple  & -2.57 (0.03) & -3.4 (0.08) & -2.05 (0.26) \\
Cont.  & -1.75 (0.02) & \textbf{-1.39 (0.04)} & -1.75 (0.18) \\
LogR.  & \textbf{-1.05 (0.01)} & -3.26 (0.31) & -2.12 (0.39) \\
Triplet  & -2.03 (0.02) & -1.91 (0.08) & \textbf{-1.55 (0.35)} \\
\bottomrule
\end{tabular}
\label{table:gp_pred_mse}
\end{table}

\subsection{DML Improves High-D VAE BO}\label{Sec:Blo}
The experiments consist of three main steps (see Appendix~\ref{App:Exps}), namely training the VAE, fitting the GP and running the BO loop. As in~\citep{2020_Tripp} we periodically retrain the VAE after a set number of iterations and collected points. To ensure a fair comparison with prior work, we unify our experimental setup with~\citep{2020_Tripp} setting the retraining frequency of the VAE to $r=50$ steps and the rank weight parameter to $k=10^{-3}$ (see~\citep{2020_Tripp} for an ablation study on those two hyper-parameters). In this experiment, we pre-train the VAEs using \emph{all} available data per task (10,000 in topology, 40,000 in expression and 250,000 in molecules). 

\textbf{ELBO Specifications \& Baselines:} We call {LBO} the baseline method from~\citep{2020_Tripp}; {TP-LBO} is similar to~\citep{2020_Tripp} but adds the target prediction~\citep{2018_Eissman} component $\text{Com}_{\text{label}}^{(\text{TP})}$; {W-LBO} is again similar to~\citep{2020_Tripp} but uses a GP model with \emph{input warping} from~\cite{2020_Cowen}; S-LBO, C-LBO, LR-LBO, and T-LBO are our methods combining different metric losses with VAE BO. Table~\ref{table:warm_elbo_components} summarises the ELBO components used for each method (we refer the reader to Appendix~\ref{App:elbo_comp_details} for a detailed description). 

\begin{table}[h!]\caption{ELBO Components per baseline.}
\centering
\begin{tabular}{c| c} 
 \toprule
     Notation    & ELBO \\
 \hline
 LBO    & $\text{Com}_{\text{label}}(\cdot)$ \\ 
 TP-LBO  & $\text{Com}^{(\text{TP})}_{\text{label}}(\cdot)$ \\
 W-LBO  & $\text{Com}_{\text{label}}(\cdot)$ \& input warping\\ 
 S -LBO & $\text{Com}_{\text{label}}(\cdot) + \text{Com}_{\text{metric}=\text{simple}}(\cdot)$ \\
 C-LBO  & $\text{Com}_{\text{label}}(\cdot) + \text{Com}_{\text{metric=s-cont}}(\cdot)$ \\
 LR-LBO  & $\text{Com}_{\text{label}}(\cdot) + \text{Com}_{\text{metric=log-ratio}}(\cdot)$ \\
 T-LBO  & $\text{Com}_{\text{label}}(\cdot) + \text{Com}_{\text{metric=s-triple}}(\cdot)$ \\
 \bottomrule
\end{tabular}
\label{table:warm_elbo_components}
\end{table}
We additionally include random search ({RS}) in the topology and expression tasks and the molecule-specific baselines ({CEM-PI}, {DbAS}, {FBVAE} and {RWR}) from~\citep{2020_Tripp}.
\begin{figure*}[ht!]
    \centering
    \includegraphics[scale=0.107]{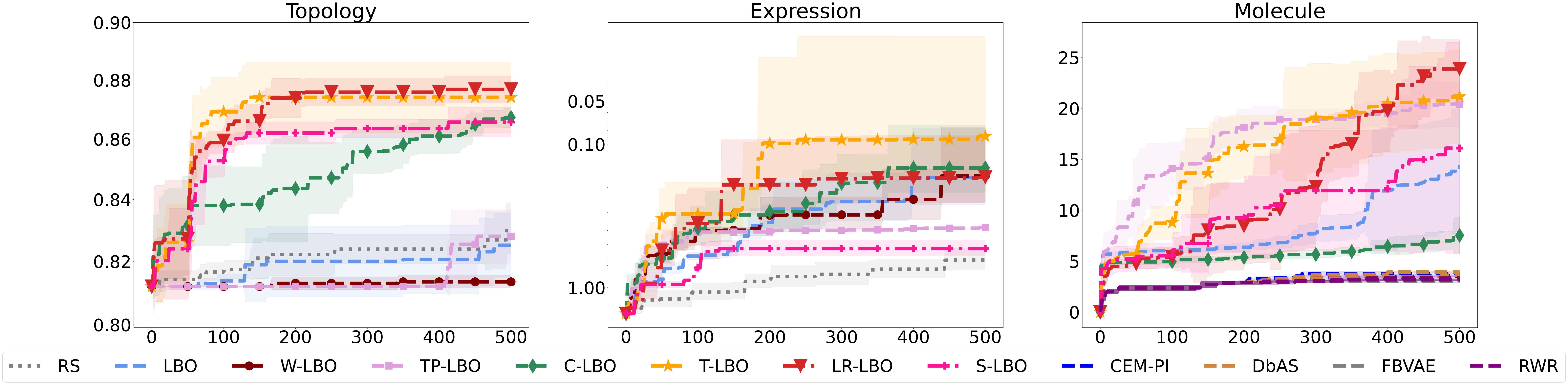}
    \caption{Cosine similarity score on the Topology task, regret on the Expression task and penalised logP-score for the Molecule task. The best value averaged over 5 seeds (and its standard deviation) is shown at each iteration.}
    \label{fig:score_warm}
\end{figure*}

\newpage

\textbf{Results:} Figure~\ref{fig:score_warm} summarises our findings on all three tasks averaged over 5 random seeds. First, metric learning improves {LBO}'s performance, which we found to be less competitive in the topology and molecule tasks. Metric losses based on triplet information (i.e. soft-triplet and log-ratio) consistently outperform all methods across all tasks. At the same time, contrastive BO achieves significant gains in the first two tasks (topology and expression) but under-performs in the molecule task. Simple metric losses, moreover, attain improvements to LBO in topology and molecule tasks but not in the expression task. Target prediction VAEs (i.e. {TP-LBO}), on the other hand, yield competitive results in the molecule task but fail in topology and expression. Both {T-LBO} and {LR-LBO} are performant in all three tasks. Finally, W-LBO achieved no gains on topology and expression tasks compared to LBO and thus was not run in the molecule task.




\subsection{High-D VAE BO with Limited Black-Box Queries}\label{subsec:cold_start_res} 
The previous section demonstrated that metric learning aids VAE BO with access to large amounts of labelled data. In many black-box problems however, we are more interested semi-supervised setting that we introduced in Section~\ref{Sec:NewVAE}. Here, we only allow access to \emph{$1\%$} of the labels (chosen at random) from $\mathcal{D}_{\mathbb{L}}$ and pre-train the VAEs solely to reconstruct the structured inputs. Given those models, we then implement Algorithm~\ref{Algo:Overall}, executing BO and periodic refinement of the latent space based on label information (i.e. weighted retraining). During this fine-tuning phase, we use the baselines from Table~\ref{table:warm_elbo_components}, reporting the results in the molecule task in the main paper and all others in Appendix~\ref{App:Exps}. 

We implement Algorithm~\ref{Algo:Overall} identically for all baselines and impose a total budget of 1000 iterations. We terminate the loop early in case any baseline recovers previous state-of-the-art PLogP values~\cite{2020_Tripp}. For a fair comparison to~\cite{2020_Tripp}, we also pre-train their VAE using all unlabeled data but additionally incorporate an implementation with $1\%$ of $\mathcal{D}_{\mathbb{L}}$ which we entitle LBO-$1\%$. 

\begin{figure}[h!]
    \centering
    \includegraphics[scale=0.11]{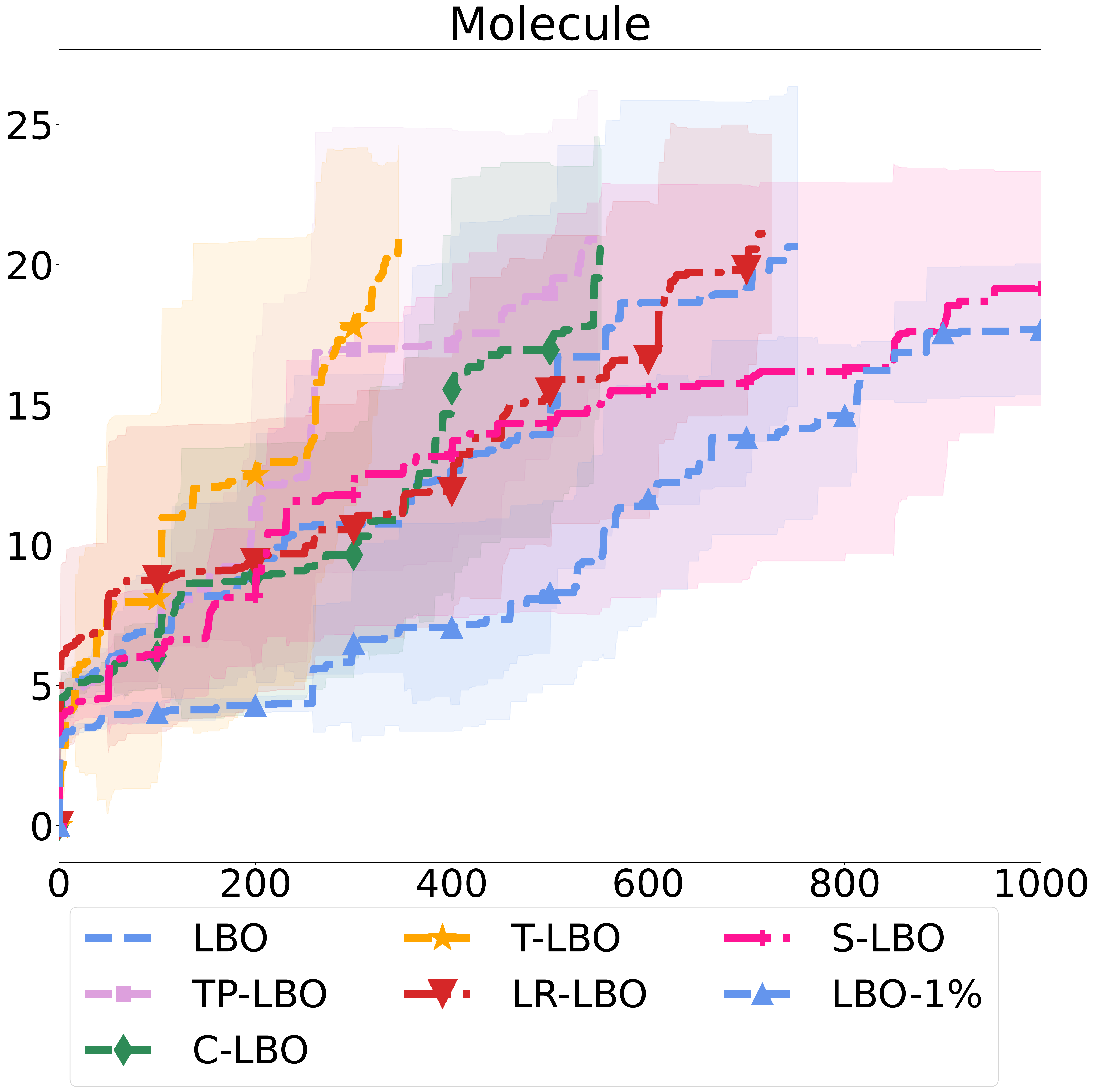}
    \caption{BO results for the PlogP molecule generation benchmark using the semi-supervised approach. Shaded regions denote $\pm$ 1-standard deviation over 5 random seeds.}
    \label{fig:score_cold}
\end{figure}

\textbf{Molecule Generation with Limited Labels:} 
Figure~\ref{fig:score_cold} summarises our results conveying that all algorithms except LBO-$1\%$ and S-LBO can indeed recover the best PlogP score from the Section~\ref{Sec:Blo} while needing few black-box evaluations \emph{in total}\footnote{By total, we mean \emph{all} labelled data used to train the GPs and the VAEs in weighted retraining as well as the data acquired in BO.}. It takes on average 2437 evaluations for {T-LBO}, 2596 for {TP-LBO}, and 2753 for {LBO}. Again, we realise that metric learning is beneficial since T-LBO outperforms other algorithms reducing total black-box evaluation demands by $\approx 11.5 \%$ compared to LBO and by $\approx 6.1\%$ versus TP-LBO. On the other hand, C-LBO and LR-BO provide competitive baselines to LBO but underperform compared to both TP-LBO and T-LBO. Both S-LBO and  LBO-$1\%$ fall short when data is limited, exhausting all 1000 iterations while attaining only about $86\%$ and $77\%$ respectively of the objective score from the fully supervised setting of Section~\ref{Sec:Blo}. To the best of our knowledge, this is the first recorded result of penalised logP molecule values with thousands rather than hundreds of thousands of \emph{total} black-box evaluations. 

From the above results, we conclude that: 1) BO can still be successful in high-d structured tasks when \emph{only limited labels are available} as long as algorithms use the semi-supervised setting, 2) our adaptations of triplet metric learning methods further those improvements, consistently achieving the best performance across all tasks, and 3) our proposed weighted target prediction VAEs indeed present a competitive baseline in the molecule task.


Additional results comparing our method to competitor methods on the ZINC 250K PlogP molecule generation benchmark can be found in Table~\ref{table:sota_table}. It is worth noting that these alternative methods do not necessarily use a BO-based approach. Nonetheless, this comparison reveals that our method, to the best of our knowledge, outperforms other methods in raw PlogP score while using only a fraction of the labelled data they require.

\begin{table}[]\caption{Top scoring molecules on the penalised logP (PlogP) benchmark. Number of function evaluations given in the right column. Full results are provided in Appendix B (We note that we achieve a new state-of-the-art score of 38.57 using only 7,750 function evaluations i.e. ca. 3\% of the available labels.)}
\centering
\begin{tabular}{c c c} 
 \toprule
 Method & PlogP & \# evals \\
 \hline
 ZINC-250K~\citep{sterling2015zinc} & $4.52$ & -\\ \hline
 JANUS~\citep{2021_janus} & $21.92$ & 250,500\\
 IS-MI~\citep{2021_Notin} & $27.60$ & 250,500\\
 LSO~\citep{2020_Tripp} & $27.84$ & 250,500\\
 All SMILES VAE~\citep{2019_Alperstein} & $29.80$ & 250,500\\ \hline
 T-LBO (ours) & \textbf{38.57} & 7,750\\
 \bottomrule
\end{tabular}
\label{table:sota_table}
\end{table}

\section{Related Work}\label{Sec:RelatedWork}


\textbf{High-Dimensional BO:} 
High-d BO schemes can be categorised into methods based on embeddings and methods that rely on assumptions about the problem structure. The foundational work on embedding-based methods was undertaken in \citep{2016_Wang} where random embeddings were used to scale BO to a billion dimensions. This work was built on in subsequent work \citep{2014_Garnett, 2017_Rana, 2017_Li, 2020_Letham}. Methods based on the assumption of additive structure in the objective have also been widely applied \citep{2015_Kandasamy, 2017_Gardner, 2020_Binois}. Methods that rely on assumptions about the problem structure include local modelling approaches such as TuRBO \citep{2019_Eriksson} or context-specific kernels \citep{2018_Oh} as well as methods based on deep kernel learning \citep{2016_Wilson, 2018_Jean}. None of the aforementioned approaches however are well-suited to high-dimensional and structured input spaces. BO over structured inputs such as strings \citep{2020_Moss_Boss, 2020_Moss}, graphs \citep{2019_Dhamala} and combinatorial inputs \citep{2020_Deshwal} is an active area of research. Non VAE-based approaches however, lack the capabilities to generate novel structures such as molecules \citep{2018_Gomez} without invoking domain-specific engineering such as synthesis graphs as in \citep{2020_Korovina}. VAE-based methods are prevalent \citep{2018_Gomez, 2020_Griffiths, 2018_Lu, 2018_Eissman, 2020_Siivola, 2020_Tripp, 2019_Zhang} yet suffer from the outstanding question of how best to encourage label guidance, the problem addressed in this paper.


\textbf{Deep Metric Learning:} While many approaches aim to extend deep metric losses to regression settings \citep{2020_Thoma} or construct discriminative latent spaces by other means \citep{2018_Tzoreff, 2021_Ha}, we use this section to survey related work on combining deep metric learning with VAEs. To the best of our knowledge ours is the first work to consider deep metric learning VAEs in the context of BO. In \citep{2018_Ishfaq} a triplet loss VAE tailored for classification tasks is introduced. \citep{2021_Tanaka} use a contrastive loss under weak supervision with an application towards finding disentangled representations of musical instrument sounds. The closest deep metric learning and VAE model to ours is that of \citep{2021_Koge} where the authors use the continuous log-ratio loss \citep{2019_Kim} designed for prediction tasks. Although applicable to continuous domains, the work in~\citep{2019_Kim} relies on data augmentation protocols that assume prior knowledge of the black-box. We discuss additional related work in Appendix~\ref{App:additional_related_work}.





\section{Conclusion}

We propose a method for high-d BO with VAEs using deep metric learning to affect a discriminative latent space. We instantiate our method using four different metric losses, demonstrating state-of-the-art performance on the ZINC-250K penalised logP molecule generation benchmark. Importantly, in the semi-supervised setting, comparable performance to previous approaches is achieved using just 1\% of the available labelled data and superior performance with ca. 3\% of the labels. Additionally, we present a proof of sublinear regret and introduce a new competitive baseline that combines weighted retraining with target prediction yielding favourable results in molecule generation.  

Future work could feature the exploration of different forms of metric losses \citep{2013_Bellet} as well as more chemically-principled objectives for molecule generation \citep{2019_Brown, 2020_moses, 2021_Griffiths}. Our theoretical results are predicated on an assumption of coverage over $\bm{x}^{\star}$ for VAEs. In subsequent work, we wish to relax this assumption and prove a PAC-Bayes generalisation bound. We hope that the principles outlined in this paper may be used to design VAE-based BO schemes that operate successfully over continuous and structured input spaces.

\bibliographystyle{unsrtnat}
\bibliography{bib}

\newpage
\etocsettocdepth.toc {subsection}


\appendix
\onecolumn
\tableofcontents

\section{Derivations of Evidence Lower Bounds}
In this section, we describe softening strategies for the triplet and contrastive metric loss functions and provide derivations of the associated ELBO objectives.

\subsection{Contrastive Loss VAE \& $\mathcal{L}_{\text{cont.}}$}\label{App:contrastive}
\paragraph{Contrastive Loss:} Most frequently encountered in classification settings, the contrastive loss aims to minimise the Euclidean distance between inputs, e.g. images, of the same class whilst maximising the distance between inputs of different classes. The features learned by contrastive loss deep metric learning have been observed to improve generalisation performance in downstream classification tasks \cite{2020_Khosla} a property we hypothesise to be important in BO. Concretely, for two inputs $\left\langle \bm{x}_{i}, c_i\right\rangle$ and $\left\langle \bm{x}_{j}, c_j \right \rangle$ with $c_i$ and $c_j$ being class labels, a contrastive loss in its most basic form~\cite{hadsell2006dimensionality} can be computed as: $\mathcal{L}_{\text{cont.}} (\cdot) \propto ||\bm{z}_{i} - \bm{z}_{j}||_{q}  \ \ \text{if $c_i = c_j$} \ \ \text{or} \ \ \mathcal{L}_{\text{cont.}} (\cdot) \propto \max\{0, \rho - ||\bm{z}_{i} - \bm{z}_{j}||_{q}\} \ \ \text{if $c_i \neq c_j$}$, where $||\cdot||_q$ denotes a $q$ norm, $\bm{z}_{i}$ and $\bm{z}_{j}$ are latent encodings of $\bm{x}_{i}$ and $\bm{x}_{j}$, and $\rho$ is a tuneable margin that defines a radius around $||\bm{z}_{i} - \bm{z}_j||_{q}$. Clearly, dissimilar pairs contribute to the loss only if $||\bm{z}_{i} - \bm{z}_j||_{q} \leq \rho$, otherwise $\mathcal{L}_{\text{cont.}}(\cdot) = 0$. This definition allows immediate extension to continuous class labels by assuming $c_i = c_j \text{ if } |f(\bm{x}_i) - f(\bm{x}_j)| < \eta$ and $c_i\neq c_j$ otherwise, where $\eta$ is a threshold parameter controlling the granularity of class separation. Finally, in order to connect the value of the contrastive loss $\mathcal{L}_{\text{cont}}(\cdot)$ with the magnitude of class mismatch ~\cite{pan2018}, the associated margin is chosen as $\rho_{i,j} = |f(\bm{x}_i) - f(\bm{x}_j)|$.

\paragraph{Soft Contrastive Loss:}
As shown in Figure \ref{Fig:SoftContrastive}, $\mathcal{L}_{\text{cont}}(\cdot)$ as defined above exhibits discontinuous behaviour around the line $|f(\bm{x}_i) - f(\bm{x}_j)| = \eta$ which can be detrimental for GP regression. To remedy this issue, we introduce two contrastive penalty measures $\mathcal{L}^{(\text{BO})}_{\text{cont}, 1}(\cdot)$ and $\mathcal{L}^{(\text{BO})}_{\text{cont}, 2}(\cdot)$ defined as:
\begin{align}\label{eq:soft-cont}
    \mathcal{L}^{(\text{BO})}_{\text{cont}, 1}(\bm{z}_i, \bm{z}_j) =& \text{ReLU}\bigg[\frac{1}{\eta}\max\{\eta, \Delta_{\bm{z}}\}\times \left[\min\{\eta, \Delta_{\bm{z}}\} - \Delta_f\right]\bigg]\mathbbm{1}_{\{\Delta_f < \eta\}},\\\nonumber
    \mathcal{L}^{(\text{BO})}_{\text{cont}, 2}(\bm{z}_i, \bm{z}_j) =& \text{ReLU}\bigg[\left[2 - \frac{1}{\eta}\min\{\eta, \Delta_{\bm{z}}\}\right]\times \left[\Delta_f-\max\{\eta, \Delta_{\bm{z}}\}\right]\bigg]\mathbbm{1}_{\{\Delta_f \ge \eta\}}
\end{align}
where $\mathbbm{1}_{\{A\}}$ is a characteristic function for condition $A$ \footnote{I.e. $\mathbbm{1}_{\{A\}} = 1$ if 
condition $A$ is met and $0$ otherwise}, $\Delta_{\bm{z}} = ||\bm{z}_i - \bm{z}_j||_q$, $\Delta_f = |f(\bm{x}_i) - f(\bm{x}_j)|$, and $\eta > 0$ is a proximity hyperparameter. The first penalty measure  $\mathcal{L}^{(\text{BO})}_{\text{cont}, 1}(\bm{z}_i, \bm{z}_j)$ discourages  points to be distant in the latent space if their objective function values are close. The first factor $\frac{1}{\eta}\max\{\eta, \Delta_{\bm{z}}\}$ plays the role of a multiplicative weight, scaling proportionally to $\Delta_{\bm{z}}$ and allowing us to tune the value $\mathcal{L}^{(\text{BO})}_{\text{cont}, 1}(\bm{z}_i, \bm{z}_j)$ proportionally to $\Delta_{\bm{z}}$. The second factor $\min\{\eta, \Delta_{\bm{z}}\} - \Delta_f$ imposes a gradual change of the contrastive loss near the line $|f(\bm{x}_i) - f(\bm{x}_j)| = \eta$ (cf. Figure \ref{Fig:SoftContrastive}). The function $\mathcal{L}^{(\text{BO})}_{\text{cont}, 2}(\bm{z}_i, \bm{z}_j)$ discourages points to be close in the latent space if their objective function values are distant. This function also comprises a multiplicative weight $[2 - \frac{1}{\eta}\min\{\eta,\Delta_{\bm{z}}\}]$ that decreases linearly for $0 < \Delta_{\bm{z}}\le \eta$ and a smoothing factor $\Delta_f - \max\{\eta, \Delta_{\bm{z}}\}$ which assures smooth behaviour around the line $|f(\bm{x}_i) - f(\bm{x}_j)| = \eta$ (cf. Figure \ref{Fig:SoftContrastive}).   
\begin{figure*}
\centering
\subfloat{\includegraphics[scale=.25]{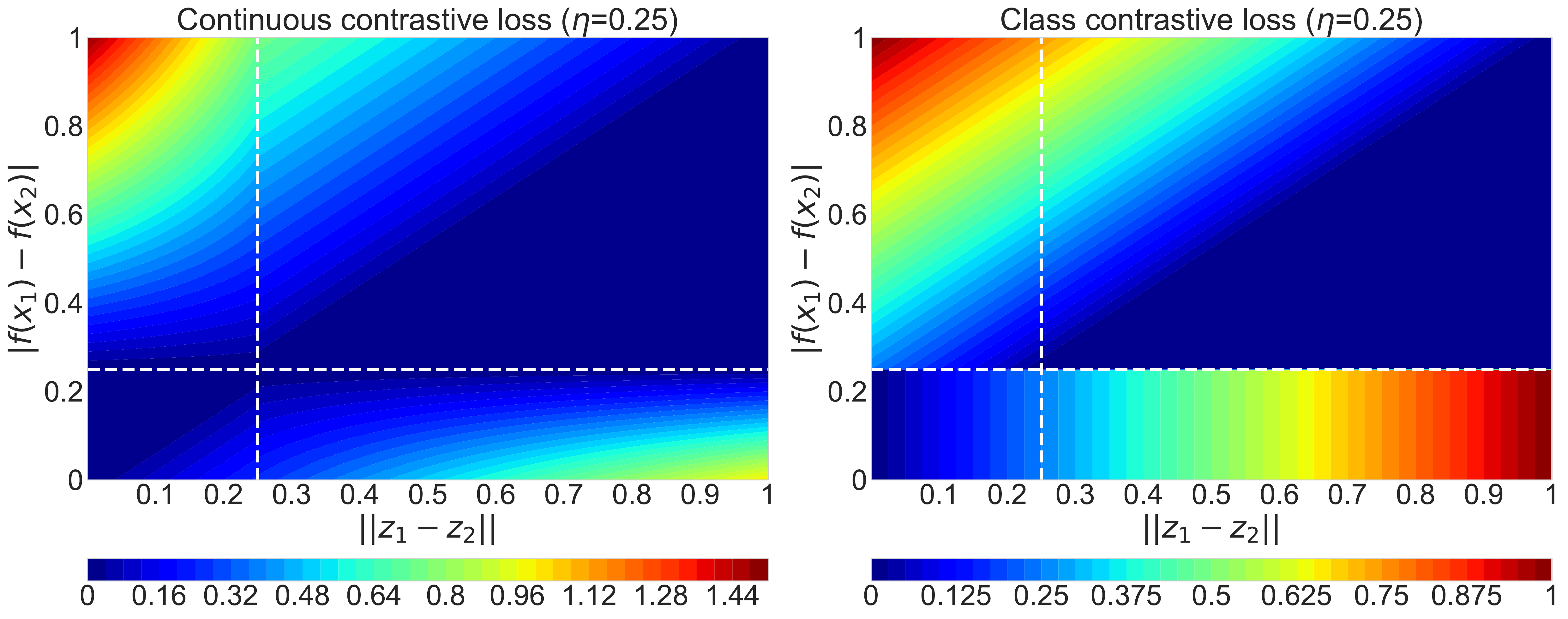}}
\caption{\textbf{Soft Contrastive Loss}. The right figure shows the discontinuity in the original class contrastive loss: $\mathcal{L}_{\text{cont}}(\cdot)\propto \Delta_{\bm{z}}$ if $|f(\bm{x}_i) - f(\bm{x}_j)| < \eta$ and $\mathcal{L}_{\text{cont}}(\cdot)\propto \text{ReLU}(|f(\bm{x}_i) - f(\bm{x}_j)| - \Delta_{\bm{z}})$  if $|f(\bm{x}_i) - f(\bm{x}_j)| \ge \eta$ in the absence of the softening mechanism presented in equation (\ref{eq:soft-cont}). The discontinuity appears for all values of $\Delta_{\bm{z}}$ when approaching the line $|f(\bm{x}_i) - f(\bm{x}_j)| = \eta$ from below, and for values $0 <\Delta_{\bm{z}} \le \eta$ when approaching this line from above. The softening mechanism allows to make a continuous transition between these two regimes. As shown in Table \textcolor{blue}{1} in the main paper, smooth behaviour of the contrastive loss facilitates GP regression.}
\label{Fig:SoftContrastive}
\end{figure*} 
Combining these two measure gives a soft contrastive loss over continuous support:
\begin{equation*}
    \mathcal{L}^{(\text{BO})}_{\text{s-cont}}(\bm{z}_i,\bm{z}_j) = \mathcal{L}^{(\text{BO})}_{\text{cont}, 1}(\bm{z}_i, \bm{z}_j) + \mathcal{L}^{(\text{BO})}_{\text{cont}, 2}(\bm{z}_i, \bm{z}_j)
\end{equation*}

\paragraph{Variational Soft Contrastive Loss:} To construct the joint latent model for this contrastive loss let us introduce a pair of  Bernoulli random variables $a_{ij}$ and $b_{ij}$ for input points $\bm{x}_i, \bm{x}_j\in\mathcal{D}_{\mathbb{L}}$ such that $a_{ij} = \mathbbm{1}_{\{|f(\bm{x}_i) - f(\bm{x}_j)| < \eta\}}$ and $b_{ij} = \mathbbm{1}_{\{|f(\bm{x}_i) - f(\bm{x}_j)| \ge \eta\}}$.
Given latent representations $\bm{z}_i, \bm{z}_j\sim q_{\bm{\phi}}(\cdot|\bm{x}_i, \bm{x}_j)$ for input points $\bm{x}_i$, $\bm{x}_j\in\mathcal{D}_{\mathbb{L}}$ respectively, we set probability distributions for the random variables $a_{ij}$ and $b_{ij}$ to be:
\begin{align*}
    &\mathbb{P}[a_{ij}=1|\bm{z}_i, \bm{z}_j] = e^{-\mathcal{L}^{(\text{BO})}_{\text{cont}, 1}(\bm{z}_i,\bm{z}_j)}, \\ &\mathbb{P}[b_{ij}=1|\bm{z}_i, \bm{z}_j] = e^{-\mathcal{L}^{(\text{BO})}_{\text{cont}, 2}(\bm{z}_i,\bm{z}_j)}.
\end{align*}
In other words, random variable $a_{ij}$ is more likely to take on the value $1$ for latent inputs $\bm{z}_i,\bm{z}_j$ with close function values (i.e. $|f(\bm{x}_i) - f(\bm{x}_j)| < \eta$) if the distance between these two points in the latent space is small (i.e. $\Delta_{\bm{z}} \le |f(\bm{x}_i) - f(\bm{x}_j)|$). Random variable $b_{ij}$ on the other hand, is more likely to take on a value of $1$ for latent inputs $\bm{z}_i,\bm{z}_j$ with distant function values (i.e. $|f(\bm{x}_i) - f(\bm{x}_j)| \ge \eta$) if these two  points in the latent space are distant (i.e. $\Delta_{\bm{z}} > |f(\bm{x}_i) - f(\bm{x}_j)|$). It is important to note that the labelled dataset $\mathcal{D}_{\mathbb{L}} = \langle\bm{x}^{(\text{l})}_n, f(\bm{x}^{(\text{l})}_n)\rangle^N_{n=1}$ provides us with realisations of the random variables $\mathcal{A} = \langle a_{ij}\rangle^{N,N}_{i,j=1}$ and $\mathcal{B} = \langle b_{ij} \rangle^{N,N}_{i,j=1}$ for all pairs of input points $\bm{x}^{(\text{l})}_{i,j} = \langle \bm{x}_i,\bm{x}_j \rangle$. Merging these realisations with the unlabelled data $\mathcal{D}_{\mathbb{U}}= \langle\bm{x}^{(\text{u})}_m\rangle^M_{m=1}$ for the joint log-likelihood gives: 
\begin{align*}
    \log p_{\bm{\phi},\bm{\theta}}(\mathcal{D}_{\mathbb{U}}, &\mathcal{D}_{\mathbb{L}},\mathcal{A},\mathcal{B})  = \log\int p(\mathcal{A}|\bm{z}_{\mathbb{L}})p(\mathcal{B}|\bm{z}_{\mathbb{L}})p_{\bm{\theta}}(\mathcal{D}_{\mathbb{L}}|\bm{z}_{\mathbb{L}})p_{\bm{\theta}}(\mathcal{D}_{\mathbb{U}}|\bm{z}_{\mathbb{U}})p(\bm{z}_{\mathbb{L}})p(\bm{z}_{\mathbb{U}})d\bm{z}_{\mathbb{L}}d\bm{z}_{\mathbb{U}}
\end{align*}
where marginalisation is over the collection of latent points  $\bm{z}_{\mathbb{L}} = \langle \bm{z}^{(\text{l})}_n\rangle^N_{n=1}$ and $\bm{z}_{\mathbb{U}} = \langle\bm{z}^{(\text{u})}_m\rangle^M_{m=1}$. Adopting the weight function $w(\bm{x}^{(\text{l})})\propto f(\bm{x}^{(\text{l})})$ for labelled input datapoints $\bm{x}^{(\text{l})}\in\mathcal{D}_{\mathbb{L}}$ from ~\cite{2020_Tripp} and utilising a weighted log-likelihood formulation~\cite{WeightedLogLikelihood}:
\begin{align*}
    &\log p_{\bm{\phi},\bm{\theta}}(\mathcal{D}_{\mathbb{U}}, \mathcal{D}_{\mathbb{L}},  \mathcal{A},\mathcal{B}) = \log\bigg[\int\mathcal{G}\prod_{n=1}^N\left[p_{\bm{\theta}}(\bm{x}^{(\text{l})}_n|\bm{z}^{(\text{l})}_n)p_{\bm{\theta}}(f(\bm{x}^{(\text{l})}_n)|\bm{z}^{(\text{l})}_n)p(\bm{z}^{(\text{l})}_n)\right]^{w(\bm{x}^{(\text{l})}_n)} \prod_{m=1}^Mp_{\bm{\theta}}(\bm{x}^{(\text{u})}_m|\bm{z}^{(\text{u})}_m)p(\bm{z}^{(\text{u})}_{m})d\bm{z}_{\mathbb{L}}d\bm{z}_{\mathbb{U}}\bigg], 
\end{align*}
where $\mathcal{G} = \left[\prod^{N,N}_{i,j=1}p(a_{ij}|\bm{z}^{(\text{l})}_{i,j})p(b_{ij}|\bm{z}^{(\text{l})}_{i,j})\right]^{w_{i,j}}$ with $w_{i,j} = w(\bm{x}^{(\text{l})}_i)w(\bm{x}^{(\text{l})}_j)$ and $\bm{z}^{(\text{l})}_{i,j} = \langle\bm{z}^{(\text{l})}_i, \bm{z}^{(\text{l})}_j\rangle$. Introducing weighted variational distributions $q_{\bm{\phi}}(\bm{z}_{\mathbb{L}}, \bm{z}_{\mathbb{U}}|\mathcal{D}_{\mathbb{L}}, \mathcal{D}_{\mathbb{U}}) = q^{(\mathbb{L})}_{\bm{\phi}}(\bm{z}_{\mathbb{L}}|\mathcal{D}_{\mathbb{L}})q^{(\mathbb{U})}_{\bm{\phi}}(\bm{z}_{\mathbb{U}}|\mathcal{D}_{\mathbb{U}})$, where $q^{(\mathbb{L})}_{\bm{\phi}}(\bm{z}_{\mathbb{L}}|\mathcal{D}_{\mathbb{L}}) = \prod_{n=1}^N\left[q^{(\mathbb{L})}_{\bm{\phi}}(\bm{z}^{(\text{l})}_n|\bm{x}^{(\text{l})}_n)\right]^{w(\bm{x}^{(\text{l})}_n)}$ and $q^{(\mathbb{U})}_{\bm{\phi}}(\bm{z}_{\mathbb{U}}|\mathcal{D}_{\mathbb{U}}) = \prod_{m=1}^{M}q^{(\mathbb{U})}_{\bm{\phi}}(\bm{z}^{(\text{u})}_{m}|\bm{x}^{(\text{u})}_m)$ we obtain:
\begin{align*}
    \log & p_{\bm{\phi},\bm{\theta}}(\mathcal{D}_{\mathbb{U}}, \mathcal{D}_{\mathbb{L}},  \mathcal{A},\mathcal{B}) = \log\Bigg[\int\mathcal{G}\prod_{n=1}^N\left[\frac{p_{\bm{\theta}}(\bm{x}^{(\text{l})}_n|\bm{z}^{(\text{l})}_n)p_{\bm{\theta}}(f(\bm{x}^{(\text{l})}_n)|\bm{z}^{(\text{l})}_n)p(\bm{z}^{(\text{l})}_n)}{q^{(\mathbb{L})}_{\bm{\phi}}(\bm{z}^{(\text{l})}_n|\bm{x}^{(\text{l})}_n)}\right]^{w(\bm{x}^{(\text{l})}_n)}\times\\\nonumber
    &\hspace{25em}\prod_{m=1}^M\frac{p_{\bm{\theta}}(\bm{x}^{(\text{u})}_m|\bm{z}^{(\text{u})}_{m})p(\bm{z}^{(\text{u})}_{m})}{q^{(\mathbb{U})}_{\bm{\phi}}(\bm{z}^{(\text{u})}_{m}|\bm{x}^{(\text{u})}_m)}q^{(\mathbb{L})}_{\bm{\phi}}(\bm{z}_{\mathbb{L}}|\mathcal{D}_{\mathbb{L}})q^{(\mathbb{U})}_{\bm{\phi}}(\bm{z}_{\mathbb{U}}|\mathcal{D}_{\mathbb{U}})d\bm{z}_{\mathbb{L}}d\bm{z}_{\mathbb{U}}\Bigg].
\end{align*}
Using Jensen's inequality:
\begin{align*}
    \log p_{\bm{\phi},\bm{\theta}}(\mathcal{D}_{\mathbb{U}}, \mathcal{D}_{\mathbb{L}},  \mathcal{A},\mathcal{B}) &\ge \int\sum_{m=1}^M\log\left[\frac{p_{\bm{\theta}}(\bm{x}^{(\text{u})}_m|\bm{z}^{(\text{u})}_{m})p(\bm{z}^{(\text{u})}_{m})}{q^{(\mathbb{U})}_{\bm{\phi}}(\bm{z}^{(\text{u})}_{m}|\bm{x}^{(\text{u})}_m)}\right]q^{(\mathbb{U})}_{\bm{\phi}}(\bm{z}_{\mathbb{U}}|\mathcal{D}_{\mathbb{U}})d\bm{z}_{\mathbb{U}} \\\nonumber
    &+\int\sum_{n=1}^Nw(\bm{x}^{(\text{l})}_n)\log\left[\frac{p_{\bm{\theta}}(\bm{x}^{(\text{l})}_n|\bm{z}^{(\text{l})}_n)p_{\bm{\theta}}(f(\bm{x}^{(\text{l})}_n)|\bm{z}^{(\text{l})}_n)p(\bm{z}^{(\text{l})}_n)}{q^{(\mathbb{L})}_{\bm{\phi}}(\bm{z}^{(\text{l})}_n|\bm{x}^{(\text{l})}_n)}\right] q^{(\mathbb{L})}_{\bm{\phi}}(\bm{z}_{\mathbb{L}}|\mathcal{D}_{\mathbb{L}})d\bm{z}_{\mathbb{L}}\\\nonumber
    &+\int\left[\sum^{N,N}_{i,j=1}w_{i,j}\log\left[p(a_{ij}|\bm{z}^{(\text{l})}_{i,j})p(b_{ij}|\bm{z}^{(\text{l})}_{i,j})\right]\right]q^{(\mathbb{L})}_{\bm{\phi}}(\bm{z}_{\mathbb{L}}|\mathcal{D}_{\mathbb{L}})d\bm{z}_{\mathbb{L}},
\end{align*}
and rewriting using expectation operators and the $\text{KL}$ divergence:
\begin{align*}
    \log  p_{\bm{\phi},\bm{\theta}}(\mathcal{D}_{\mathbb{U}}, \mathcal{D}_{\mathbb{L}},  \mathcal{A},\mathcal{B}) &\ge\sum_{n=1}^N w(\bm{x}_{n}^{(\text{l})})\Big[\mathbb{E}_{q^{(\mathbb{L})}_{\bm{\phi}}(\bm{z}_{n}^{(\text{l})}|\bm{x}_{n}^{(\text{l})})}[\log p_{\bm{\theta}}(\bm{x}_n^{(\text{l})}|\bm{z}_{n}^{(\text{l})})+ \log p_{\bm{\theta}}(f(\bm{x}_{n}^{(\text{l})})|\bm{z}_{n}^{(\text{l})})]-\text{KL}(q^{(\mathbb{L})}_{\bm{\phi}}(\bm{z}^{(\text{l})}_{n}|\bm{x}^{(\text{l})}_{n})||p(\bm{z}_{n}^{(\text{l})}))\Big] \\\nonumber
    & + \sum_{m=1}^M\Big[\mathbb{E}_{q^{(\mathbb{U})}_{\bm{\phi}}(\bm{z}_{m}^{(\text{u})}|\bm{x}^{(\text{u})}_m)}\left[\log[p_{\bm{\theta}}(\bm{x}^{(\text{u})}_m|\bm{z}^{(\text{u})}_{m})]\right]- \text{KL}(q^{(\mathbb{U})}_{\bm{\phi}}(\bm{z}_{m}^{(\text{u})}|\bm{x}^{(\text{u})}_{m})||p(\bm{z}^{(\text{u})}_{m})) \Big] \\\nonumber
    & + \sum_{i,j=1}^{N,N}w_{i,j}\bigg[\mathbb{E}_{q^{(\mathbb{L})}_{\bm{\phi}}(\bm{z}^{(\text{l})}_{i,j}|\bm{x}^{(\text{l})}_{i,j})}\left[\log(p(a_{ij}|\bm{z}^{(\text{l})}_{i,j}))\right]+ \mathbb{E}_{q^{(\mathbb{L})}_{\bm{\phi}}(\bm{z}^{(\text{l})}_{i,j}|\bm{x}^{(\text{l})}_{i,j})}\left[\log(p(b_{ij}|\bm{z}^{(\text{l})}_{i,j}))\right]\bigg].
\end{align*}
Now, using the form of probability distribution for Bernoulli random variables $a_{ij}$ and $b_{ij}$ we have (considering cases with $a_{ij} = 1$ and $b_{ij} = 1$, similar to~\cite{karaletsos2016}):
\begin{align*}
    \mathbb{E}_{q^{(\mathbb{L})}_{\bm{\phi}}(\bm{z}^{(\text{l})}_{i,j}|\bm{x}^{(\text{l})}_{i,j})}\left[\log(p(a_{ij}|\bm{z}^{(\text{l})}_{i,j}))\right] &= (1 - a_{ij})\mathbb{E}_{q^{(\mathbb{L})}_{\bm{\phi}}(\bm{z}^{(\text{l})}_{i,j}|\bm{x}^{(\text{l})}_{i,j})}\left[\log\left[1 - e^{-\mathcal{L}^{(\text{BO})}_{\text{cont}, 1}(\bm{z}^{(\text{l})}_{i,j})}\right]\right]  -a_{ij}\mathbb{E}_{q^{(\mathbb{L})}_{\bm{\phi}}(\bm{z}^{(\text{l})}_{i,j}|\bm{x}^{(\text{l})}_{i,j})}\left[\mathcal{L}^{(\text{BO})}_{\text{cont}, 1}(\bm{z}^{(\text{l})}_{i,j})\right] \\
    &= -\mathbb{E}_{q^{(\mathbb{L})}_{\bm{\phi}}(\bm{z}^{(\text{l})}_{i,j}|\bm{x}^{(\text{l})}_{i,j})}\left[\mathcal{L}^{(\text{BO})}_{\text{cont}, 1}(\bm{z}^{(\text{l})}_{i,j})\right],\\\nonumber
    \mathbb{E}_{q^{(\mathbb{L})}_{\bm{\phi}}(\bm{z}^{(\text{l})}_{i,j}|\bm{x}^{(\text{l})}_{i,j})}\left[\log(p(b_{ij}|\bm{z}^{(\text{l})}_{i,j}))\right] &=  (1 - b_{ij})\mathbb{E}_{q^{(\mathbb{L})}_{\bm{\phi}}(\bm{z}^{(\text{l})}_{i,j}|\bm{x}^{(\text{l})}_{i,j})}\left[\log\left[1 - e^{-\mathcal{L}^{(\text{BO})}_{\text{cont}, 2}(\bm{z}^{(\text{l})}_{i,j})}\right]\right] -b_{ij}\mathbb{E}_{q^{(\mathbb{L})}_{\bm{\phi}}(\bm{z}^{(\text{l})}_{i,j}|\bm{x}^{(\text{l})}_{i,j})}\left[\mathcal{L}^{(\text{BO})}_{\text{cont}, 2}(\bm{z}^{(\text{l})}_{i,j})\right] \\
    &= -\mathbb{E}_{q^{(\mathbb{L})}_{\bm{\phi}}(\bm{z}^{(\text{l})}_{i,j}|\bm{x}^{(\text{l})}_{i,j})}\left[\mathcal{L}^{(\text{BO})}_{\text{cont}, 2}(\bm{z}^{(\text{l})}_{i,j})\right].
\end{align*}
Combining these results gives the expression for the composite ELBO objective:
\begin{align*}
    \textbf{ELBO}_{\textbf{DML}}(\bm{\phi}, \bm{\theta}|\mathcal{D}_{\mathbb{L}}, \mathcal{D}_{\mathbb{U}}, \mathcal{A}, \mathcal{B}) &=\sum_{n=1}^N w(\bm{x}_{n}^{(\text{l})})\bigg[\mathbb{E}_{q^{(\mathbb{L})}_{\bm{\phi}}(\bm{z}_{n}^{(\text{l})}|\bm{x}_{n}^{(\text{l})})}[\log p_{\bm{\theta}}(\bm{x}_n^{(\text{l})}|\bm{z}_{n}^{(\text{l})})\\
    & \hspace{10em} +\log p_{\bm{\theta}}(f(\bm{x}_{n}^{(\text{l})})|\bm{z}_{n}^{(\text{l})})]-\text{KL}(q^{(\mathbb{L})}_{\bm{\phi}}(\bm{z}^{(\text{l})}_{n}|\bm{x}^{(\text{l})}_{n})||p(\bm{z}_{n}^{(\text{l})}))\bigg] \\\nonumber
    & + \sum_{m=1}^M\bigg[\mathbb{E}_{q^{(\mathbb{U})}_{\bm{\phi}}(\bm{z}_{m}^{(\text{u})}|\bm{x}^{(\text{u})}_m)}\left[\log[p_{\bm{\theta}}(\bm{x}^{(\text{u})}_m|\bm{z}^{(\text{u})}_{m})]\right] -\text{KL}(q^{(\mathbb{U})}_{\bm{\phi}}(\bm{z}_{m}^{(\text{u})}|\bm{x}^{(\text{u})}_{m})||p(\bm{z}^{(\text{u})}_{m})) \bigg] \\\nonumber
    &\underbrace{- \sum_{i,j=1}^{N,N}w_{i,j}\mathbb{E}_{q^{(\mathbb{L})}_{\bm{\phi}}(\bm{z}^{(\text{l})}_{i,j}|\bm{x}^{(\text{l})}_{i,j})}\left[\mathcal{L}^{(\text{BO})}_{\text{s-cont}}(\bm{z}^{(\text{l})}_{i,j})\right]}_{\text{Com}_{\text{metric=s-cont}}(\cdot)}.
\end{align*}
An inspection of the objective uncovers two familiar components: the first term is $\text{Com}_{\text{unlabelled}}(\cdot)$ the standard variational ELBO objective~\cite{kingma2014} and the second term is $\text{Com}_{\text{label}}(\cdot)$ the weighted ELBO objective from~\cite{2020_Tripp} endowed with black-box function observations. Finally, $\text{Com}_{\text{metric=s-cont}}(\cdot)$ is a novel contrastive loss-based amendment responsible for the construction of the latent space.

\subsection{Triplet Loss VAE \& $\mathcal{L}_{\text{triple}}$}\label{App: Triplet_loss_section}
\paragraph{Triplet Loss:} The triplet loss $\mathcal{L}_{\text{triple}}(\cdot)$, differs from $\mathcal{L}_{\text{cont.}}(\cdot)$ in that it measures distances between input triplets rather than input pairs $\bm{x}_{i}, \bm{x}_{j}$. To define $\mathcal{L}_{\text{triple}}(\cdot)$, we require an anchor/base input (e.g., an image of a dog) $\bm{x}^{(\text{b})}$, a positive input (e.g., a rotated image of a dog) $\bm{x}^{(\text{p})}$ and a negative input (e.g., an image of a cat) $\bm{x}^{(\text{n})}$. Given a separation margin $\rho$, we map to encodings $\bm{z}^{(\text{b})}$, $\bm{z}^{(\text{p})}$ and $\bm{z}^{(\text{n})}$ such that: $||\bm{z}^{(\text{b})} - \bm{z}^{(\text{p})} ||_{q} + \rho\leq ||\bm{z}^{(\text{b})} - \bm{z}^{(\text{n})} ||_{q}$. Consequently, minimising  $\mathcal{L}_{\text{triple}}(\cdot)= \max\left\{0, ||\bm{z}^{(\text{b})} - \bm{z}^{(\text{p})} ||_{q} + \rho - ||\bm{z}^{(\text{b})} - \bm{z}^{(\text{n})} ||_{q} \right\}$ yields a structured space where positive and negative pairs cluster together subject to separation by a margin $\rho$.

\paragraph{Soft Triplet Loss:}
As shown in Figure \ref{Fig:SoftTriplet}, the triplet loss $\mathcal{L}_{\text{triple}}(\cdot)$ exhibits discontinuous behaviour around the planes $|f(\bm{x}_a) - f(\bm{x}_p)| = \eta$ and $|f(\bm{x}_a) - f(\bm{x}_n)| = \eta $, which can be detrimental for GP regression. To remedy this issue, we introduce a soft version of the triplet loss function by considering the following penalty measure for a given latent anchor point $\bm{z}_i$ and $\bm{z}_j, \bm{z}_k$:  
\begin{align}\label{eq:soft-triplet}
    \mathcal{L}^{(\text{BO})}_{\text{s-triple}}&(\bm{z}_i, \bm{z}_j, \bm{z}_k) = \log\left[1 + e^{\Delta^{+}_{\bm{z}} - \Delta^{-}_{\bm{z}}}\right]w^{(\text{p})}w^{(\text{n})}\times \mathbbm{1}_{\{|f(\bm{x}_i) - f(\bm{x}_j)| < \eta \ \& \ |f(\bm{x}_i) - f(\bm{x}_k)| \ge \eta\}},
\end{align}
where $\Delta^{+}_{\bm{z}} = ||\bm{z}_i - \bm{z}_j||_q$, $\Delta^{-}_{\bm{z}} = ||\bm{z}_i - \bm{z}_k||_q$ and $w^{(\text{p})} = \frac{f_{\nu}(\eta - |f(\bm{x}_i) - f(\bm{x}_{j})|)}{f_{\nu}(\eta)}$, $w^{(\text{n})} = \frac{f_{\nu}( |f(\bm{x}_i) - f(\bm{x}_{k})| - \eta)}{f_{\nu}(1 - \eta)}$ are weight measures associated with points  $\bm{x}_j\in\mathcal{D}_{\text{p}}(\bm{x}_i;\eta)$ and $\bm{x}_k\in\mathcal{D}_{\text{n}}(\bm{x}_i;\eta)$ respectively. $f_{\nu}(a) = \text{tanh}(a/2\nu)$ is a smoothing function with $\nu$ a temperature parameter such that if $\lim_{\nu\to 0}f_{\nu}(a) = 1$, $\mathcal{L}^{(\text{BO})}_{\text{s-triple}}(\bm{z}_i, \bm{z}_j, \bm{z}_k)$ approaches the standard triplet loss. Intuitively, this function encourages points with similar function values to the anchor to be close to it and points with dissimilar function values to be distant from the anchor. The weight $w^{(\text{p})} \propto \eta - |f(\bm{x}_i) - f(\bm{x}_j)|$ assigns higher weight for points in $\mathcal{D}_{\text{p}}(\bm{x}_i;\eta)$ that have close function values to the anchor $f(\bm{x}_i)$ and weight $w^{(\text{n})}\propto |f(\bm{x}_i) - f(\bm{x}_k)| - \eta $ assigns higher weight for points in $\mathcal{D}_{\text{n}}(\bm{x}_i;\eta)$ that have distant function values to the anchor $f(\bm{x}_i)$. These weights allow us to smooth the penalty function around the planes $|f(\bm{x}_i) - f(\bm{x}_j)| = \eta$ and $|f(\bm{x}_i) - f(\bm{x}_k)| = \eta$ (cf. Figure \ref{Fig:SoftTriplet}).
\begin{figure*}
\centering
\subfloat[Triplet loss map as a function of  negative pair embedding distance $||\bm{z}_a - \bm{z}_n||$ and negative pair label distance $|f(\bm{x}_a) - f(\bm{x}_n)|$ with (left) and without (right) incorporation of softening weights.]{\includegraphics[scale=.25]{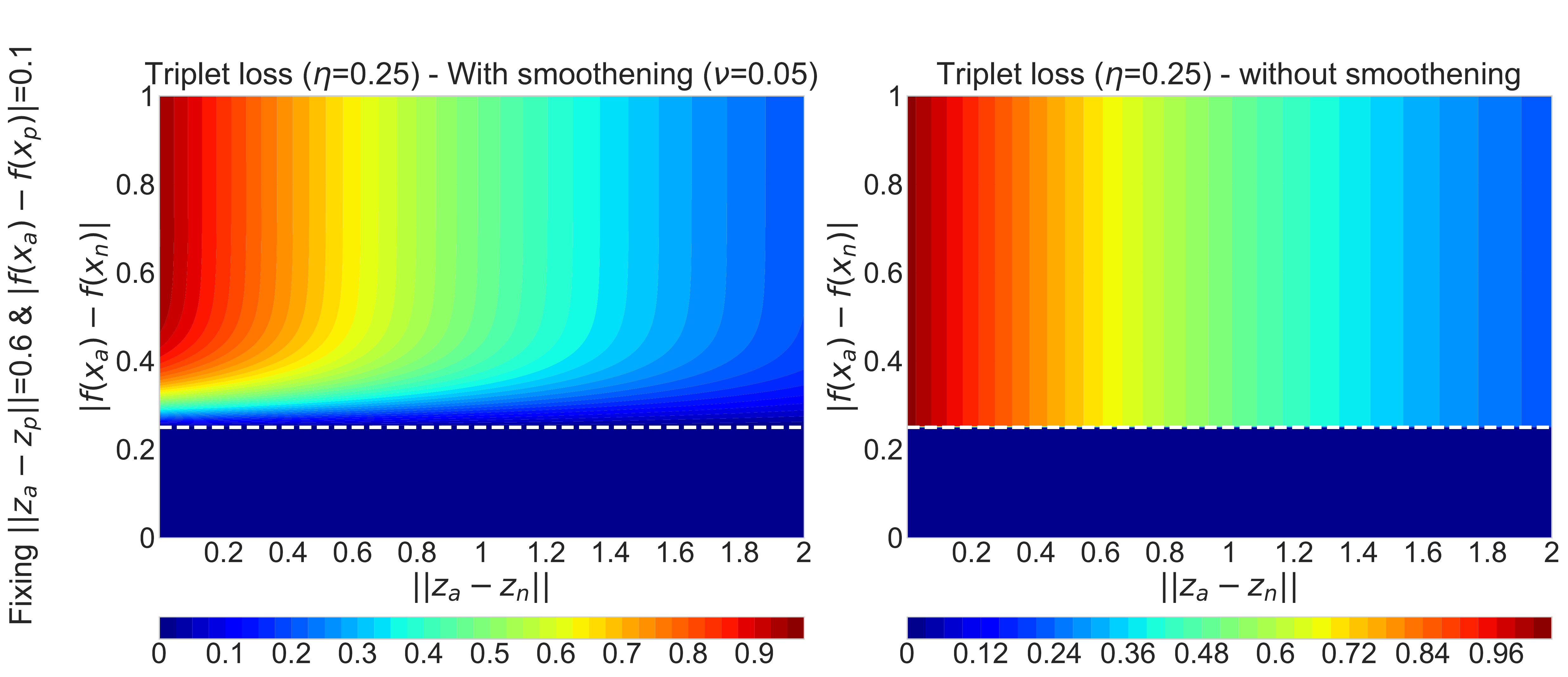}}
\par
\centering
\subfloat[Triplet loss map as a function of  positive pair embedding distance $||\bm{z}_a - \bm{z}_p||$ and positive pair label distance $|f(\bm{x}_a) - f(\bm{x}_p)|$ with (left) and without (right) incorporation of softening weights.]{\label{main:c}\includegraphics[scale=.25]{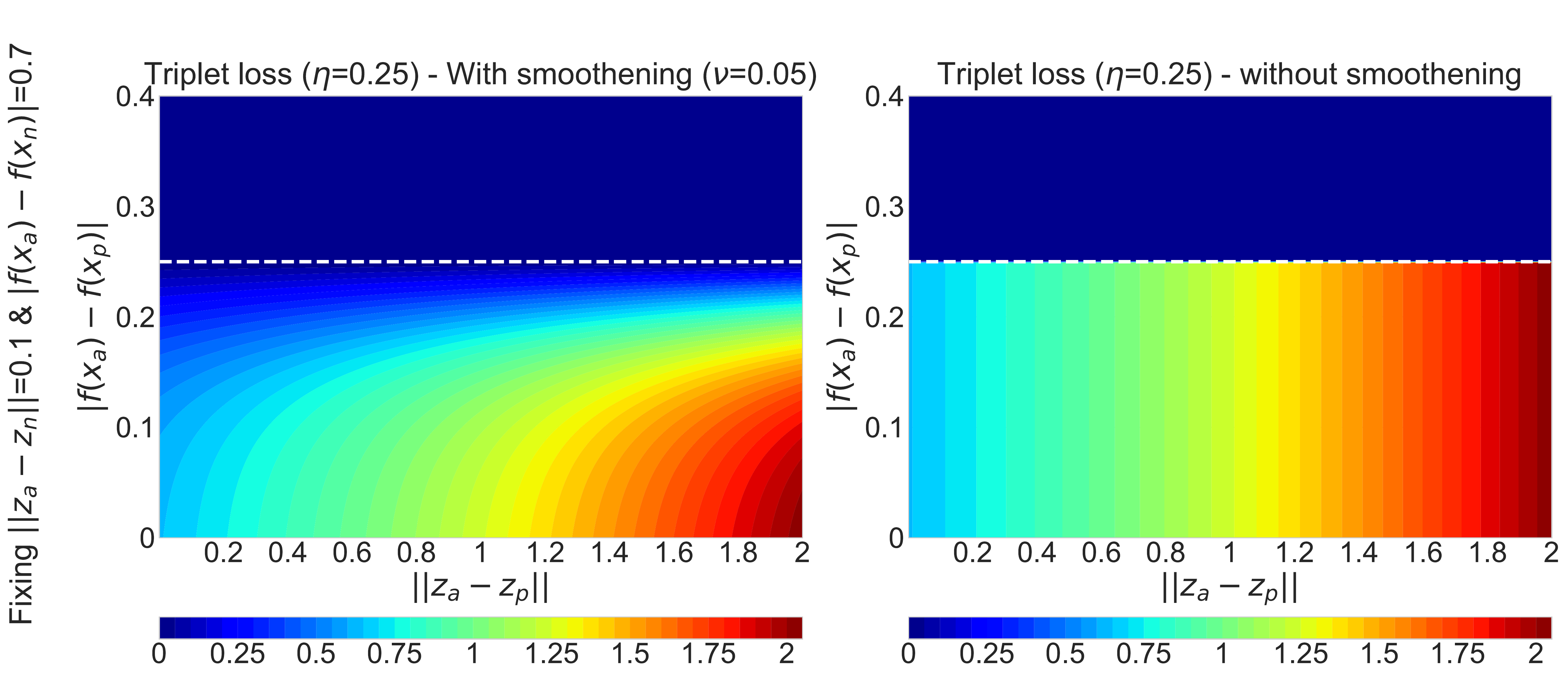}}
\caption{\textbf{Soft Triplet Loss}. The right figure shows the discontinuity in the original class triplet loss given in the main paper by equation (\textcolor{blue}{2}) in the absence of the softening mechanism presented in equation (\ref{eq:soft-triplet}). The discontinuity appears at the level $|f(\bm{x}_a) - f(\bm{x}_p)| = \eta$  (resp. $|f(\bm{x}_a) - f(\bm{x}_n)| = \eta$) corresponding to the limit beyond which $\bm{x}_p$ (resp. $\bm{x}_n$) no longer belongs to the set of positive (resp. negative) datapoints with respect to an anchor point $\bm{x}_a$. The left figures demonstrates that the softening mechanism enables a continuous transition between these two regimes.}
\label{Fig:SoftTriplet}
\end{figure*} 
Indeed, without $w^{(\text{p})}$ and $w^{(\text{n})}$, magnitudes of $|f(\bm{x}_i) - f(\bm{x}_j)|$ and $|f(\bm{x}_i) - f(\bm{x}_k)|$ do not affect the penalty measure and only their relation to parameter $\eta$ is relevant. Presence of weights  $w^{(\text{p})}$ and $w^{(\text{n})}$, on the other hand, allow us to penalise points in the latent space based on the distance between their associated function values. In particular, considering the limits $\lim |f(\bm{x}_i) - f(\bm{x}_j)|\to \eta^{-}$ and $\lim |f(\bm{x}_i) - f(\bm{x}_k)| \to \eta^{+}$ the penalty measure converges to $0$ with the rate controlled by parameter $\eta$.

\paragraph{Variational Soft Triplet Loss:}
To construct the joint latent model shaped by triplet loss, for each datapoint $\bm{x}_i\in\mathcal{D}_{\mathbb{L}}$ we use the definition of positive and negative datapoints with respect to $\bm{x}_i$: 
\begin{align*}
    &\mathcal{D}_{\text{p}}(\bm{x}_i;\eta) = \langle\bm{x}\in\mathcal{D}_{\mathbb{L}}: |f(\bm{x}_i) - f(\bm{x})| < \eta\rangle \ \ \ \  \text{and} \ \ \ \mathcal{D}_{\text{n}}(\bm{x}_i;\eta) = \langle\bm{x}\in\mathcal{D}_{\mathbb{L}}: |f(\bm{x}_i) - f(\bm{x})| \ge \eta\rangle,
\end{align*}
with $\eta > 0$ being a proximity hyperparameter. For each ordered triple of input points  $\bm{x}_{i},\bm{x}_j,\bm{x}_k\in\mathcal{D}_{\mathbb{L}}$ we consider a Bernoulli random variable $c_{i,j,k} = \mathbbm{1}_{\{|f(\bm{x}_i) - f(\bm{x}_j)| < \eta \& |f(\bm{x}_i) - f(\bm{x}_k)|\ge \eta\}}$. Given the latent representations $\bm{z}_i, \bm{z}_j, \bm{z}_k\sim q_{\bm{\phi}}(\cdot,\cdot,\cdot|\bm{x}_i,\bm{x}_j,\bm{x}_k)$ for input points $\bm{x}_i$, $\bm{x}_j$, $\bm{x}_k$ respectively, the probability distribution for random variable $c_{i,j,k}$ is given as:
\begin{align*}
    &\mathbb{P}[c_{i,j,k} = 1|\bm{z}_i, \bm{z}_j, \bm{z}_k] = e^{-\mathcal{L}^{(\text{BO})}_{\text{s-triple}}(\bm{z}_i, \bm{z}_j, \bm{z}_k))},
\end{align*}
where $\mathcal{L}^{(\text{BO})}_{\text{s-triple}}(\bm{z}_i, \bm{z}_j, \bm{z}_k))$ is the softening triplet loss function for continuous support defined in Equation (\ref{eq:soft-triplet}). In other words, the random variable $c_{i,j,k}$ is more likely to take on the value $1$ for the anchor point $\bm{z}_i$ and points $\bm{z}_j\in\mathcal{D}_{\text{p}}(\bm{x}_i;\eta)$,$\bm{z}_k\in\mathcal{D}_{\text{n}}(\bm{x}_i;\eta)$ if point $\bm{z}_j$ with a small function distance $|f(\bm{x}_i) - f(\bm{x}_j)|$ is much closer (i.e (i.e.$\Delta^+_{\bm{z}} \ll \Delta^-_{\bm{z}}$) to the anchor than point $\bm{z}_k$ with a large function distance $|f(\bm{x}_i) - f(\bm{x}_k)|$. It is again important to note that dataset $\mathcal{D}_{\mathbb{L}} = \langle\bm{x}^{(\text{l})}_n, f(\bm{x}^{(\text{l})}_n)\rangle^N_{n=1}$ provides us with realisations of random variables $\mathcal{C} = \langle c_{i,j,k}\rangle^{N,N,N}_{i,j,k=1}$ for all ordered triplets $\bm{x}^{(\text{l})}_{i,j,k} = \langle\bm{x}^{(\text{l})}_i,\bm{x}^{(\text{l})}_j,\bm{x}^{(\text{l})}_k\rangle$. Combining these realisations with unlabelled data $\mathcal{D}_{\mathbb{U}} = \langle\bm{x}^{(\text{u})}_m\rangle^M_{m=1}$, the joint log-likelihood is given as (following marginalisation over latent points $\bm{z}_{\mathbb{L}} = \langle \bm{z}^{(\text{l})}_n\rangle^N_{n=1}$ and $\bm{z}_{\mathbb{U}} = \langle\bm{z}^{(\text{u})}_m\rangle^M_{m=1}$):
\begin{align*}
    \log & p_{\bm{\phi},\bm{\theta}}(\mathcal{D}_{\mathbb{L}}, \mathcal{D}_{\mathbb{U}},  \mathcal{C}) = \log\int p_{\bm{\theta}}(\mathcal{D}_{\mathbb{L}}|\bm{z}_{\mathbb{L}}) \times p_{\bm{\theta}}(\mathcal{C}|\bm{z}_{\mathbb{L}})p_{\bm{\theta}}(\mathcal{D}_{\mathbb{U}}|\bm{z}_{\mathbb{U}})p(\bm{z}_{\mathbb{U}})p(\bm{z}_{\mathbb{L}})d\bm{z}_{\mathbb{L}}d\bm{z}_{\mathbb{U}}
\end{align*}
Adopting the weight function $w(\bm{x}^{(\text{l})})\propto f(\bm{x}^{(\text{l})})$ for labelled input datapoints $\bm{x}^{(\text{l})}\in\mathcal{D}_{\mathbb{L}}$ from~\cite{2020_Tripp} and utilising a weighted log-likelihood formulation~\cite{WeightedLogLikelihood}:
\begin{align*}
    \log p_{\bm{\phi},\bm{\theta}}(\mathcal{D}_{\mathbb{L}}, \mathcal{D}_{\mathbb{U}},  \mathcal{C}) = \log\Bigg[\int\mathcal{H}\prod_{n=1}^N\left[p_{\bm{\theta}}(\bm{x}^{(\text{l})}_n|\bm{z}^{(\text{l})}_n)p_{\bm{\theta}}(f(\bm{x}^{(\text{l})}_n)|\bm{z}^{(\text{l})}_n)p(\bm{z}^{(\text{l})}_n)\right]^{w(\bm{x}^{(\text{l})}_n)} \prod_{m=1}^Mp_{\bm{\theta}}(\bm{x}^{(\text{u})}_m|\bm{z}^{(\text{u})}_m)p(\bm{z}^{(\text{u})}_{m})d\bm{z}_{\mathbb{L}}d\bm{z}_{\mathbb{U}}\Bigg], 
\end{align*}
where $\mathcal{H} = \left[\prod^{N,N,N}_{i,j,k=1}p(c_{i,j,k}|\bm{z}^{(\text{l})}_i, \bm{z}^{(\text{l})}_j,\bm{z}^{(\text{l})}_k)\right]^{w_{i,j,k}}$ with $w_{i,j,k} = w(\bm{x}^{(\text{l})}_i)w(\bm{x}^{(\text{l})}_j)w(\bm{x}^{(\text{l})}_k)$. Introducing weighted variational distributions $q_{\bm{\phi}}(\bm{z}_{\mathbb{L}}, \bm{z}_{\mathbb{U}}|\mathcal{D}_{\mathbb{L}}, \mathcal{D}_{\mathbb{U}}) = q^{(\mathbb{L})}_{\bm{\phi}}(\bm{z}_{\mathbb{L}}|\mathcal{D}_{\mathbb{L}})q^{(\mathbb{U})}_{\bm{\phi}}(\bm{z}_{\mathbb{U}}|\mathcal{D}_{\mathbb{U}})$ where $q^{(\mathbb{L})}_{\bm{\phi}}(\bm{z}_{\mathbb{L}}|\mathcal{D}_{\mathbb{L}}) = \prod_{n=1}^N\left[q^{(\mathbb{L})}_{\bm{\phi}}(\bm{z}^{(\text{l})}_n|\bm{x}^{(\text{l})}_n)\right]^{w(\bm{x}^{(\text{l})}_n)}$ and $q^{(\mathbb{U})}_{\bm{\phi}}(\bm{z}_{\mathbb{U}}|\mathcal{D}_{\mathbb{U}}) = \prod_{m=1}^{M}q^{(\mathbb{U})}_{\bm{\phi}}(\bm{z}^{(\text{u})}_{m}|\bm{x}^{(\text{u})}_m)$ we have:
\begin{align*}
    \log p_{\bm{\phi},\bm{\theta}}(\mathcal{D}_{\mathbb{L}}, \mathcal{D}_{\mathbb{U}},  \mathcal{C}) &=  \log\Bigg[\int\mathcal{H}\prod_{n=1}^N\left[\frac{p_{\bm{\theta}}(\bm{x}^{(\text{l})}_n|\bm{z}^{(\text{l})}_n)p_{\bm{\theta}}(f(\bm{x}^{(\text{l})}_n)|\bm{z}^{(\text{l})}_n)p(\bm{z}^{(\text{l})}_n)}{q^{(\mathbb{L})}_{\bm{\phi}}(\bm{z}^{(\text{l})}_n|\bm{x}^{(\text{l})}_n)}\right]^{w(\bm{x}^{(\text{l})}_n)}\\
    & ~~~~ \times \prod_{m=1}^M\frac{p_{\bm{\theta}}(\bm{x}^{(\text{u})}_m|\bm{z}^{(\text{u})}_{m})p(\bm{z}^{(\text{u})}_{m})}{q^{(\mathbb{U})}_{\bm{\phi}}(\bm{z}^{(\text{u})}_{m}|\bm{x}^{(\text{u})}_m)}q^{(\mathbb{L})}_{\bm{\phi}}(\bm{z}_{\mathbb{L}}|\mathcal{D}_{\mathbb{L}})q^{(\mathbb{U})}_{\bm{\phi}}(\bm{z}_{\mathbb{U}}|\mathcal{D}_{\mathbb{U}})d\bm{z}_{\mathbb{L}}d\bm{z}_{\mathbb{U}}\Bigg].
\end{align*}
Applying Jensen's inequality:
\begin{align*}
    \log p_{\bm{\phi},\bm{\theta}}(\mathcal{D}_{\mathbb{L}}, \mathcal{D}_{\mathbb{U}},  \mathcal{C}) \ge & \int\sum_{m=1}^M\log\left[\frac{p_{\bm{\theta}}(\bm{x}^{(\text{u})}_m|\bm{z}^{(\text{u})}_{m})p(\bm{z}^{(\text{u})}_{m})}{q^{(\mathbb{U})}_{\bm{\phi}}(\bm{z}^{(\text{u})}_{m}|\bm{x}^{(\text{u})}_m)}\right]q^{(\mathbb{U})}_{\bm{\phi}}(\bm{z}_{\mathbb{U}}|\mathcal{D}_{\mathbb{U}})d\bm{z}_{\mathbb{U}} \\
    & + \int\sum_{n=1}^Nw(\bm{x}^{(\text{l})}_n)\log\left[\frac{p_{\bm{\theta}}(\bm{x}^{(\text{l})}_n|\bm{z}^{(\text{l})}_n)p_{\bm{\theta}}(f(\bm{x}^{(\text{l})}_n)|\bm{z}^{(\text{l})}_n)p(\bm{z}^{(\text{l})}_n)}{q^{(\mathbb{L})}_{\bm{\phi}}(\bm{z}^{(\text{l})}_n|\bm{x}^{(\text{l})}_n)}\right] q^{(\mathbb{L})}_{\bm{\phi}}(\bm{z}_{\mathbb{L}}|\mathcal{D}_{\mathbb{L}})d\bm{z}_{\mathbb{L}}  \\\nonumber
    & + \int\left[\sum^{N,N,N}_{i,j,k=1}w_{i,j,k}\log\left[p(c_{ijk}|\bm{z}^{(\text{l})}_i, \bm{z}^{(\text{l})}_j, \bm{z}^{(\text{l})}_k)\right]\right] q^{(\mathbb{L})}_{\bm{\phi}}(\bm{z}_{\mathbb{L}}|\mathcal{D}_{\mathbb{L}})d\bm{z}_{\mathbb{L}},
\end{align*}
and rewriting using expectation operators and the $\text{KL}$ divergence:
\begin{align*}
    \log p_{\bm{\phi},\bm{\theta}}(\mathcal{D}_{\mathbb{L}}, \mathcal{D}_{\mathbb{U}},  \mathcal{C}) &\ge\sum_{n=1}^N w(\bm{x}_{n}^{(\text{l})})\bigg[\mathbb{E}_{q^{(\mathbb{L})}_{\bm{\phi}}(\bm{z}_{n}^{(\text{l})}|\bm{x}_{n}^{(\text{l})})}[\log p_{\bm{\theta}}(\bm{x}_n^{(\text{l})}|\bm{z}_{n}^{(\text{l})}) +\log p_{\bm{\theta}}(f(\bm{x}_{n}^{(\text{l})})|\bm{z}_{n}^{(\text{l})})]-\text{KL}(q^{(\mathbb{L})}_{\bm{\phi}}(\bm{z}^{(\text{l})}_{n}|\bm{x}^{(\text{l})}_{n})||p(\bm{z}_{n}^{(\text{l})}))\bigg] \\\nonumber
    &+\sum_{m=1}^M\bigg[\mathbb{E}_{q^{(\mathbb{U})}_{\bm{\phi}}(\bm{z}_{m}^{(\text{u})}|\bm{x}^{(\text{u})}_m)}\left[\log[p_{\bm{\theta}}(\bm{x}^{(\text{u})}_m|\bm{z}^{(\text{u})}_{m})]\right] - \text{KL}(q^{(\mathbb{U})}_{\bm{\phi}}(\bm{z}_{m}^{(\text{u})}|\bm{x}^{(\text{u})}_{m})||p(\bm{z}^{(\text{u})}_{m})) \bigg]\\\nonumber
    & + \sum_{i,j,k=1}^{N,N,N}w_{i,j,k}\bigg[\mathbb{E}_{q^{(\mathbb{L})}_{\bm{\phi}}(\bm{z}^{(\text{l})}_i,\bm{z}^{(\text{l})}_j,\bm{z}^{(\text{l})}_k|\bm{x}^{(\text{l})}_{i,j,k})}\Big[\log(p(c_{ijk}|\bm{z}^{(\text{l})}_i, \bm{z}^{(\text{l})}_j,\bm{z}^{(\text{l})}_k))\Big]\bigg].
\end{align*}
Now, using the form of probability distribution for Bernoulli random variables $c_{ijk}$ we have (considering cases with $c_{ijk} = 1$, similar to~\cite{karaletsos2016}):
\begin{align*}
    &\mathbb{E}_{q^{(\mathbb{L})}_{\bm{\phi}}(\bm{z}^{(\text{l})}_i, \bm{z}^{(\text{l})}_j, \bm{z}^{(\text{l})}_k|\bm{x}^{(\text{l})}_{i,k,k})}\left[\log(p(c_{ijk}|\bm{z}^{(\text{l})}_i, \bm{z}^{(\text{l})}_j, \bm{z}^{(\text{l})}_k))\right] =-\mathbb{E}_{q^{(\mathbb{L})}_{\bm{\phi}}(\bm{z}^{(\text{l})}_i,\bm{z}^{(\text{l})}_j, \bm{z}^{(\text{l})}_k|\bm{x}^{(\text{l})}_{i,j,k})}\left[\mathcal{L}^{(\text{BO})}_{\text{s-triple}}(\bm{z}^{(\text{l})}_i, \bm{z}^{(\text{l})}_j, \bm{z}^{(\text{l})}_k))\right].
\end{align*}
Combining these results gives the expression for the composite ELBO objective:
\begin{align*}
    \textbf{ELBO}_{\textbf{DML}}(\bm{\phi}, \bm{\theta}|\mathcal{D}_{\mathbb{L}}, \mathcal{D}_{\mathbb{U}}, \mathcal{C}) &=\sum_{n=1}^N w(\bm{x}_{n}^{(\text{l})})\Big[\mathbb{E}_{q^{(\mathbb{L})}_{\bm{\phi}}(\bm{z}_{n}^{(\text{l})}|\bm{x}_{n}^{(\text{l})})}\big[\log p_{\bm{\theta}}(\bm{x}_n^{(\text{l})}|\bm{z}_{n}^{(\text{l})})\\
    & \hspace{14em} +\log p_{\bm{\theta}}(f(\bm{x}_{n}^{(\text{l})})|\bm{z}_{n}^{(\text{l})})\big]-\text{KL}(q^{(\mathbb{L})}_{\bm{\phi}}(\bm{z}^{(\text{l})}_{n}|\bm{x}^{(\text{l})}_{n})||p(\bm{z}_{n}^{(\text{l})}))\Big]\\
    & + \sum_{m=1}^M\bigg[\mathbb{E}_{q^{(\mathbb{U})}_{\bm{\phi}}(\bm{z}_{m}^{(\text{u})}|\bm{x}^{(\text{u})}_m)}\left[\log[p_{\bm{\theta}}(\bm{x}^{(\text{u})}_m|\bm{z}^{(\text{u})}_{m})]\right] - \text{KL}(q^{(\mathbb{U})}_{\bm{\phi}}(\bm{z}_{m}^{(\text{u})}|\bm{x}^{(\text{u})}_{m})||p(\bm{z}^{(\text{u})}_{m})) \bigg]\\\nonumber
    &\underbrace{-\sum_{i,j,k=1}^{N,N,N}w_{i,j,k}\mathbb{E}_{q^{(\mathbb{L})}_{\bm{\phi}}(\bm{z}^{(\text{l})}_i,\bm{z}^{(\text{l})}_j, \bm{z}^{(\text{l})}_k|\bm{x}^{(\text{l})}_{i,j,k})}\left[\mathcal{L}^{(\text{BO})}_{\text{s-triple}}(\bm{z}^{(\text{l})}_i, \bm{z}^{(\text{l})}_j, \bm{z}^{(\text{l})}_k)\right]}_{\text{Comp}_{\text{metric=s-triple}}(\cdot)}.
\end{align*}
An inspection of the objective again uncovers two familiar components: the first term is $\text{Com}_{\text{unlabelled}}(\cdot)$ the standard variational ELBO objective~\cite{kingma2014} and the second term is  $\text{Com}_{\text{label}}(\cdot)$ the weighted ELBO objective from ~\cite{2020_Tripp} endowed with black-box function observations. Finally, $\text{Com}_{\text{metric=s-triple}}(\cdot)$ is a novel triplet loss-based amendment responsible for the construction of the latent space. 

\subsection{ELBO Components of Experiments from the Main Paper}\label{App:elbo_comp_details}
Briefly defined in the main paper, each acronym used in the experiments section corresponds to a different experimental setting. Re-iterating here, we wish to provide specific description of the component(s) constituting their respective ELBO as well as what components are used for pre-training and retraining. \autoref{table:warm_elbo_components} from the main text summarises the information from the following paragraphs.

\subsubsection{LBO}
The acronym \textbf{LBO} is used to describe the setting presented in \citep{2020_Tripp} in which all the available labelled data points $\mathcal{D}_{\mathbb{L}}$ are used to pre-train and retrain the VAE. Its ELBO is then simply the weighted ELBO from \citep{2020_Tripp} (i.e. $\text{Com}_{\text{label}}(\cdot)$)
\begin{align}\label{eq:LBO-ELBO}
    \sum_{n=1}^N w(\bm{x}_{n}^{(\text{l})})\bigg[\mathbb{E}_{q^{(\mathbb{L})}_{\bm{\phi}}(\bm{z}_{n}^{(\text{l})}|\bm{x}_{n}^{(\text{l})})}[\log p_{\bm{\theta}}(\bm{x}_n^{(\text{l})}|\bm{z}_{n}^{(\text{l})})]-\text{KL}(q^{(\mathbb{L})}_{\bm{\phi}}(\bm{z}^{(\text{l})}_{n}|\bm{x}^{(\text{l})}_{n})||p(\bm{z}_{n}^{(\text{l})}))\bigg].
\end{align}
\subsubsection{W-LBO}
The acronym \textbf{W-LBO} describes the same setting as \textbf{LBO} with respect to the VAE and its ELBO, i.e. during pre-training and retraining we seek to minimize ELBO~(\ref{eq:LBO-ELBO}). The difference resides in the way the surrogate model is built for the BO steps, as for \textbf{W-LBO} a parametrised input transformation known as input warping is performed and tuned during the fitting of the other parameters of the GP model. It is interesting to benchmark our approach against this method as it also acts on the model's input space (in our case the VAE latent space) and can be viewed as a space-shaping transformation aiming to improve the surrogate model fit \cite{2014_Snoek, 2020_Balandat, 2020_Cowen}.
\subsubsection{TP-LBO}
This setting is similar to \textbf{LBO} but also uses target prediction \citep{2018_Eissman}. The ELBO used for \textbf{TP-LBO} (Target-Prediction-LBO) is the one described in Section~\ref{Sec:NewVAE} which is a combination of a (weighted) VAE reconstruction loss and a (weighted) regression loss. The loss is $$\text{Com}_{\text{label}}(\bm{\theta}, \bm{\phi}) + w(\bm{x}_{i}) \mathbb{E}_{q_{\bm{\phi}}(\bm{z}_{i}|\bm{x}_i)}[\log h_{\bm{\theta}}(f(\bm{x}_{i})|\bm{z}_{i})]$$ with $h_{\bm{\theta}}(f(\bm{x}_{i})|\bm{z}_{i})$ being an additional decoder network (sharing parameters with $g_{\bm{\theta}}(\cdot)$) geared towards reconstructing $f(\bm{x}_i)$ in $\mathcal{D}_{\mathbb{L}}$.  

\subsubsection{C-LBO}
In \textbf{C-LBO} (Contrastive-LBO), as in previous settings, we have access to all the labelled points $\mathcal{D}_{\mathbb{L}}$ so the ELBO used for pre-training and retraining is a combination of $\text{Com}_{\text{label}}(\cdot)$ without target prediction ~(\ref{eq:LBO-ELBO}) and $\text{Com}_{\text{s-cont}}(\cdot)$, i.e. 
\begin{align*}
    \sum_{n=1}^N w(\bm{x}_{n}^{(\text{l})})\bigg[\mathbb{E}_{q^{(\mathbb{L})}_{\bm{\phi}}(\bm{z}_{n}^{(\text{l})}|\bm{x}_{n}^{(\text{l})})}[\log p_{\bm{\theta}}(\bm{x}_n^{(\text{l})}|\bm{z}_{n}^{(\text{l})})]-\text{KL}(q^{(\mathbb{L})}_{\bm{\phi}}(\bm{z}^{(\text{l})}_{n}|\bm{x}^{(\text{l})}_{n})||p(\bm{z}_{n}^{(\text{l})}))\bigg] + \text{Com}_{\text{s-cont}}(\cdot)
\end{align*}
as defined above in \ref{App:contrastive}.
\subsubsection{T-LBO}
The ELBO used in the setting \textbf{T-LBO} (Triplet-LBO), both for pre-training and retraining, is 
\begin{align*}
    \sum_{n=1}^N w(\bm{x}_{n}^{(\text{l})})\bigg[\mathbb{E}_{q^{(\mathbb{L})}_{\bm{\phi}}(\bm{z}_{n}^{(\text{l})}|\bm{x}_{n}^{(\text{l})})}[\log p_{\bm{\theta}}(\bm{x}_n^{(\text{l})}|\bm{z}_{n}^{(\text{l})})]-\text{KL}(q^{(\mathbb{L})}_{\bm{\phi}}(\bm{z}^{(\text{l})}_{n}|\bm{x}^{(\text{l})}_{n})||p(\bm{z}_{n}^{(\text{l})}))\bigg] + \text{Com}_{\text{s-triple}}(\cdot).
\end{align*}
Indeed, we have access to all labelled points here as well so we use $\text{Com}_{\text{label}}(\cdot)$ without target prediction~(\ref{eq:LBO-ELBO}) in conjunction with $\text{Com}_{\text{s-triple}}(\cdot)$ described in \ref{App: Triplet_loss_section}.

\subsection{Alternative Metrics}\label{App:AdditionalMetrics}
\subsubsection{S-LBO}
The simple loss makes use of the distance between function values to shape the latent space such that encoded points have similar distance, i.e. minimising $\text{Com}_{\text{metric}=\text{simple}} (\cdot, \cdot) \propto w_{ij} \mathbb{E}_{q_{\bm{\phi}}(\cdot)}\left[| \ ||\Delta \bm{z}_{ij}||-|\Delta f_{ij}| \ |\right]$ with $\Delta \bm{z}_{ij}= \bm{z}_{i} - \bm{z}_{j}$.
However this loss has a disadvantage in that it will try to move points such that the loss is minimised even though this is uninformative for the GP. Indeed in the case where $\Delta f$ is large, if $\Delta z$ is larger, this loss will try to bring encodings closer such that the loss is minimised. But if $\Delta f$ is already larger than the GP lengthscale, reducing the distance between the encodings will not have a strong impact on the modelling as the predicted correlation between function values is already low. This simplicity in the loss may make the training of the VAE focus on points that are already far away and that do not need more attention, while it could focus on more important points, e.g. points that have a smaller scale in the total loss but are much more important to cluster or separate.

\subsubsection{LR-LBO}
This loss introduced in~\citep{2019_Kim} aims to relax the triplet loss to continuous labels without the need for the usage of a threshold to separate positive $\bm{z}_j$ and negative $\bm{z}_k$ points with respect to the anchor/base point $\bm{z}_i$: $\text{Com}_{\text{metric}=\text{log-ratio}} (\cdot, \cdot) = w_{ijk} \mathbb{E}_{q_{\bm{\phi}}(\cdot)}[(\log \sfrac{||\Delta \bm{z}_{ij}||}{||\Delta \bm{z}_{ik}||} - \log \sfrac{|\Delta f_{ij}|}{|\Delta f_{ik}|})^{2}]$. However, this loss also suffers from a flaw in view of the downstream GP modelling step. 

Suppose a triplet of points such that $|\Delta f_{ij}| = |\Delta f_{ik}|$ and $||\Delta \bm{z}_{ij}|| = ||\Delta \bm{z}_{ik}|| $, then the log-ratio loss associated to this triplet is zero. However, these three points can be in a configuration that would need to be changed, e.g. points $\bm{z}_j$ and $\bm{z}_k$ can be both far away from anchor $\bm{z}_i$ while their images $f_j$ and $f_k$ can be equally close to the anchor's image $f_i$. Therefore bad configurations where $||\Delta \bm{z}_{ij}|| = ||\Delta \bm{z}_{ik}||$ are large but $|\Delta f_{ij}| = |\Delta f_{ik}|$ are small (and vice versa) are not penalised properly.

\subsection{Semi-Supervised Setting}
In this setting (c.f. Section~\ref{subsec:cold_start_res}) we only assume knowledge of $1\%$ of the labelled dataset to start with, $\mathcal{D}_{\mathbb{L}}^{\text{1\%}}$. For all algorithms, the VAE is trained on unlabelled data only (i.e. using $\text{Com}_{\text{label}}$). We then allow the access to $1\%$ of the labelled dataset $\mathcal{D}_{\mathbb{L}}$ to start the BO loop, i.e. fitting the GP model and start BO iterations. The regular retrainings of the VAE are then done on $\mathcal{D}_{\mathbb{U}}$ and the collected labelled data points up to this point. Results show that even in this extreme setting, DML combined with BO is able to outperform baseline \citep{2020_Tripp}.

Note that we have run the special \textbf{LBO-1\%} setting in which we train the VAE \emph{from scratch} only on $\mathcal{D}_{\mathbb{L}}^{1\%}$ which means we do not allow access to large amounts of unlabelled data. This is done to check if the baseline from \citep{2020_Tripp} can recover competitive results when starting from an extremely small labelled dataset. As seen in the main paper results, it is unable to do so in less than 1000 iterations.

\section{Experiment Setup}\label{App:Exps}
\subsection{Task Descriptions}
In this section we provide further details about each task described in brief in the main paper.
\subsubsection{Topology Shape Fitting}
In this task we use the Topology dataset \cite{sosnovik2017neural} which is a set of $40 \times 40$ gray-scale images generated from mechanical parts. These images are categorised into 10'000 classes, each class containing 100 images of the same part at different stages of reconstruction/resolution. We select only the last or best image for each class, leaving us with effectively 10'000 images in total in our dataset. Picking one image at random and setting it as the target for all subsequent experiments, we seek to optimize the cosine similarity $\cos(\bm{x},\bm{x}')=\bm{x}\bm{x'}^T/\Vert \bm{x} \Vert \Vert \bm{x}' \Vert$ between a new point $x$ and the target $x'$. We use a VAE with a continuous latent space of dimension $20$ (i.e. $\mathcal{Z} \subseteq \mathbb{R}^{20}$) and a standard convolutional architecture, alternating strided convolutional and batch normalisation layers. 
\subsubsection{Expression Reconstruction}
Similar to \cite{2020_Tripp} and \cite{kusner2017grammar} the goal is to generate single-variable mathematical expressions that minimises the mean squared error to a target expression \texttt{1/3 * x * sin(x*x)} evaluated at 1000 values uniformly-distributed between $-10$ and $+10$. The dataset consists of 100'000 such univariate expressions generated by a formal grammar using the GrammarVAE from \cite{kusner2017grammar}. The expressions are first embedded to a $12 \times 15$ matrix following \cite{kusner2017grammar} and subsequently encoded to a continuous space $\mathcal{Z} \subseteq \mathbb{R}^{25}$. Note that in our experiments we only use a subset of the total dataset as it already contains 
the target equation. Ranking the points by their score, we take the bottom $35\%$ of the dataset and $N_{good}$ equations sampled from the remainder, to which we remove the top $3\%$ best points, i.e. we sample $N_{good}$ ($=5\%$ of dataset) equations from the $65^{th}$-percentile to the $3^{rd}$-percentile. In addition to removing the best possible expressions from the dataset, this procedure also leaves a greater margin for progression enabling us to compare the experimental settings and tested algorithms more easily. We end up with 40'000 expressions in the dataset.
\subsubsection{Chemical Design}
Following \cite{jin2018}, the goal of this task is to optimise the properties of molecules from the ZINC250K dataset \cite{sterling2015zinc} where each molecule's property or score is its penalised \textit{water-octanol partition coefficient} (PlogP). Molecules are represented as a unique SMILES sequence using a Junction-Tree VAE \cite{jin2018}, a state-of-the-art generative model for producing valid molecular structures. In this task the continuous latent space used to represent the inputs is of dimension 56 (i.e. $\mathcal{Z} \subseteq \mathbb{R}^{56}$).
\subsection{Phases of Experiments Explained} 
Each experimental setting is comprised of three steps. First a VAE is trained on a dataset to learn how to map the original input space to a low-dimensional latent space and reconstruct points. Then before starting the BO, we fit our surrogate GP model. Finally we run the BO loop. We give further details on each step in the following.
\subsubsection{Training the VAE}
For the Topology and Expression tasks, we choose to train the VAE from scratch. In the molecule task we start with an already-trained JTVAE \citep{jin2018} as performed in \citep{2020_Tripp}. We train their respective VAE with the desired ELBO components and dataset depending on the experimental setting (e.g. with metric learning or without, with labelled or unlabelled data, with weights or without \dots). This model will then be used in the BO loop and will be updated as we collect new points. Based on the results reported in \citep{2020_Tripp} we choose to retrain every $r=50$ acquisition steps. At each retraining, we recompute the weights as explained in \citep{2020_Tripp} with $k=1e-3$.
\subsubsection{Fitting the GP}
We use a sparse GP \citep{2009_Titsias} with 500 inducing points trained on the $N_{best}$ points with higher score and $N_{rand}$ points taken at random from the remainder of the dataset (the number of points varies across tasks, see \ref{App:hyper_parameters}). Inputs are normalised and targets are standardised in Topology and Expression task but not in Molecule, similar to \citep{2020_Tripp}. Finally the GP is trained from scratch after each retraining of the VAE as targets can change but it is not retrained from scratch after each acquisition as we only add one point to the dataset (and in turn only 50 points in-between each VAE retraining as $r=50$) which saves time and computational resources in practice while ensuring a good model fit over regions of the latent space populated with good points.
\subsubsection{Running the BO Loop}
To perform BO on the latent space we follow Algorithm \textcolor{blue}{1} from the main paper and use Expected Improvement \citep{jones1998} as an acquisition function. At each step, we acquire one new point by optimising the acquisition function but before evaluating it on the black-box we check if it is already present in the current dataset. If it is present, we perturb it or restart the iteration until we find a novel point that is not already in our dataset. This scheme - applied to all experimental settings for the sake of fair comparison - enforces novelty in our exploration and can avoid the optimiser becoming stuck at the same point for multiple iterations.

\subsection{Topology \& Expression Latent Space Distance Histograms}
We demonstrate the latent structure of the VAE on the topology and expression tasks in the same fashion as in the main paper, i.e. comparing the VAE latent space distribution of distances with and without metric learning, see Figure~\ref{fig:latent_dist_topo} and Figure~\ref{fig:latent_dist_expr}. In Topology, because of the nature of this task, similar inputs already have similar scores. Indeed inputs are images and we optimise the similarity measure between each input and a target image. Noticing this property, it is understandable that the latent space obtained by training a vanilla VAE already exhibits some desirable structure, i.e. points with similar scores will be closer together in latent space. Using the right metric loss while training the VAE encourages this structure even more. From Figure~\ref{fig:latent_dist_topo} we can visualise that points with higher scores have been clustered closer together in $\mathcal{Z}$ but also a potential negative effect; points with the lowest scores do not appear to be closer together on average than to other points. Latent separation in this toy example is less evident when compared to other scenarios and such an observation partially explains the similar regret results across many algorithms. However the desired clustering of points is detectable in the expression task (see Figure~\ref{fig:latent_dist_expr}).
\begin{figure*}[h!]
    \centering
    \includegraphics[scale=0.22]{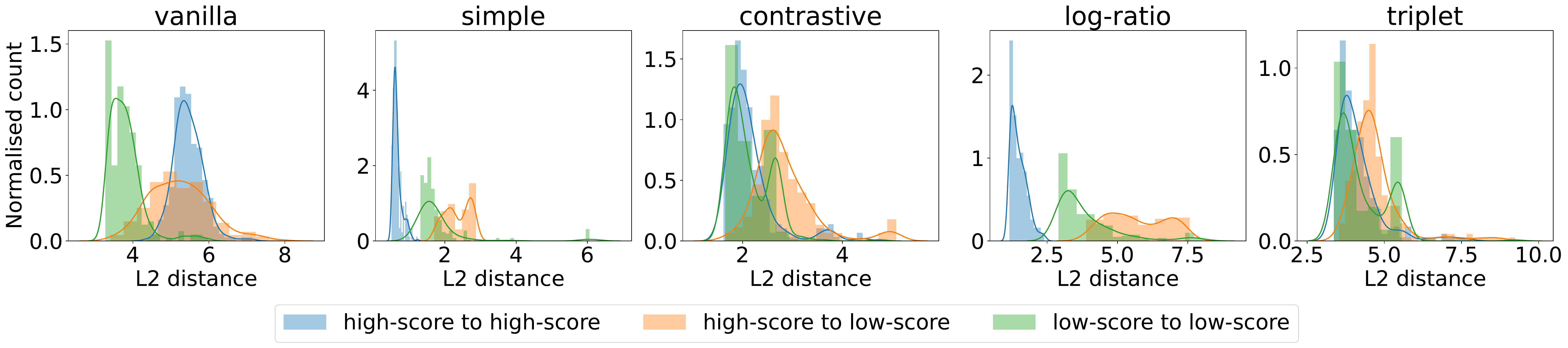}
    \caption{Latent space distances in the topology task.}
    \label{fig:latent_dist_topo}
\end{figure*}
\begin{figure*}[h!]
    \centering
    \includegraphics[scale=0.22]{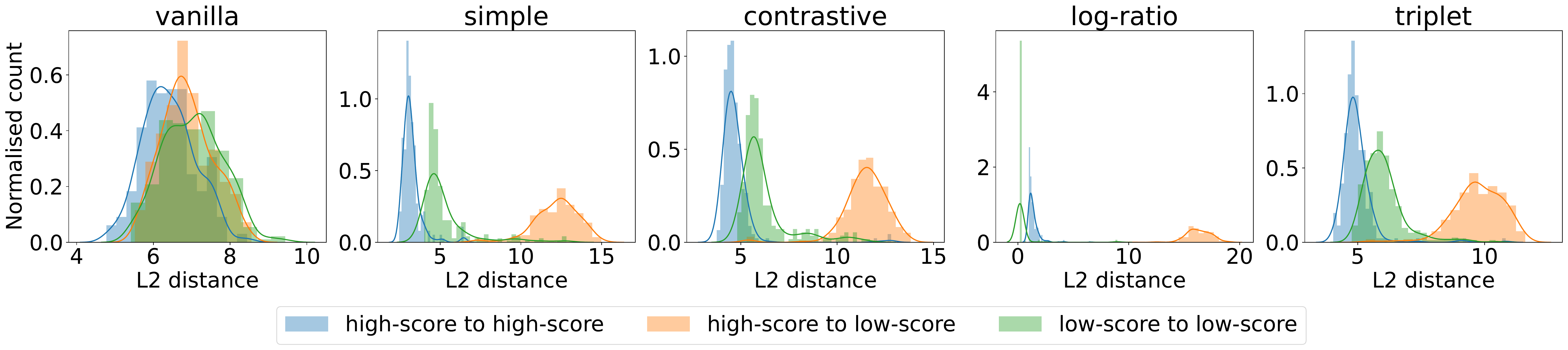}
    \caption{Latent space distances in the expression task.}
    \label{fig:latent_dist_expr}
\end{figure*}

\subsection{Topology and Expression semi-supervised regret results}
The Topology and Expression tasks were also put to the test of the semi supervised setting, i.e. starting with little labelled data and a large amount of unlabelled data. The setting and algorithm details are similar to what is described in the main paper as well as in Appendix section~\ref{App:elbo_comp_details}. The results are presented in Figure~\ref{fig:cold_start_topo_expr}.

\begin{figure*}
    \centering
    \includegraphics[scale=0.12]{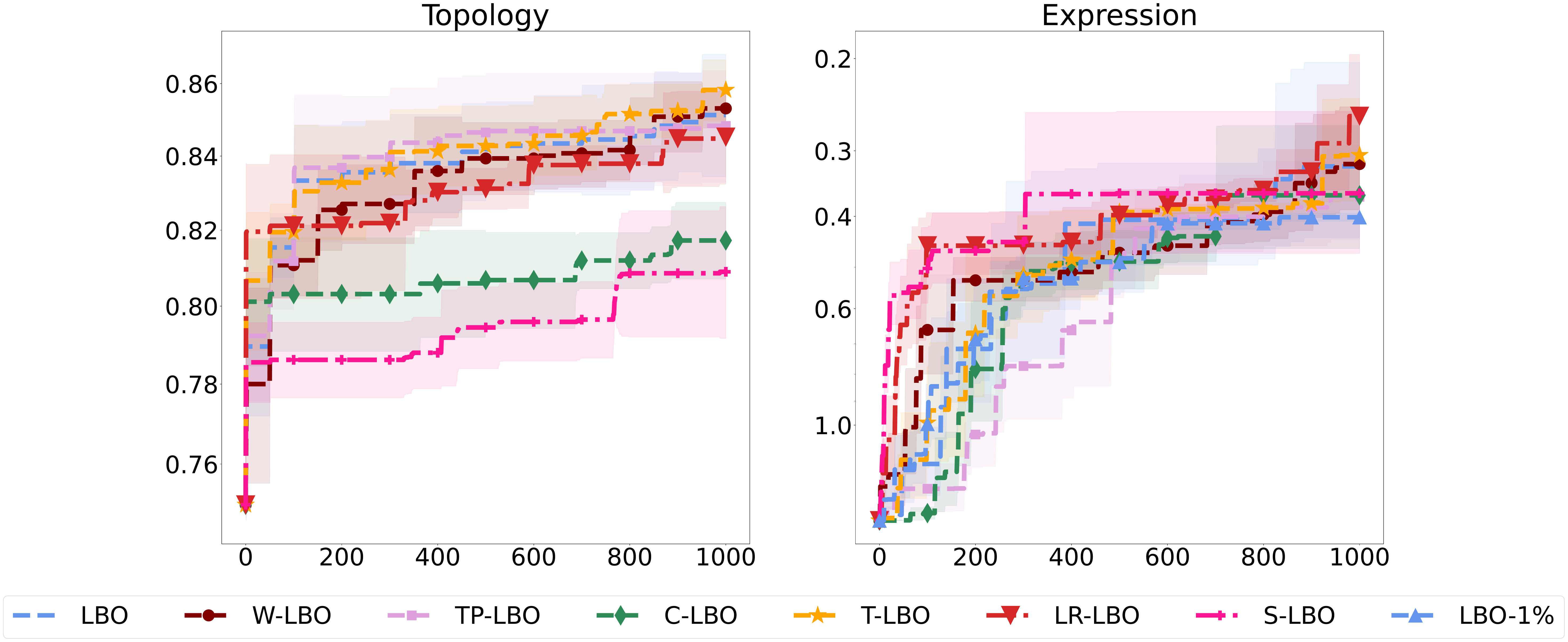}
    \caption{Regret results in the semi-supervised setting on Topology and Expression tasks.}
    \label{fig:cold_start_topo_expr}
\end{figure*}



\subsection{Hyper-Parameters}\label{App:hyper_parameters}
For reproducibility, we now detail all hyper-parameters used across all experiments. Note that in the semi-supervised setting, as we start from $\mathcal{D}_{\mathbb{L}}^{1\%}$, we cannot have the same number of points to initially train our GP model so we make use of all available data. The rest is similar.

\begin{table*}[h]\caption{Hyper-parameters of LBO and W-LBO settings on all tasks.}
\centering
\begin{tabular}{|c | c | c c c|| c c |} 
 \hline
  \multicolumn{2}{|c|}{} & \multicolumn{3}{c||}{LBO} & \multicolumn{2}{c|}{W-LBO} \\
 \cline{3-7}
 \multicolumn{2}{|c|}{}  & Topology & Expression & Molecule & Topology & Expression \\
 \hline
 \parbox[t]{1mm}{\multirow{8}{*}{\rotatebox[origin=c]{90}{VAE pre-training}}}
 & epochs & 300 & 300 & - & 300 & 300  \\
 & k & 1e-3 & 1e-3 & 1e-3 & 1e-3 & 1e-3 \\
 & r & 50 & 50 & 50  & 50 & 50 \\
 & optimiser & Adam & Adam & Adam  & Adam & Adam \\
 & lr & 1e-3 & 1e-3 & 1e-3  & 1e-3 & 1e-3  \\
 & batch size & 1024 & 256 & -  & 1024 & 256 \\
 & $\beta_{KL}^{init}$ & 1e-6 & 1e-6 & 1e-3  & 1e-6 & 1e-6 \\
 & $\beta_{KL}^{final}$ & 1e-4 & 0.04 & 1e-3 & 1e-4 & 0.04 \\
 \hline
 \parbox[t]{1mm}{\multirow{6}{*}{\rotatebox[origin=c]{90}{VAE retraining}}}
 & epochs & 1 & 1 & 1  & 1 & 1 \\
 & k & 1e-3 & 1e-3 & 1e-3 & 1e-3 & 1e-3 \\
 & optimiser & Adam & Adam & Adam  & Adam & Adam \\
 & lr & 1e-3 & 1e-3 & 1e-3 -  & 1e-3 & 1e-3 \\
 & batch size & 256 & 256 & 128  & 256 & 256  \\
 & $\beta_{KL}$ & 1e-4 & 0.04 & 1e-3  & 1e-4 & 0.04 \\
 \hline
 \parbox[t]{1mm}{\multirow{5}{*}{\rotatebox[origin=c]{90}{GP}}}
 & inducing pts. & 500 & 500 & 500 & 500 & 500  \\
 & best pts. & 2500 & 2500 & 2000 & 2500 & 2500 \\
 & rand. pts. & 500 & 500 & 8000 & 500 & 500 \\
 & kernel & RBF & RBF & RBF  & RBF & RBF \\
 & mean & const. & const. & const. & const. & const. \\
 & transf. & - & - & - & Kumaraswarmy & Kumaraswarmy \\
 \hline
 & Acq. func. & EI & EI & EI & EI & EI\\
 & optimiser & LBFGS & LBFGS & LBFGS & LBFGS & LBFGS \\
 \hline
\end{tabular}
\label{table:hyper-params-LBO}
\end{table*}

\begin{table*}[h]\caption{Hyper-parameters of C-LBO and T-LBO settings on all tasks.}
\centering
\begin{tabular}{|c | c | c c c|| c c c |} 
 \hline
 \multicolumn{2}{|c|}{} & \multicolumn{3}{c||}{C-LBO} &  \multicolumn{3}{c|}{T-LBO} \\
 \cline{3-8}
 \multicolumn{2}{|c|}{}  & Topology & Expression & Molecule & Topology & Expression & Molecule \\
 \hline
 \parbox[t]{1mm}{\multirow{11}{*}{\rotatebox[origin=c]{90}{VAE pre-training}}}
 & epochs & 300 & 300 & 20 & 300 & 300 & 20 \\
 & k & 1e-3 & 1e-3 & 1e-3  & 1e-3 & 1e-3 & 1e-3 \\
 & optimiser & Adam & Adam & Adam  & Adam & Adam & Adam \\
 & lr & 1e-3 & 1e-3 & 1e-3  & 1e-3 & 1e-3 & 1e-3 \\
 & batch size & 1024 & 256 & 1024  & 1024 & 256 & 128 \\
 & $\beta_{KL}^{init}$ & 1e-6 & 1e-6 & 1e-3 & 1e-6 & 1e-6 & 1e-3 \\
 & $\beta_{KL}^{final}$ & 1e-4 & 0.04 & 1e-3  & 1e-4 & 0.04 & 1e-3 \\
 & metric & s-cont & s-cont & s-cont  & s-triple & s-triple & s-triple \\
 & $\rho$ & 0.1 & 0.1 & 0.1 & 0.1 & 0.1 & 0.1 \\
 & $\eta$ & - & - & - & 0 & 0 & 0 \\
 & $\beta_{\text{metric}}$ & 1 & 10 & 1 & 1 & 10 & 1 \\
 \hline
 \parbox[t]{1mm}{\multirow{11}{*}{\rotatebox[origin=c]{90}{VAE retraining}}}
 & epochs & 1 & 1 & 1 & 1 & 1 & 1 \\
 & k & 1e-3 & 1e-3 & 1e-3 & 1e-3 & 1e-3 & 1e-3 \\
 & r & 50 & 50 & 50 & 50 & 50 & 50 \\
 & optimiser & Adam & Adam & Adam  & Adam & Adam & Adam \\
 & lr & 1e-3 & 1e-3 & 1e-3 & 1e-3 & 1e-3 & 1e-3 \\
 & batch size & 256 & 256 & 128  & 256 & 256 & 128  \\
 & $\beta_{KL}$ & 1e-4 & 0.04 & 1e-3  & 1e-4 & 0.04 & 1e-3 \\
 & metric & s-cont & s-cont & s-cont  & s-triple & s-triple & s-triple \\
 & $\rho$ & 0.1 & 0.1 & 0.1 & 0.1 & 0.1 & 0.1 \\
 & $\eta$ & - & - & - & 0 & 0 & 0 \\
 & $\beta_{\text{metric}}$ & 1 & 10 & 1 & 1 & 10 & 1 \\
 \hline
 \parbox[t]{1mm}{\multirow{5}{*}{\rotatebox[origin=c]{90}{GP}}}
 & inducing pts. & 500 & 500 & 500 & 500 & 500 & 500  \\
 & best pts. & 2500 & 2500 & 2000 & 2500 & 2500 & 2500 \\
 & rand. pts. & 500 & 500 & 8000 & 500 & 500 & 8000 \\
 & kernel & RBF & RBF & RBF  & RBF & RBF & RBF \\
 & mean & const. & const. & const. & const. & const. & const. \\
 \hline
 & Acq. func. & EI & EI & EI & EI & EI & EI\\
 & optimiser & LBFGS & LBFGS & LBFGS & LBFGS & LBFGS & LBFGS \\
 \hline
\end{tabular}
\label{table:hyper-params-C-LBO}
\end{table*}

\begin{table*}[h]\caption{Hyper-parameters of S-LBO and LR-LBO settings on all tasks.}
\centering
\begin{tabular}{|c | c | c c c|| c c c |} 
 \hline
 \multicolumn{2}{|c|}{} & \multicolumn{3}{c||}{S-LBO} &  \multicolumn{3}{c|}{LR-LBO} \\
 \cline{3-8}
 \multicolumn{2}{|c|}{}  & Topology & Expression & Molecule & Topology & Expression & Molecule \\
 \hline
 \parbox[t]{1mm}{\multirow{11}{*}{\rotatebox[origin=c]{90}{VAE pre-training}}}
 & epochs & 300 & 300 & 20 & 300 & 300 & 20 \\
 & k & 1e-3 & 1e-3 & 1e-3  & 1e-3 & 1e-3 & 1e-3 \\
 & optimiser & Adam & Adam & Adam  & Adam & Adam & Adam \\
 & lr & 1e-3 & 1e-3 & 1e-3  & 1e-3 & 1e-3 & 1e-3 \\
 & batch size & 1024 & 256 & 1024  & 1024 & 256 & 128 \\
 & $\beta_{KL}^{init}$ & 1e-6 & 1e-6 & 1e-3 & 1e-6 & 1e-6 & 1e-3 \\
 & $\beta_{KL}^{final}$ & 1e-4 & 0.04 & 1e-3  & 1e-4 & 0.04 & 1e-3 \\
 & metric & s-cont & s-cont & s-cont  & s-triple & s-triple & s-triple \\
 & $\beta_{\text{metric}}$ & 1 & 10 & 1 & 1 & 10 & 1 \\
 \hline
 \parbox[t]{1mm}{\multirow{11}{*}{\rotatebox[origin=c]{90}{VAE retraining}}}
 & epochs & 1 & 1 & 1 & 1 & 1 & 1 \\
 & k & 1e-3 & 1e-3 & 1e-3 & 1e-3 & 1e-3 & 1e-3 \\
 & r & 50 & 50 & 50 & 50 & 50 & 50 \\
 & optimiser & Adam & Adam & Adam  & Adam & Adam & Adam \\
 & lr & 1e-3 & 1e-3 & 1e-3 & 1e-3 & 1e-3 & 1e-3 \\
 & batch size & 256 & 256 & 128  & 256 & 256 & 128  \\
 & $\beta_{KL}$ & 1e-4 & 0.04 & 1e-3  & 1e-4 & 0.04 & 1e-3 \\
 & metric & s-cont & s-cont & s-cont  & s-triple & s-triple & s-triple \\
 & $\beta_{\text{metric}}$ & 1 & 10 & 1 & 1 & 10 & 1 \\
 \hline
 \parbox[t]{1mm}{\multirow{5}{*}{\rotatebox[origin=c]{90}{GP}}}
 & inducing pts. & 500 & 500 & 500 & 500 & 500 & 500  \\
 & best pts. & 2500 & 2500 & 2000 & 2500 & 2500 & 2500 \\
 & rand. pts. & 500 & 500 & 8000 & 500 & 500 & 8000 \\
 & kernel & RBF & RBF & RBF  & RBF & RBF & RBF \\
 & mean & const. & const. & const. & const. & const. & const. \\
 \hline
 & Acq. func. & EI & EI & EI & EI & EI & EI\\
 & optimiser & LBFGS & LBFGS & LBFGS & LBFGS & LBFGS & LBFGS \\
 \hline
\end{tabular}
\label{table:hyper-params-LR-LBO}
\end{table*}

\begin{table*}[h!]\caption{Hyper-parameters of TP-LBO setting on all tasks.}
\centering
\begin{tabular}{|c | c | c c c|} 
 \hline
 &  & Topology & Expression & Molecule \\
 \hline
 \parbox[t]{1mm}{\multirow{10}{*}{\rotatebox[origin=c]{90}{VAE pre-training}}}
 & epochs & 300 & 300 & - \\
 & k & 1e-3 & 1e-3 & 1e-3 \\
 & optimiser & Adam & Adam & Adam \\
 & lr & 1e-3 & 1e-3 & 1e-3 \\
 & batch size & 1024 & 256 & - \\
 & $\beta_{KL}^{init}$ & 1e-6 & 1e-6 & 1e-6 \\
 & $\beta_{KL}^{final}$ & 1e-4 & 0.04 & 1e-3 \\
 & regressor & MLP-128-128 & MLP-128-128 & MLP-128-128 \\
 & $\beta_{R}$ & 10 & 1 & 10 \\
 \hline
 \parbox[t]{1mm}{\multirow{10}{*}{\rotatebox[origin=c]{90}{VAE retraining}}}
 & epochs & 1 & 1 & 1 \\
 & k & 1e-3 & 1e-3 & 1e-3 \\
 & r & 50 & 50 & 50 \\
 & optimiser & Adam & Adam & Adam \\
 & lr & 1e-3 & 1e-3 & 1e-3 \\
 & batch size & 256 & 256 & 128 \\
 & $\beta_{KL}$ & 1e-4 & 0.04 & 1e-3 \\
 & regressor & MLP-128-128 & MLP-128-128 & MLP-128-128 \\
 & $\beta_{R}$ & 10 & 1 & 10 \\
 \hline
 \parbox[t]{1mm}{\multirow{5}{*}{\rotatebox[origin=c]{90}{GP}}}
 & inducing pts. & 500 & 500 & 500 \\
 & best pts. & 2500 & 2500 & 2000 \\
 & rand. pts. & 500 & 500 & 8000 \\
 & kernel & RBF & RBF & RBF \\
 & mean & const. & const. & const. \\
 \hline
 & Acq. func. & EI & EI & EI \\
 & optimiser & LBFGS & LBFGS & LBFGS \\
 \hline
\end{tabular}
\label{table:hyper-params-R-LBO}
\end{table*}


\subsection{Analysis of other Hyper-Parameters \& Generated Molecules}
\paragraph{Weight and Margin Effects on Performance:} In this section we collate all additional experiments implemented. We choose to implement these experiments on the expression task as it is more complex than a toy task yet does not carry the computational overhead of the molecule task. We compare the hard and soft versions of the contrastive loss, we vary the threshold parameter $\rho$ as well as the weight parameter from the soft-triplet loss $\nu$. Results are displayed in Figures \ref{fig:warm_ablation_expr} and \ref{fig:cold_ablation_expr}. 
\paragraph{Generated Molecules Depiction:} In Figure~\ref{Fig:molecule--samples}, we demonstrate molecules generated with our method across the VAEs retaining phases. Figure~\ref{Fig:molecule--samples} (a) shows the case when using the complete dataset $\mathcal{D}_{\mathbb{L}}$, while Figure~\ref{Fig:molecule--samples} (b) reiterate these results but only accessing $\mathcal{D}_{\mathbb{L}}^{1\%}$. In both those cases, we observe that T-LBO is capable of significantly improving PlogP values beyond the best molecules available in the dataset. 

\begin{figure*}[ht!]
    \centering
    \includegraphics[scale=0.11]{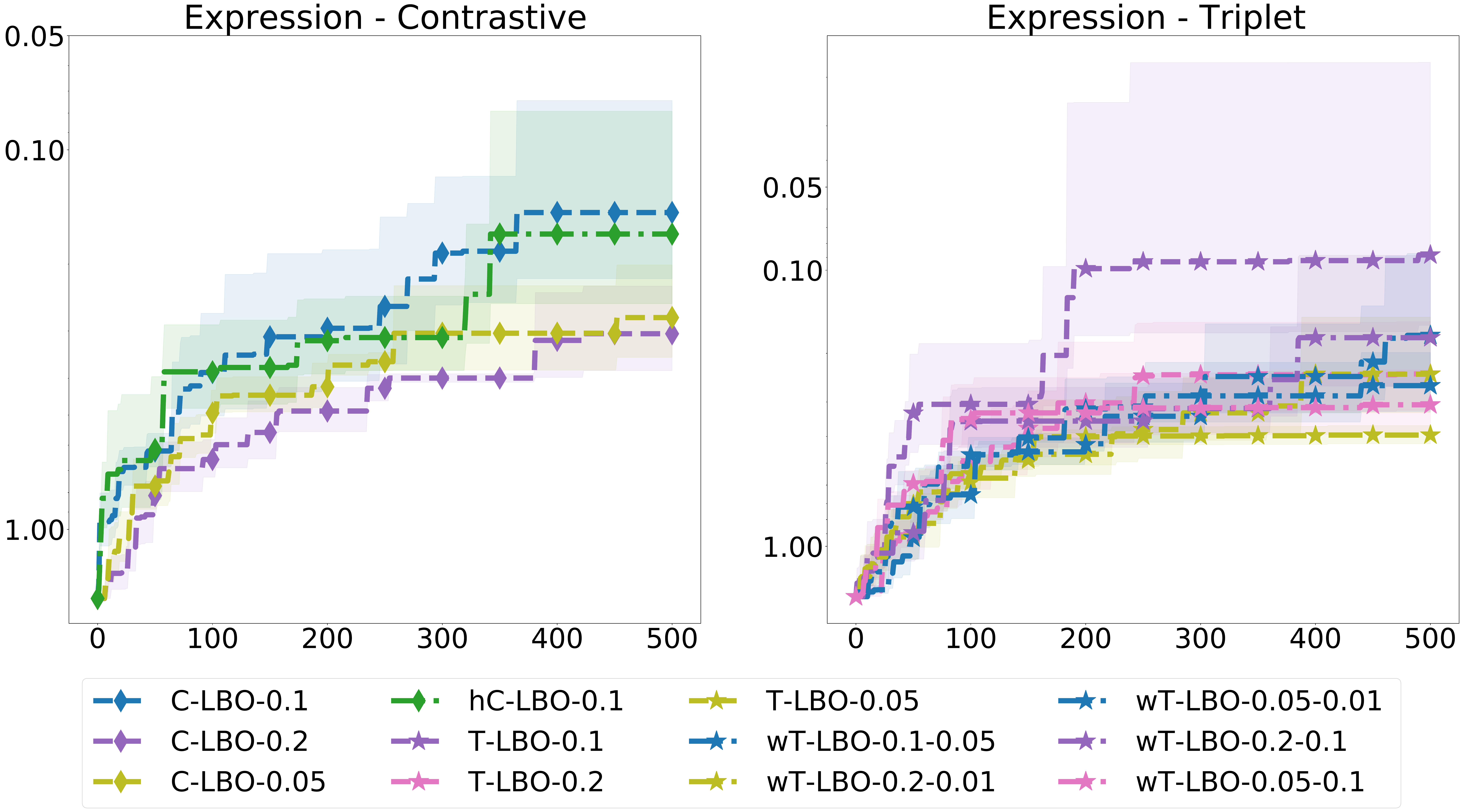}
    \caption{Comparison of the effect of parameter $\rho$ and $\nu$ across settings on the Expression task in the same setting as in \citep{2020_Tripp}. The threshold value $\rho$ is the first value in each of the figure labels. C-LBO and T-LBO are the settings described in \ref{App:elbo_comp_details} while hC-LBO is the hard contrastive loss and wT-LBO uses the weighted triplet loss with parameter $\nu$ being the second value in the label.}
    \label{fig:warm_ablation_expr}
\end{figure*}

\begin{figure*}[ht!]
    \centering
    \includegraphics[scale=0.11]{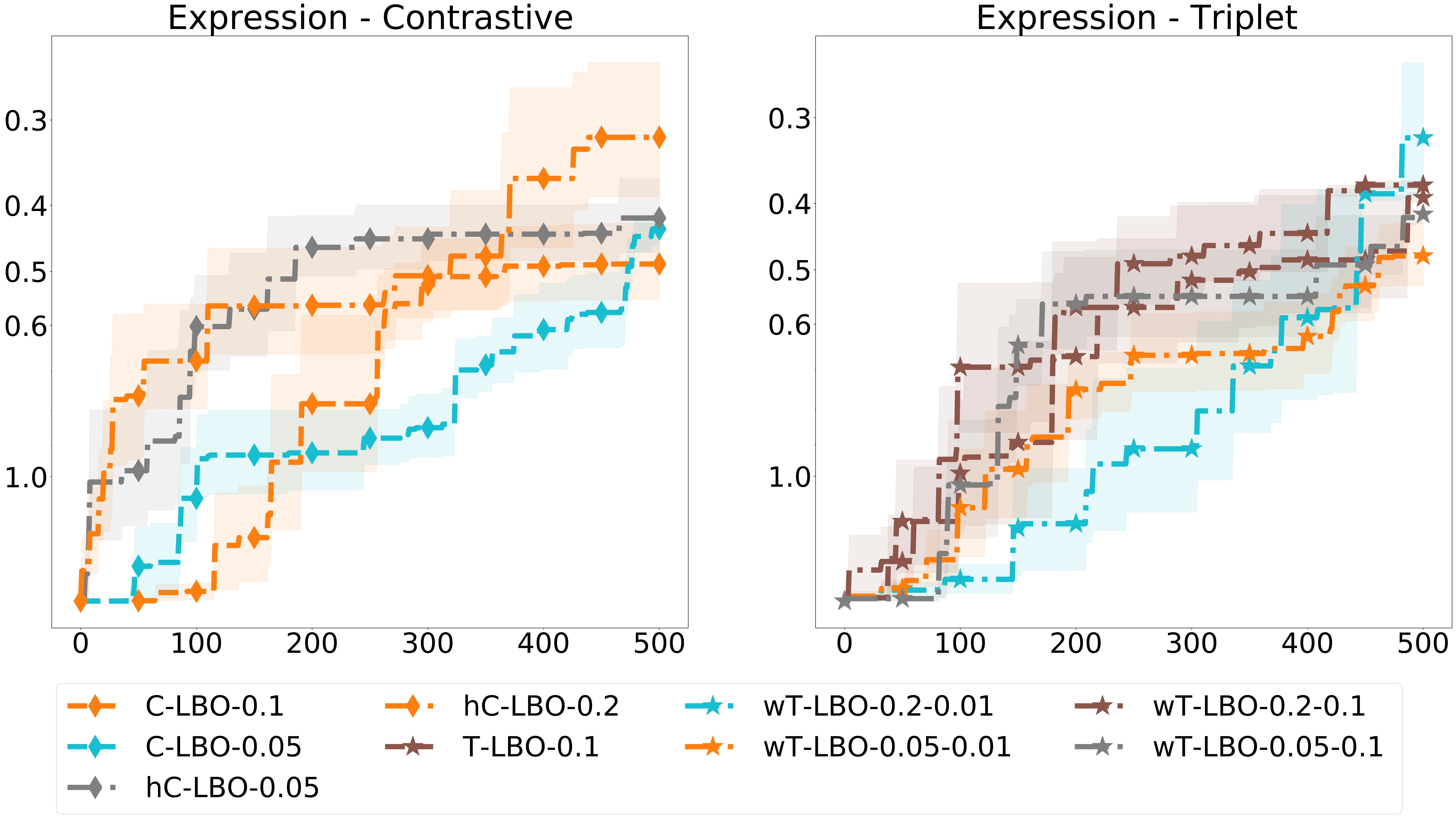}
    \caption{Comparison of the effect of parameter $\rho$ and $\nu$ across settings on the Expression task, in the semi-supervised setting. The threshold value $\rho$ is the first value in each of the figure labels. C-LBO and T-LBO are the settings described in \ref{App:elbo_comp_details} while hC-LBO is the hard contrastive loss and wT-LBO uses the weighted triplet loss with parameter $\nu$ being the second value in the label.}
    \label{fig:cold_ablation_expr}
\end{figure*}

\newpage
\subsection{Hardware}
For further reproducibility, we also provide details concerning the hardware we utilised in our experiments. We report an estimation of the running time of training the VAEs in Table \ref{table:hardware}. All experiments were run on a single GPU (either \textbf{NVIDIA} \texttt{Tesla V100} or \texttt{GeForce}).
\begin{table}[h!]\caption{Average estimated runtime.}
\centering 
\begin{tabular}{l c} 
 \hline
 Task & Average time\\
 \hline
 Topology & 0.75h\\
 Expression& 1.5h\\
 Molecule & 30h\\
\end{tabular}
\label{table:hardware}
\end{table}

\begin{figure*}
\centering
\subfloat[\textbf{T-LBO} -- Starting from all available datapoints $\mathcal{D}_\mathbb{L}$, the best molecule initially observed in the dataset (displayed in the top-left corner) has a score of $5.80$. Across acquisitions and retrainings of the JTVAE with triplet loss, the best score increases reaching $29.06$ after the final retraining (bottom right molecule)]{\label{app:TLBO-samples}\includegraphics[scale=0.3]{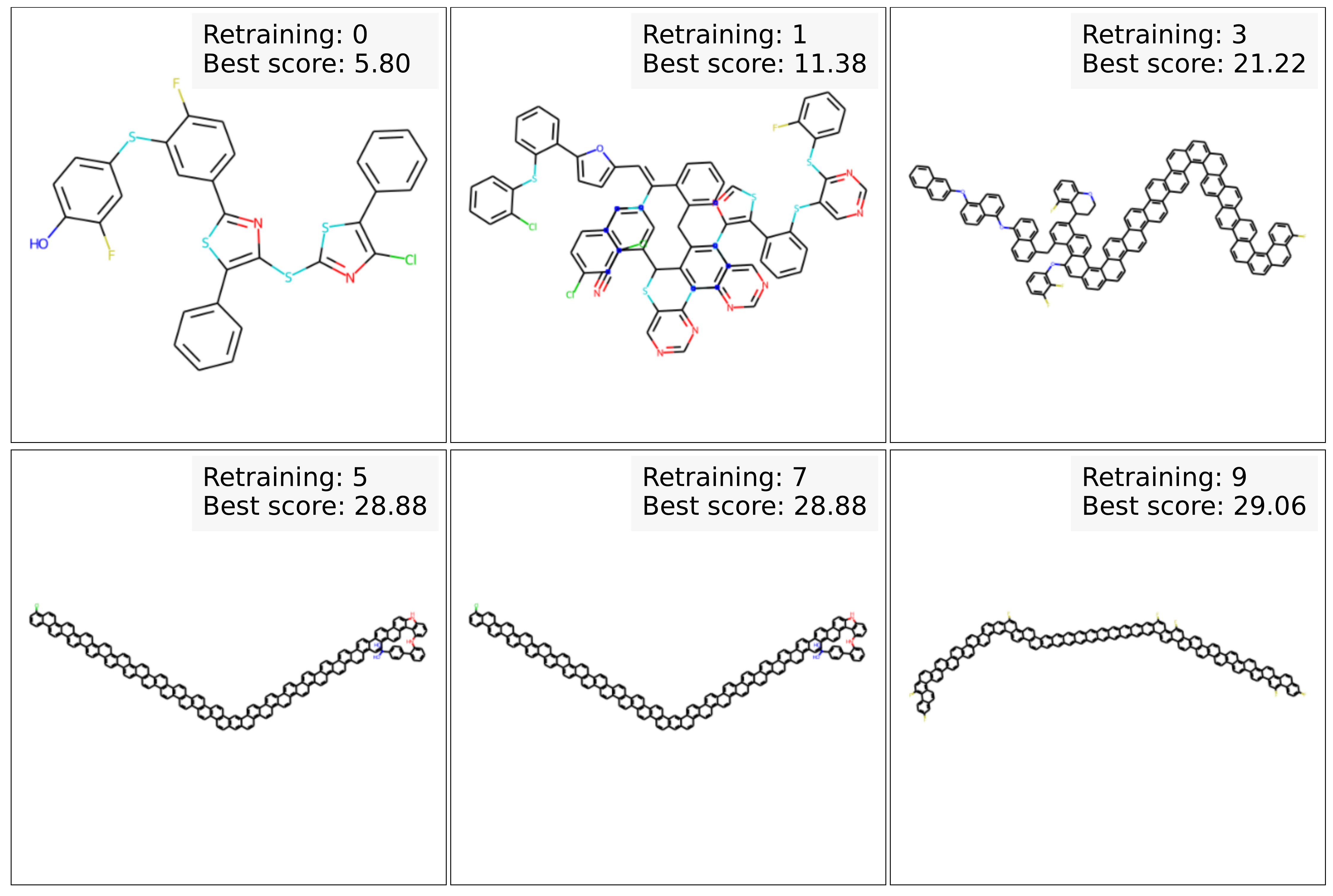}}
\par
\centering
\subfloat[\textbf{T-LBO} -- Starting with observation of only 1\% of labelled datapoints $\mathcal{D}_\mathbb{L}^\text{1\%}$, the best molecule initially available (displayed on the top-left corner) has a score of $4.09$. Under this semi-supervised setup, our method manage to recover \textbf{T-LBO} results reaching $29.14$ after only $6$ retrainings of the JTVAE with triplet loss (bottom right molecule)]{\label{app:STLBO-samples}\includegraphics[scale=0.3]{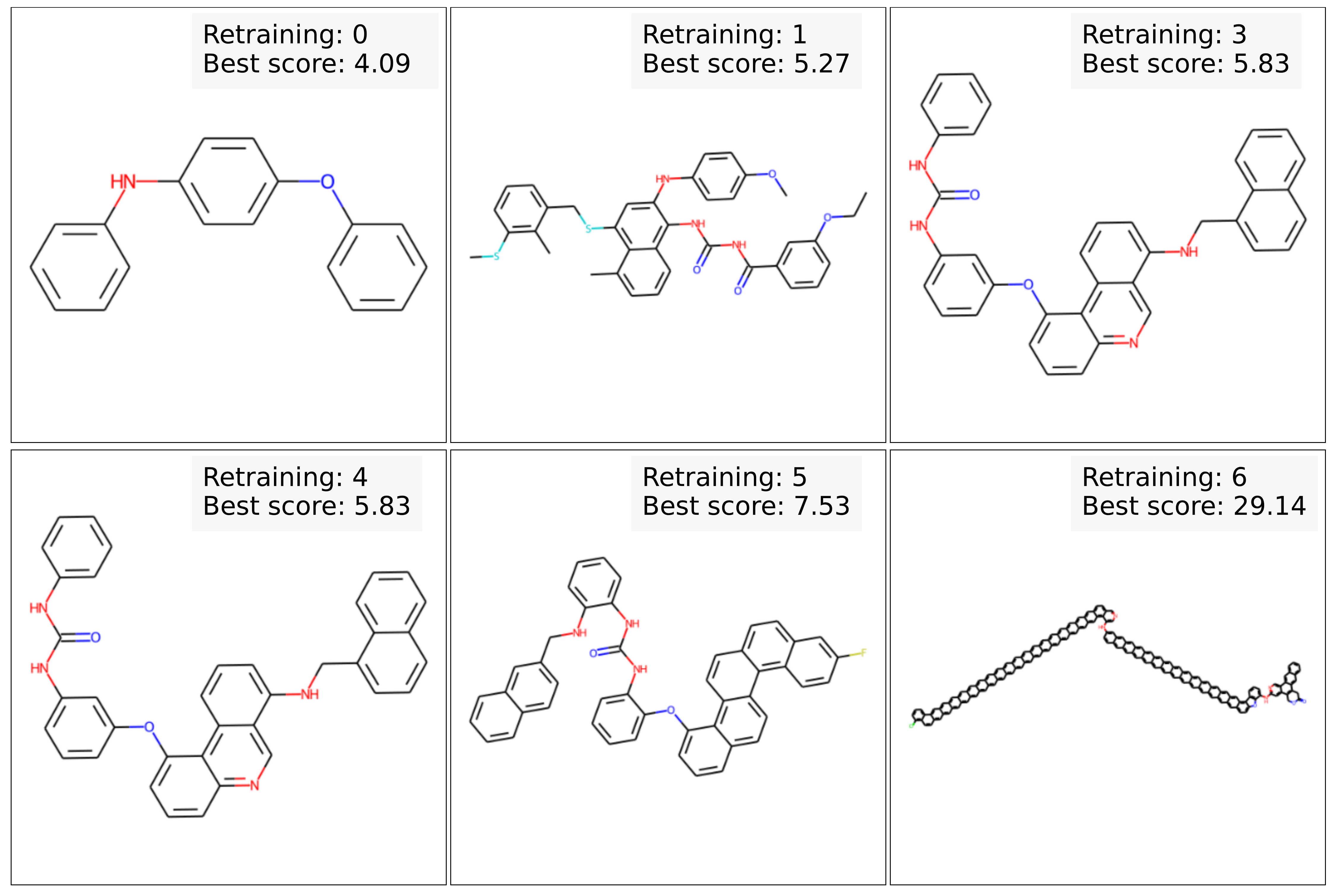}}
\caption{Evolution of the molecules obtained when applying \textbf{T-LBO} to penalised logP maximisation.}
\label{Fig:molecule--samples}
\end{figure*}

\section{Proof of Theorem 1}\label{App:Theory}
For clarity the proof of Theorem \textcolor{blue}{1} is split into subsections. First, we describe the assumptions we make for the black-box objective function and the encoder-decoder mappings. Second we provide the proof of  the sub-linear regret guarantees stated in Theorem \textcolor{blue}{1}. This proof will demonstrate that the declared assumptions afford sufficient conditions for vanishing regret. Third, we provide necessary conditions for vanishing regret by constructing an example black-box objective for which any latent space Bayesian optimisation method achieves constant regret. 

\subsection{Assumptions}\label{App:Assumptions}
We state here the assumptions guaranteeing vanishing regret for Algorithm \textcolor{blue}{1}.

\begin{assumption}\label{App:Assumption_1}
Let us consider the black-box objective $f(\cdot)$ defined on the primal space $\mathcal{X}$ and its latent counterpart $f_{\text{latent}}(\bm{z}) = \mathbb{E}_{\bm{x}\sim g_{\bm{\theta
}}(\cdot|\bm{z})}\left[f(\bm{x})\right]$ defined on the latent space via a decoder $g_{\bm{\theta}}(\cdot|\bm{z})$. We assume:
\begin{enumerate}
    \item Each evaluation of the black-box function $f(\bm{x})$ is subject to zero-mean Gaussian noise, i. e. $y(\bm{x}) = f(\bm{x}) + \epsilon$, where  $\epsilon\sim\mathcal{N}(\cdot; 0,\sigma^2_{\text{noise}})$.
    \item The function $f_{\text{latent}}(\cdot)$ is smooth according to the reproducing kernel Hilbert space associated with the GP squared exponential covariance function  $\text{k}_{\text{SE}}(\cdot,\cdot)$ (cf.~\cite{rasmussen2006}).
    \item The function $f(\cdot)$ is bounded, i.e. for any $\bm{x}\in \mathcal{X}$ we have  $|f(\bm{x})|\le G_{f}$ for some constant $G_f > 0$.
\end{enumerate}
\end{assumption}

\begin{assumption}\label{App:Assumption_2}
Without loss of generality, we assume the following:
\begin{enumerate}
    \item Given a dataset of observations $\mathcal{D}_{\mathbb{Z}} = \langle \bm{z}_i, f(\bm{x}_i)\rangle^N_{i=1}$, the associated posterior variance for $f_{\text{latent}}(\cdot)|\mathcal{D}_{\mathbb{Z}}$ is lower and upper-bounded, i.e. there are constants  $g_1,G_1 > 0$, such that for any $\bm{z}\in\mathcal{Z}$:
    $\sigma_{f_{\text{latent}}}(\bm{z}|\mathcal{D}_{\mathbb{Z}}) \ge g_1$ and $\sigma_{f_{\text{latent}}}(\bm{z}|\mathcal{D}_{\mathbb{Z}})  \le G_1$.
    \item In the covariance function associated with the trained GP for function $f_{\text{latent}}(\bm{z})$ for any $\bm{z}\in\mathcal{Z}$ we have $\text{k}_{\text{SE}}(\bm{z}, \bm{z}) =1$.
    \item The noise random variables $\epsilon\sim\mathcal{N}(\cdot; 0, \sigma^2_{\text{noise}})$ which corrupt the black-box function evaluations at each iteration $\ell$ of BO are uniformly bounded by $\sigma_{\text{noise}}$.
\end{enumerate}
\end{assumption}

\begin{assumption}\label{App:Assumption_3}
We assume that starting from some epoch $\tilde{\ell}$ the decoder $g_{\bm{\theta}^*_{\ell}}(\cdot|\bm{z})\in\mathcal{P}(\mathcal{X})$ improves its recovery ability of the global maximiser $\bm{x}^* = \arg\max_{\bm{x}\in\mathcal{X}}f(\bm{x})$ with all subsequent epochs, in the sense that  for any $\ell \ge \tilde{\ell}$ for input $\bm{x}^*$ the probability that it can be recovered by the decoder for some latent input $\bm{z}' = \bm{z}'(\ell)$ is increasing with epochs:
\begin{align*}
    \forall \ell\ge \tilde{\ell}, \ \  \exists \bm{z}' = \bm{z}'(\ell)\ \in\mathcal{Z}\text{ such that } \mathbb{P}\left[\bm{x}^*\sim g_{\bm{\theta}^*_{\ell}}(\cdot|\bm{z}')\right]\ge 1 - \Upsilon(\ell),
\end{align*}
for some decreasing positive-valued function $\Upsilon(\ell)$ such that  $\lim_{T\to +\infty}\frac{\int^T_{0}\Upsilon(a)da}{T} \to 0$.
\end{assumption}

Although the first two assumptions are standard in the BO literature (cf.~\cite{RegretEI2017}) the last assumption is necessary to study the regret behaviour between outer epochs when encoder and decoder are re-trained. This assumption presents the sufficient conditions for achieving sub-linear regret. In section \ref{App: necessary_condition} we provide the necessary conditions.

\subsection{Assumption Validation on Empirical Domain}\label{App:AssumptionV}
On a high-level the Assumption~\ref{App:Assumption_3} postulates that as we retrain the VAE every $n$ BO acquisitions, we can find better and better latent points, in the sense that the probabilities of generating the global optimiser $\boldsymbol{x^*}$  when decoding these points increase and converge asymptotically to 1.

To illustrate this assumption we would ideally want to exhibit a sequence of latent points $(\boldsymbol{z}'(\ell))_{\ell=0,\dots,L}$ with $\boldsymbol{z}'(\ell)$ belonging to the latent space of the VAE obtained after
$\ell$ retraining steps, and such that the sequence of optimal generation probability $(p'(\ell) = \mathbb{P}[\boldsymbol{x}^* \sim g_{\boldsymbol{\theta}^*_\ell}(\cdot|\boldsymbol{z}'(\ell))])_{\ell=0,\dots,L}$ is increasing towards $1$.    

\paragraph{Experimental design:} We choose to illustrate the soundness of Assumption 3 on the Topology experiment as in this case we know the global optimiser $\boldsymbol{x^*}$ (that is the target image) and can therefore compute the probability of optimal generation from any latent points. As the property described in Assumption~\ref{App:Assumption_3} is asymptotic and we have a finite number of BO steps, we loosen the focus on optimum generation and consider near-optimum generation probability $\tilde{p}(\ell) = \mathbb{P}_{\boldsymbol{x} \sim g_{\boldsymbol{\theta}^*_\ell}(\cdot|\boldsymbol{\tilde{z}}(\ell))}[\text{sim}(\boldsymbol{x}, \boldsymbol{x}^*) > \alpha]$ where $\text{sim}(\boldsymbol{x}, \boldsymbol{x}^*)$ quantifies the similarity between $\boldsymbol{x}$ and $\boldsymbol{x}^*$  and $\alpha$ is a given threshold. We proceed as follows to obtain a sequence of latent points $(\boldsymbol{\tilde{z}}(\ell))_{\ell=0,\dots,10}$ for which the probabilities of generating near-optimal points $(\tilde{p}(\ell))_{\ell=0,\dots,10}$ is increasing.  
\begin{itemize}
    \item At the initial step $\ell=0$: we encode the target image $\boldsymbol{x^*}$ to get a latent point $\boldsymbol{\tilde{z}}(0)$ (encoder outputs a mean and a standard deviation and we set $\boldsymbol{\tilde{z}}(0)$ to the mean). To get the initial near-optimal recovery probability $\tilde{p}(0)$, we decode $\boldsymbol{\tilde{z}}^*(0)$ $100$ times back to the original space and approximate the exact probability by the frequency of near-optimal generation based on the similarity score and threshold ($\alpha=.8$).
    \item After retraining step $\ell+1$, we encode the target image $\boldsymbol{x^*}$ and use the mean and standard deviation output by the encoder to sample candidate latent points $(\boldsymbol{\tilde{z}}_i)_i$. For each candidate latent point $\boldsymbol{\tilde{z}}_i$, we compute it's associate $\tilde{p}_i$ using MC sampling as described for the initial step. We keep testing latent points until obtaining one latent point $\boldsymbol{\tilde{z}}_{i^*})$ for which $\tilde{p}_{i^*} > \tilde{p}(\ell)$ and we set $\boldsymbol{\tilde{z}}(\ell + 1)  = \boldsymbol{\tilde{z}}_{i^*}$ and $\tilde{p}(\ell + 1) = \tilde{p}_{i^*}$.
\end{itemize}

We consider the Assumption~\ref{App:Assumption_3} corroborated when the process above terminates as it then produce a sequence of latent points whose optimum recovery ability is increasing.
  	
\paragraph{Experimental results:} Repeating the experiment across 5 random seeds we were able to exhibit sequences of latent points $(\boldsymbol{\tilde{z}}(\ell))_{\ell=0,\dots,10}$ as described in our experimental design (average and standard deviation of $\tilde{p}(\ell)$ are reported Figure~\ref{fig:domain_recovery_experiment} for each retraining step $\ell$). Moreover, we note that we obtain higher probability of generating near-optimal points when searching in  latent spaces shaped with T-LBO than without (i.e. LBO). 

\begin{figure}
    \centering
    \includegraphics[scale=0.7]{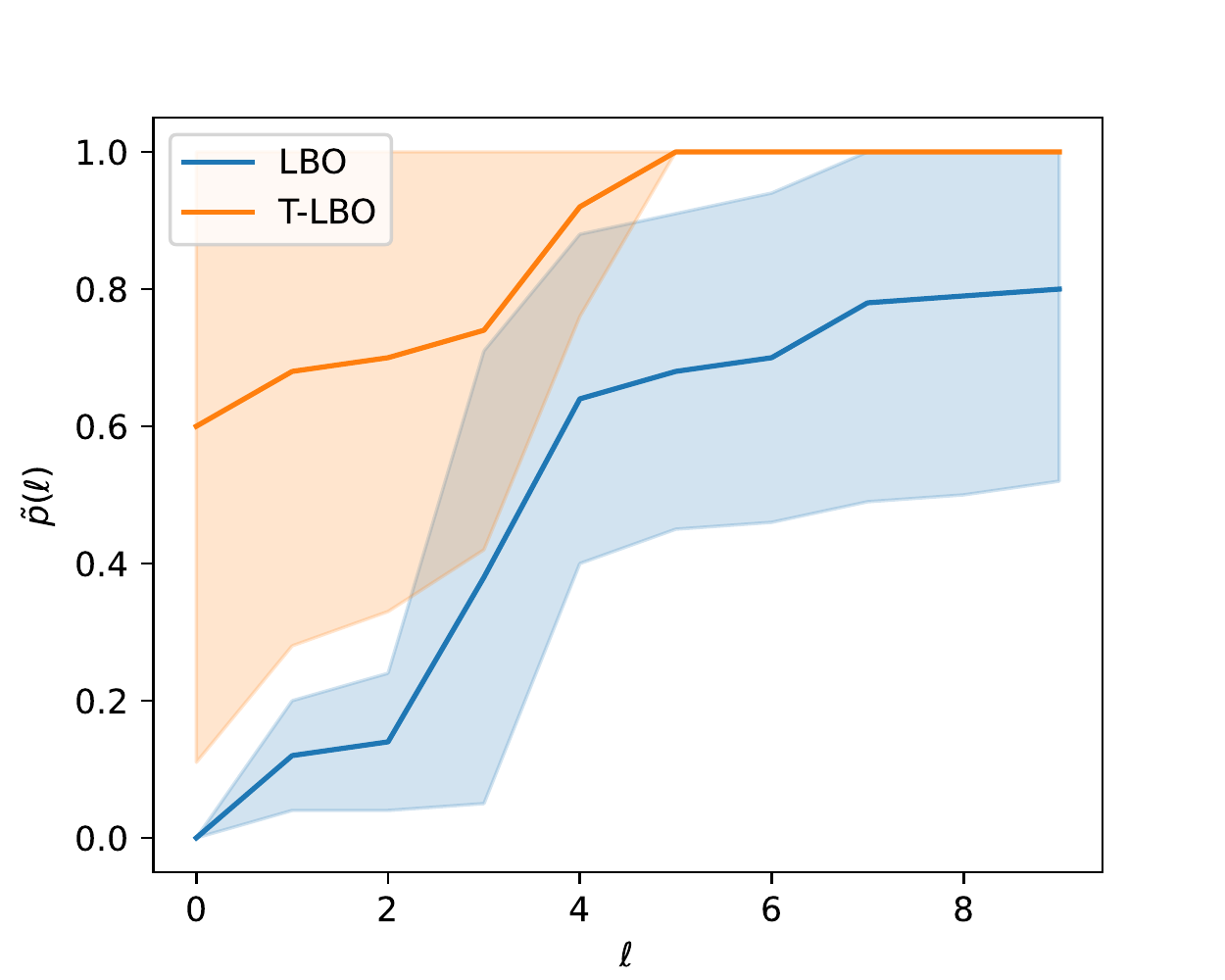}
    \caption{Domain Recovery results on Topology task. T-LBO is able to achieve $\tilde{p}(10)=1$ with no standard deviation where LBO only achieves about $\tilde{p}(10)=0.8$ with about 0.3 standard deviation.}
    \label{fig:domain_recovery_experiment}
\end{figure}

The intention of this domain recovery experiment is not to claim that Assumption~\ref{App:Assumption_3} holds for any general optimisation problem but rather that it holds in the setting where $\bm{x^*}$ is known. This is not an artefact of our proof but reflects rather the fact that knowledge of the black-box is assumed generally in such regret proofs, for example assumptions about the RKHS norm of the black-box in \citep{Srinivas2010}.

\subsection{Proof of Vanishing Regret}
In this section we present the proof of vanishing regret within the scope of Assumptions \ref{App:Assumption_1},\ref{App:Assumption_2} and \ref{App:Assumption_3}. We start by fixing the stochasticity induced by all encoders and decoders up to epoch $\ell$ and study the regret behaviour during $q$ consecutive steps of the BO procedure on this epoch. Then, using assumption 3, we consider the effect of switching the encoder and decoder as a result of re-training. Finally, we derive the optimal separation of the total evaluation budget $B$ between outer epochs $L$ and the total number of BO steps $q$. \\ 

Let us consider all stochasticity induced by the encoder and decoder during  epoch $\ell$. This can be formally defined as following collection of independent random variables:
\begin{enumerate}
    \item Stochasticity induced by the encoder when constructing the initial latent dataset $\mathcal{D}_{\mathbb{Z}}$ consisting of points  $\bm{z}_{1,\ell},\ldots, \bm{z}_{N_{\ell-1},\ell}$ with $N_{\ell-1}$ being the size of dataset $\mathcal{D}_{\mathbb{Z}}$ at the end of the $\ell-1^{th}$ epoch. This stochasticity is defined by a collection of i.i.d random variables $\bm{\zeta}_1,\ldots, \bm{\zeta}_{N_{\ell-1}}$. 
    \item Stochasticity induced by the decoder when constructing the primal outputs corresponding to new latent candidates during $q$ steps of the BO procedure. This stochasticity is defined by a collection of i.i.d random variables $\bm{\eta}_{\ell,0},\ldots, \bm{\eta}_{\ell, q-1}$ associated with the decoder.
\end{enumerate}
For clarity, we denote the stochasticity defined by $\bm{\zeta}_1,\ldots, \bm{\zeta}_{N_{\ell-1}}$ as $\mathbb{A}_{\ell}$ and share it across all $q$ steps of BO routine. We denote the stochasticity induced by $\bm{\eta}_{\ell,0},\ldots, \bm{\eta}_{\ell, k-1}$ in the first $k$ steps of the BO procedure as $\mathbb{B}_{\ell,k}$, so $\mathbb{B}_{\ell,0} \subseteq \mathbb{B}_{\ell, 1} \subseteq \ldots \subseteq \mathbb{B}_{\ell,q-1}$ and combine these collections in $\mathbb{U}_{\ell} = \{\mathbb{A}_1, \ldots, \mathbb{A}_{\ell}\}$ and $\mathbb{V}_{\ell,k} = \{\mathbb{B}_{1,q-1}, \ldots, \mathbb{B}_{\ell-1,q-1},\mathbb{B}_{\ell,k}\}$ - all stochasticity introduced by the generative model for the first full $\ell-1$ epochs and first $k$ inner iterations of BO look at $\ell^{th}$ epoch. Following the definition of cumulative regret we define the notion of stochastic cumulative regret at epoch $\ell^{th}$ as 
\begin{equation*}
    R_{\ell}(\mathbb{U}_{\ell}, \mathbb{V}_{\ell,q-1}) = \sum_{k=0}^{q-1}f(\bm{x}^{*}) - f_{\text{latent}}(\hat{\bm{z}}_{\ell,k+1}(\mathbb{U}_{\ell}, \mathbb{V}_{\ell,k})),
\end{equation*}
where $\bm{x}^* = \arg\max_{\bm{x}\in\mathcal{X}}f(\bm{x})$ is a global maximiser of the black-box objective,  $\hat{\bm{z}}_{\ell,k+1}(\mathbb{U}_{\ell},\mathbb{V}_{\ell,k}) = \arg\max_{\bm{z}\in\mathcal{Z}}\alpha_{\text{EI}}(\bm{z}|\mathcal{D}_{\mathbb{Z}}(\mathbb{U}_{\ell}, \mathbb{V}_{\ell,k}))$ is the latent point obtained by maximising the the  Expected Improvement acquisition function defined for observations $\mathcal{D}_{\mathbb{Z}}(\mathbb{U}_{\ell}, \mathbb{V}_{\ell,k}))$, and $f_{\text{latent}}(\bm{z}) = \mathbb{E}_{\bm{x}\sim g_{\bm{\theta}^*_{\ell}(\cdot|\bm{z})}}\left[{f(\bm{x})}\right]$ Given this definition, it is easy to see that the overall cumulative regret after $L$ epochs is:
\begin{equation*}
    \text{Regret}_{L,q}(\langle \hat{\bm{z}}_{\ell,k}(\mathbb{U}_{L},\mathbb{V}_{L,q-1})\rangle) = \sum_{\ell=1}^{L}R_{\ell}(\mathbb{U}_{\ell}, \mathbb{V}_{\ell,q-1}).
\end{equation*}
where due to the relation $\mathbb{U}_{\ell} \subset \mathbb{U}_{L}$ and $\mathbb{V}_{\ell, k} \subset \mathbb{V}_{L, q-1}$ we have   $\hat{\bm{z}}_{\ell,k}(\mathbb{U}_{L},\mathbb{V}_{L,q-1}) = \hat{\bm{z}}_{\ell,k}(\mathbb{U}_{\ell},\mathbb{V}_{\ell,k-1})$. To analyse the regret bound we first investigate the behaviour of the regret $R_{\ell}(\mathbb{U}_{\ell}, \mathbb{V}_{\ell,q-1})$. To do so, let us fix some realisation of all random variables collected in $\mathbb{U}_{L}, \mathbb{V}_{L,q-1}$. We denote these realisations as $U_{L}$ and $V_{L,q-1}$ respectively. Note for these fixed realisations, the dataset $\mathcal{D}_{\mathbb{Z}} = \mathcal{D}_{\mathbb{Z}}(U_{L}, V_{L, q-1})$ at any inner iteration $k$ consists of latent points (defined by fixed $U_{\ell}\subset U_{L}, V_{\ell,k}\subset V_{L,q-1}$) and the corresponding black-box function evaluation $f(\bm{x}_{\ell,k} = g_{\bm{\theta}^*_{\ell}}(\bm{z}_{\ell,k}; \bm{\eta}_{\ell, k}))$ distorted only by the observation noise  $\epsilon_k\sim\mathcal{N}(\cdot; 0,\sigma^2_{\text{noise}})$ (Assumption \ref{App:Assumption_1}). Next, for a fixed realisations $U_L, V_{L,q-1}$ we establish the following:
\begin{lemma}\label{App:additional_lemma_one}
Let $\tau,\delta_0\in(0,1)$ be the stopping criterion and confidence parameter in Algorithm \textcolor{blue}{1}. Consider fixed realisations $U_{\ell}, V_{\ell, q-1}$ and let Assumptions \ref{App:Assumption_1},\ref{App:Assumption_2},\ref{App:Assumption_3} hold. Then, for any epoch $\ell \ge \tilde{\ell}$ with probability at least $1 - 2q\delta_0$ we have:
\begin{align}\label{Lemma_one_result}
    \sum_{k=0}^{q-1}&\left[f(\bm{x}^*) - f_{\text{latent}}(\hat{\bm{z}}_{\ell, k+1}(U_{\ell}, V_{\ell, k}))\right] 
    \le  ~q\Upsilon(\ell)G_f + \mathcal{O}(\sqrt{q}\log^{d+2.5} q),
\end{align}
where $\hat{\bm{z}}_{\ell,k+1}(U_{\ell}, V_{\ell, k}) = \arg\max_{\bm{z}\in\mathcal{Z}}\alpha_{\text{EI}}(\bm{z}|\mathcal{D}_{\mathbb{Z}}(U_{\ell}, V_{\ell, k-1}))$ - is a latent candidate returned in the $k+1^{th}$ step of BO procedure that corresponds to observations $\mathcal{D}_{\mathbb{Z}}(U_{\ell}, V_{\ell, k})$ associated with observations  $U_{\ell}\subset U_{L}$ and $V_{\ell, k-1}\subset V_{L, q-1}$.
\end{lemma}

\begin{proof}
Let $\bm{x}^*_{\ell}(\Upsilon(\ell)) = \arg\max_{\bm{x}: \exists \bm{z}^*_{\ell}\in\mathcal{Z}\text{ s.t. } \mathbb{P}[\bm{x}\sim g_{\bm{\theta}^*_{\ell}}(\cdot|\bm{z}^*_{\ell})] \ge 1 - \Upsilon(\ell)}f(\bm{x})$ be the best primal point that can be recovered using the decoder at epoch $\ell$ with probability at least $1 - \Upsilon(\ell)$. Assumption \ref{App:Assumption_3} guarantees that for any fixed $U_{\ell}, V_{\ell,k}$ collection  $\{\bm{x}: \exists \bm{z}^*_{\ell}\in\mathcal{Z}\text{ s.t. } \mathbb{P}[\bm{x}\sim g_{\bm{\theta}^*_{\ell}}(\cdot|\bm{z}^*_{\ell})] \ge 1 - \Upsilon(\ell)\}$ is not empty, because at least $\bm{x}^*$ belongs to this collection. Moreover, by the definition of $\bm{x}^*_{\ell}(\Upsilon(\ell))$ we have:
\begin{align*}
    \bm{x}^*_{\ell}(\Upsilon(\ell)) = \bm{x}^*, \ \ \ \forall \ell \ge \tilde{\ell}
\end{align*}
Hence, for any epoch $\ell\ge \tilde{\ell}$ we can write:
\begin{align*}
    R_{\ell}(U_{\ell}, V_{\ell,q-1}) &= \sum_{k=0}^{q-1}\big[f(\bm{x}^*) - f_{\text{latent}}(\hat{\bm{z}}_{\ell,k+1}(U_{\ell}, V_{\ell,k}))\big] \\ &=\sum_{k=0}^{q-1}\big[f(\bm{x}^*) - f(\bm{x}^*_{\ell}(\Upsilon(\ell))) + f(\bm{x}^*_{\ell}(\Upsilon(\ell))) -  f_{\text{latent}}(\hat{\bm{z}}_{\ell,k+1}(U_{\ell}, V_{\ell,k}))\big] \\\nonumber
    &= \sum_{k=0}^{q-1}(f(\bm{x}^*) - f(\bm{x}^*) + f(\bm{x}^*) -  f_{\text{latent}}(\hat{\bm{z}}_{\ell,k+1}(U_{\ell}, V_{\ell,k}))) \\ \nonumber
    &= \sum_{k=0}^{q-1}(f(\bm{x}^*) -  f_{\text{latent}}(\hat{\bm{z}}_{\ell,k+1}(U_{\ell}, V_{\ell,k}))) \\\nonumber
    &= \sum_{k=0}^{q-1}(\underbrace{f(\bm{x}^*) - \xi_{\ell,k}}_{A_{\ell,k}(U_{\ell}, V_{\ell,k})} + \underbrace{\xi_{\ell,k} -   f_{\text{latent}}(\hat{\bm{z}}_{\ell,k+1}(U_{\ell}, V_{\ell,k}))}_{B_{\ell,k}(U_{\ell}, V_{\ell,k})})
\end{align*}
where $\xi_{\ell,k} = \xi_{\ell,k}(U_{\ell}, V_{\ell, k})$ is the maximum black-box function value observed so far (in the first $k$ inner BO steps ) at epoch $\ell$. Note that this value also depends on realisations $U_{\ell}, V_{\ell, k}$, but for brevity we use $\xi_{\ell, k}$ for this value. Now, let us study each term separately. 
\begin{enumerate}
    \item Because  $\bm{x}^* = \bm{x}^*_{\ell}(\Upsilon(\ell))$ can be recovered with probability at least $1 - \Upsilon(\ell)$ we have that $|f_{\text{latent}}(\bm{z}^*_{\ell}) - f(\bm{x}^*)| \le \Upsilon(\ell) G_f$. Hence, 
    \begin{align*}
        A_{\ell,k}(U_{\ell}, V_{\ell,k}) &= f(\bm{x}^*) - \xi_{\ell,k} \le \Upsilon(\ell)G_{f} + f_{\text{latent}}(\bm{z}^*_{\ell}) - \xi_{\ell,k}
    \end{align*}
    Lemma 6 from ~\cite{RegretEI2017} gives, that with probability at least $1 - \delta_0$ for any $\bm{z}\in\mathcal{Z}$:
    \begin{align*}
        \text{ReLU}(f_{\text{latent}}(\bm{z}) - \xi_{\ell,k}) - \sqrt{\beta_k}\sigma_{f_{\text{latent},\ell,k}}(\bm{z}) \le  \mathbb{E}_{f_{\text{latent}}(\cdot)|\mathcal{D}_{\mathbb{Z}}(U_{\ell}, V_{\ell,k})}\left[\text{ReLU}(f_{\text{latent}}(\bm{z}) - \xi_{\ell,k})\right]
    \end{align*}
    where for the squared exponential kernel $\text{k}_{\text{SE}}$ we have $\beta_k = \mathcal{O}\left(\log^{d+1} q\log^3\left[\frac{k}{\delta_0}\right]\right)$, $\sigma^2_{f_{\text{latent},\ell,k}}(\bm{z}|\mathcal{D}_{\mathbb{Z}}(U_{\ell}, V_{\ell,k}))$ is the posterior variance associated with $f_{\text{latent}}(\cdot)$ based on observations $\mathcal{D}_{\mathbb{Z}}(U_{\ell}, V_{\ell,k})$. Hence,
    we have with probability at least $1 - \delta_0$:
    \begin{align*}
        f_{\text{latent}}(z^*_{\ell}) - \xi_{\ell,k} &\le \sqrt{\beta_{k}}\sigma_{f_{\text{latent}}, \ell,k}(\bm{z}^*_{\ell}|\mathcal{D}_{\mathbb{Z}}(U_{\ell}, V_{\ell,k})) + \mathbb{E}_{f_{\text{latent}}(\cdot)|\mathcal{D}_{\mathbb{Z}}(U_{\ell}, V_{\ell,k} )}\left[\text{ReLU}(f_{\text{latent}}(\bm{z}^*_{\ell}) - \xi_{\ell,k})\right]  \\\nonumber
        &\le^{1} \sqrt{\beta_{k}}\sigma_{f_{\text{latent}},\ell, k}(\bm{z}^*_{\ell}|\mathcal{D}_{\mathbb{Z}}(U_{\ell}, V_{\ell,k}))+ \alpha_{\text{EI}}(\hat{\bm{z}}_{\ell, k+1}|\mathcal{D}_{\mathbb{Z}}(U_{\ell}, V_{\ell, k})) \\\nonumber
        &=^{2}\sigma_{f_{\text{latent}},\ell,k}(\hat{\bm{z}}_{\ell,k+1}(U_{\ell}, V_{\ell, k})|\mathcal{D}_{\mathbb{Z}}(U_{\ell}, V_{\ell,k}))\times \nu(s_{\ell,k}(\hat{\bm{z}}_{\ell,k+1}(U_{\ell}, V_{\ell, k}))) \\\nonumber
        & ~~~~~~~~ + \sqrt{\beta_{k}}\sigma_{f_{\text{latent}},\ell, k}(\bm{z}^*_{\ell}|\mathcal{D}_{\mathbb{Z}}(U_{\ell}, V_{\ell,k})),
    \end{align*}
    where we use notation $s_{\ell,k}(\hat{\bm{z}}_{\ell,k+1}(U_{\ell}, V_{\ell,k})) = \frac{\mu_{f_{\text{latent}},\ell,k}(\hat{\bm{z}}_{\ell,k+1}(U_{\ell}, V_{\ell,k})|\mathcal{D}_{\mathbb{Z}}(U_{\ell}, V_{\ell,k})) - \xi_{\ell,k}}{\sigma_{f_{\text{latent}},\ell, k}(\hat{\bm{z}}_{\ell,k+1}(U_{\ell}, V_{\ell,k})|\mathcal{D}_{\mathbb{Z}}(U_{\ell}, V_{\ell,k}))}$
    with $\mu_{f_{\text{latent}},\ell,k}(\bm{z}|\mathcal{D}_{\mathbb{Z}}(U_{\ell}, V_{\ell,k}))$ the posterior mean associated with $f_{\text{latent}}(\cdot)$ based on observations $\mathcal{D}_{\mathbb{Z}}(U_{\ell}, V_{\ell,k})$, function $\nu(s) = s\Phi(s) + \phi(s)$ with $\Phi(s),\phi(s)$ being c.d.f. and p.d.f.of a standard univariate Gaussian respectively, in step 1 we use the definition of $\hat{\bm{z}}_{\ell,k+1}(U_{\ell}, V_{\ell, k})$, and in step 2 we use the result of Lemma 1 from ~\cite{RegretEI2017}.
    Hence, we can bound the term $A_{\ell,k}(U_{\ell}, V_{\ell,k})$ as follows:
    \begin{align*}
        A_{\ell,k}(U_{\ell}, V_{\ell,k}) \le & ~\Upsilon(\ell)G_f + \sqrt{\beta_{k}}\sigma_{f_{\text{latent}},\ell, k}(\bm{z}^*_{\ell}|\mathcal{D}_{\mathbb{Z}}(U_{\ell}, V_{\ell,k})) \\
        & + \sigma_{f_{\text{latent}},\ell,k}(\hat{\bm{z}}_{\ell,k+1}(U_{\ell}, V_{\ell, k})|\mathcal{D}_{\mathbb{Z}}(U_{\ell}, V_{\ell,k})) \nu(s_{\ell,k}(\hat{\bm{z}}_{\ell,k+1}(U_{\ell}, V_{\ell, k}))).
    \end{align*}
    with probability at least $1 - \delta_0$.
    Hence, we have (with probability at least $1 - q\delta_0$ ):
    \begin{align*}
        \sum_{k=0}^{q-1}A_{\ell,k}(U_{\ell}, V_{\ell,k}) &\le^{1}
        q\Upsilon(\ell)G_f + \sum_{k=0}^{q-1}\sqrt{\beta_{k}}\sigma_{f_{\text{latent}},\ell, k}(\bm{z}^*_{\ell}|\mathcal{D}_{\mathbb{Z}}(U_{\ell}, V_{\ell,k})) \\\nonumber
        & ~~~~ + \sum_{k=0}^{q-1}\sigma_{f_{\text{latent}},\ell,k}(\hat{\bm{z}}_{\ell,k+1}(U_{\ell}, V_{\ell,k})|\mathcal{D}_{\mathbb{Z}}(U_{\ell}, V_{\ell,k})) \nu(s_{\ell,k}(\hat{\bm{z}}_{\ell,k+1}(U_{\ell}, V_{\ell,k}))) \\ 
        & \le^1 q\Upsilon(\ell)G_f + \mathcal{O}\left[\sqrt{\frac{q\beta_q\log^{d+1}q}{\log(1 + \sigma^{-2}_{\text{noise}})}}\right]\\
        & ~~~~ + \sum_{k=0}^{q-1}\sigma_{f_{\text{latent}},\ell,k}(\hat{\bm{z}}_{\ell,k+1}(U_{\ell}, V_{\ell,k})|\mathcal{D}_{\mathbb{Z}}(U_{\ell}, V_{\ell,k})) \nu(s_{\ell,k}(\hat{\bm{z}}_{\ell,k+1}(U_{\ell}, V_{\ell,k})))
    \end{align*}
    where in step 1 we use the result of Lemma 7 in ~\cite{RegretEI2017}, Cauchy-Schwartz inequality. 
    From result of Lemma 9 in ~\cite{RegretEI2017} it follows that for stopping criteria $\tau < \sqrt{\frac{g_1}{2\pi}}$ we have:
    \begin{align*}
        &\nu(s_{\ell,k}(\hat{\bm{z}}_{\ell,k+1}(U_{\ell}, V_{\ell,k}))) \le 1 + \log\left(\frac{G^2_1}{2\pi\tau^2}\right)
    \end{align*}
    Hence, applying the result of Lemma 7 in ~\cite{RegretEI2017}, Cauchy-Schwartz inequality gives:
    \begin{align}\label{interm_bound_ten_new}
        &\sum_{k=0}^{q-1}A_{\ell,k}(U_{\ell}, V_{\ell,k}) \le q\Upsilon(\ell)G_f + \mathcal{O}\left[\sqrt{\frac{q\log^{d+1}q}{\log(1 + \sigma^{-2}_{\text{noise}})}}\right] \left[\sqrt{\beta_q} + \left[1 + \log\left(\frac{G^2_1}{2\pi\tau^2}\right)\right]\right].
    \end{align}
    with probability at least $1 - q\delta_0$.

    \item For the term $B_{\ell,k}(U_{\ell}, V_{\ell,k})$ we have, with probability at least $1 - \delta_0$:
    \begin{align*}
        B_{\ell,k}(U_{\ell}, V_{\ell,k}) &= \xi_{\ell,k} - \mu_{f_{\text{latent}},\ell, k}(\hat{\bm{z}}_{\ell,k+1}(U_{\ell}, V_{\ell,k})|\mathcal{D}_{\mathbb{Z}}(U_{\ell}, V_{\ell,k})) + \mu_{f_{\text{latent}},\ell, k}(\hat{\bm{z}}_{\ell,k+1}(U_{\ell}, V_{\ell,k})|\mathcal{D}_{\mathbb{Z}}(U_{\ell}, V_{\ell,k})) \\
        & ~~~~ - f_{\text{latent}}(\hat{\bm{z}}_{\ell,k+1}(U_{\ell}, V_{\ell,k}))
    \end{align*}
    Using definition of  $s_{\ell,k}(\hat{\bm{z}}_{\ell,k+1}(U_{\ell}, V_{\ell,k})) = \frac{\mu_{f_{\text{latent}},\ell,k}(\hat{\bm{z}}_{\ell,k+1}(U_{\ell}, V_{\ell,k})|\mathcal{D}_{\mathbb{Z}}(U_{\ell}, V_{\ell,k})) - \xi_{\ell,k}}{\sigma_{f_{\text{latent}},\ell, k}(\hat{\bm{z}}_{\ell,k+1}(U_{\ell}, V_{\ell,k})|\mathcal{D}_{\mathbb{Z}}(U_{\ell}, V_{\ell,k}))}$, result of Theorem 6 in ~\cite{Srinivas2010}, and the fact that $\nu(s)- \nu(-s) = s$ we have:
    \begin{align*}
        B_{\ell,k}(U_{\ell}, V_{\ell,k}) &\le^{1} \xi_{\ell,k} - \mu_{f_{\text{latent}},\ell, k}(\hat{\bm{z}}_{\ell,k+1}(U_{\ell}, V_{\ell,k})|\mathcal{D}_{\mathbb{Z}}(U_{\ell}, V_{\ell,k})) + \sqrt{\beta_k}\sigma_{f_{\text{latent}},\ell, k}(\hat{\bm{z}}_{\ell,k+1}(U_{\ell}, V_{\ell,k})|\mathcal{D}_{\mathbb{Z}}(U_{\ell}, V_{\ell,k}))\\
        &=^{2} -s_{\ell,k}(\hat{\bm{z}}_{\ell,k+1}(U_{\ell}, V_{\ell,k}))\times \sigma_{f_{\text{latent}},\ell, k}(\hat{\bm{z}}_{\ell,k+1}(U_{\ell}, V_{\ell,k})|\mathcal{D}_{\mathbb{Z}}(U_{\ell}, V_{\ell,k}))\\
        & \hspace{20em} + \sqrt{\beta_k}\sigma_{f_{\text{latent}},\ell, k}(\hat{\bm{z}}_{\ell,k+1}(U_{\ell}, V_{\ell,k})|\mathcal{D}_{\mathbb{Z}}(U_{\ell}, V_{\ell,k})) \\\nonumber
        &=^{3} \Big[\nu(-s_{\ell,k}(\hat{\bm{z}}_{\ell,k+1}(U_{\ell},V_{\ell,k}))) + \sqrt{\beta_k} - \nu(s_{\ell,k}(\hat{\bm{z}}_{\ell,k+1}(U_{\ell},V_{\ell,k})))\Big] \\
        & \hspace{25em}\sigma_{f_{\text{latent}},\ell, k}(\hat{\bm{z}}_{\ell,k+1}(U_{\ell}, V_{\ell,k})|\mathcal{D}_{\mathbb{Z}}(U_{\ell}, V_{\ell,k}))
    \end{align*}
    Hence, with probability at least $1 - q\delta_0$ we have:
    \begin{align*}
        \sum_{k=0}^{q-1}B_{\ell,k}(U_{\ell}, V_{\ell,k}) &\le \sum_{k=0}^{q-1}\sigma_{f_{\text{latent}},\ell, k}(\hat{\bm{z}}_{\ell,k+1}(U_{\ell},V_{\ell,k})|\mathcal{D}_{\mathbb{Z}}(U_{\ell}, V_{\ell,k}))  \times \Big[\nu(-s_{\ell,k}(\hat{\bm{z}}_{\ell,k+1}(U_{\ell},V_{\ell,k}))) + \sqrt{\beta_k} \\ 
        & \hspace{25em} - \nu(s_{\ell,k}(\hat{\bm{z}}_{\ell,k+1}(U_{\ell},V_{\ell,k})))\Big] \\\nonumber
        &\le^{1} \sum_{k=0}^{q-1}\sigma_{f_{\text{latent}},\ell, k}(\hat{\bm{z}}_{\ell,k+1}(U_{\ell},V_{\ell,k})|\mathcal{D}_{\mathbb{Z}}(U_{\ell}, V_{\ell,k}))\bigg[ \sqrt{\beta_q} + 2\left[1 + \log\left(\frac{G^2_1}{2\pi\tau^2}\right) \right]\bigg]  \\\nonumber
        &\le^{2}\sqrt{
        \begin{aligned}
        &q\sum_{k=0}^{q-1}\sigma^2_{f_{\text{latent}},\ell, k}(\hat{\bm{z}}_{\ell,k+1}(U_{\ell},V_{\ell,k})|\mathcal{D}_{\mathbb{Z}}(U_{\ell}, V_{\ell,k}))\left( \sqrt{\beta_q} + 2\left[1 + \log\left(\frac{G^2_1}{2\pi\tau^2}\right) \right]\right)^2
        \end{aligned}
        }\\
        &\le^{3}\sqrt{
        \begin{aligned}
        &3q\sum_{k=0}^{q-1}\sigma^2_{f_{\text{latent}},\ell, k}(\hat{\bm{z}}_{\ell,k+1}(U_{\ell},V_{\ell,k})|\mathcal{D}_{\mathbb{Z}}(U_{\ell}, V_{\ell,k}))\left(\beta_q + 8\left[1 + \log^2\left(\frac{G^2_1}{2\pi\tau^2}\right) \right]\right)
        \end{aligned}
        },\\
    \end{align*}
    where in step 1 we use the result of Lemma 9 in ~\cite{RegretEI2017} (for stopping criteria $\tau < \sqrt{\frac{g_1}{2\pi}}$), in step 2 we use the Cauchy-Schwartz inequality, in step 3 we use $(a + b + c)^2 \le 3(a^2 + b^2 + c^2)$. Applying the result of Lemma 7 in ~\cite{RegretEI2017}  eventually gives
    \begin{align}\label{interm_bound_eleven_new}
        &\sum_{k=0}^{q-1}B_{\ell,k}(U_{\ell}, V_{\ell,k}) \le\sqrt{
        \begin{aligned}
        &3q\left[ \beta_q + 8\left[1 + \log^2\left(\frac{G^2_1}{2\pi\tau^2}\right) \right]\right] \mathcal{O}\left[\frac{\log^{d+1} q}{\log(1 + \sigma^{-2}_{\text{noise}})}\right]
        \end{aligned}
        }
    \end{align},
    with probability at least $1 - q\delta_0$.
\end{enumerate}
Combining results (\ref{interm_bound_ten_new}) and (\ref{interm_bound_eleven_new}) and using an asymptotic rate for $\beta_q = \mathcal{O}(\log^{d+4} q)$, gives \begin{align*}
    &\sum_{k=0}^{q-1}(f(\bm{x}^*) - f_{\text{latent}}(\hat{\bm{z}}_{\ell, k+1}(U_{\ell}, V_{\ell, k})))\le q\Upsilon(\ell)G_f + \mathcal{O}(\sqrt{q}\log^{d + 2.5} q),
\end{align*}
with probability at least $1 - 2q\delta_0$.
\end{proof}
Using the result of Lemma \ref{App:additional_lemma_one} for any realisation $U_{L}, V_{L, q-1}$ with probability at least $1 - 2Lq\delta_0$ we have: 
\begin{align*}
    \text{Regret}_{L,q}(\langle \hat{\bm{z}}_{\ell,k}(U_{L},V_{L,q-1})\rangle_{\ell,k}^{L, q}) &=\sum_{\ell=1}^{L}R_{\ell}(U_{\ell}, V_{\ell,q-1})  = \sum_{\ell=1}^{\tilde{\ell}-1}R_{\ell}(U_{\ell}, V_{\ell,q-1}) + \sum_{\ell=\tilde{\ell}}^{L}R_{\ell}(U_{\ell}, V_{\ell,q-1}) \\\nonumber
    &\le^{1} 2qG_f(\tilde{\ell}-1) + \sum_{\ell=\tilde{\ell}}^{L}R_{\ell}(U_{\ell}, V_{\ell,q-1}) \\
    & \le^2 2qG_f(\tilde{\ell} - 1) + qG_f\sum_{\ell = \tilde{\ell}}^L\Upsilon(\ell) + \mathcal{O}(L\sqrt{q}\log^{d + 2.5} q)\\
    & \le^3 2qG_f(\tilde{\ell}-1) + G_fB\frac{\int^L_{0}\Upsilon(a)da}{L} + \mathcal{O}\left[\frac{B}{\sqrt{q}}\log^{d + 2.5} q\right].
\end{align*}
In step 1 we use Assumption \ref{App:Assumption_1}, in step 2 we use $L - \tilde{\ell}\le L = \frac{B}{q}$ in step 3 we use that for decreasing positive valued function   $\sum_{\ell=\tilde{\ell}}^L\Upsilon(\ell) < \sum_{\ell=1}^L\Upsilon(\ell) \le \int_{0}^L\Upsilon(a)da$. Because these result holds for any realisation of $\mathbb{U}_{L}, \mathbb{V}_{L,q-1}$ we have, that with probability at least $1 - 2Lq\delta_0$:
\begin{align*}
    \text{Regret}_{L,q}(\langle \hat{\bm{z}}_{\ell,k}(\mathbb{U}_{L},\mathbb{V}_{L,q-1})\rangle_{\ell,k}^{L, q}) \le 2qG_f(\tilde{\ell}-1) + G_fB\frac{\int^L_{0}\Upsilon(a)da}{L} + \mathcal{O}\left[\frac{B}{\sqrt{q}}\log^{d + 2.5} q\right].
\end{align*}
Hence, for the averaged cumulative regret:
\begin{align*}
    &\frac{1}{B}\text{Regret}_{L,q}(\langle \hat{\bm{z}}_{\ell,k}(\mathbb{U}_{L},\mathbb{V}_{L,q-1})\rangle_{\ell,k}^{L, q}) \le  \frac{2qG_f(\tilde{\ell}-1)}{B} + G_f\frac{\int^L_{0}\Upsilon(a)da}{L} +\mathcal{O}\left[\frac{\log^{d + 2.5} q}{\sqrt{q}}\right].
\end{align*}
Choosing $q = \lceil B^{\frac{2}{3}}\rceil$ and applying that $\lim_{B\to\infty}\frac{\int^{\lceil B^{\frac{1}{3}}\rceil}_{0}\Upsilon(a)da}{\lceil B^{\frac{1}{3}}\rceil}\to 0$ (due to Assumption \ref{App:Assumption_3}) gives:
\begin{align*}
    \lim_{B\to\infty} \frac{1}{B}\text{Regret}_{L, q}(\langle \hat{\bm{z}}_{\ell, k}\rangle_{\ell,k}^{L, q})  = 0,
\end{align*}
with probability at least $1 - 2B\delta_0$. Finally choosing $\delta\in(0, \min\{1, 2\delta_0 B\})$ establishes the statement of the theorem.

\subsection{Necessary Conditions for Vanishing Regret}\label{App: necessary_condition}
According to Theorem \textcolor{blue}{1} we see that Assumption \ref{App:Assumption_3} provides sufficient conditions for vanishing regret. In this section, we study the necessary conditions for vanishing regret. \\

Let us consider an underlying black-box function defined over a bounded input space $\mathcal{X}$ with a unique global maximiser $\bm{x}^*$ and maximum value $f(\bm{x}^*)$ isolated from the rest of the range of function $f(\cdot)$, i.e. $\exists c > 0: \ f(\bm{x}^*) - \max_{\bm{x}\in\mathcal{X}\setminus\{\bm{x}^*\}}f(\bm{x}) = c$. Assume that the optimal point $\bm{x}^*$ cannot be recovered by the generative model in the sense of Assumption \ref{App:Assumption_3}. In other words, assume that among any number of epochs indexed from $1$ to  $L$,  there is a collection of indices $\mathcal{Y}(L) = \{\ell^{'}_1, \ell^{'}_2, \ldots, \ell^{'}_{|\mathcal{Y}(L)|}\}$ such that as $\lim_{L\to\infty}|\mathcal{Y}(L)| = \infty$, and on these epochs the Assumption \ref{App:Assumption_3} does not hold for global maximiser $\bm{x}^*$:
\begin{align*}
    \forall \ell^{'}\in \mathcal{Y}(L) \ \ \ \forall \bm{z}\in\mathcal{Z} \ \ \ \mathbb{P}\left[\bm{x}^*\sim g_{\bm{\theta}_{\ell^{'}}}(\cdot |\bm{z}) \right] \le \delta_1
\end{align*}  
for some positive constant $\delta_1\in(0,1)$.Then, for such epochs we have:
\begin{align*}
    f_{\text{latent}}(\bm{z}) = \mathbb{E}_{\bm{x}\sim g_{\bm{\theta}_{\ell^{'}}(\cdot|\bm{z})}}\left[f(\bm{x})\right]\le \delta_1f(\bm{x}^*) + (1 - \delta_1)(f(\bm{x}^*) - c) = f(\bm{x}^*) - c(1 - \delta_1), ~~~~ \forall\bm{z}\in\mathcal{Z}.
\end{align*}
Hence, for epochs $\ell^{'}_i\in\mathcal{Y}(L)$ we have $f(\bm{x}^*) - f_{\text{latent}}(\bm{z}) \ge c(1 - \delta_1)$ for all $\bm{z}\in\mathcal{Z}$. Hence, for the cumulative regret over $L$ iterations we have:
\begin{align*}
    \sum_{\ell=1}^L\sum_{k=0}^{q-1}\left[f(\bm{x}^*) - f_{\text{latent}}(\hat{\bm{z}}_{\ell,k})\right] &= \sum_{i=1}^{|\mathcal{Y}(L)|}\sum_{k=0}^{q-1}\left[f(\bm{x}^*) - f_{\text{latent}}(\hat{\bm{z}}_{\ell^{'}_i,k})\right] + \sum_{\ell\notin \mathcal{Y}(L)}\sum_{k=0}^{q-1}\left[f(\bm{x}^*) - f_{\text{latent}}(\hat{\bm{z}}_{\ell,k})\right] \\\nonumber
    &\ge \sum_{i=1}^{|\mathcal{Y}(L)|}\sum_{k=0}^{q-1}\left[f(\bm{x}^*) - f_{\text{latent}}(\hat{\bm{z}}_{\ell^{'}_i,k})\right]  \ge q|\mathcal{Y}(L)|c(1 - \delta_1)
\end{align*}
Hence, for the average cumulative regret we have:
\begin{align*}
    &\sum_{\ell=1}^L\sum_{k=0}^{q-1}\left[f(\bm{x}^*) - f_{\text{latent}}(\hat{\bm{z}}_{\ell,k})\right] \ge \frac{|\mathcal{Y}(L)|}{L}c(1 - \delta_1)
\end{align*}
Now, if $\lim_{L\to\infty}\frac{|\mathcal{Y}(L)|}{L} = h$ for some $h> 0$, we have that the average cumulative regret is not sub-linear. In other words, the necessary condition to guarantee sub-linear regret, is to ensure that the portion of epochs $\mathcal{Y}(L)$ is asymptotically small in comparison with $L$, i.e. $|\mathcal{Y}(L)| = o(L)$.

\section{Broader impact}\label{App:impacts}



With reference to the NeurIPS ethics guidelines, the work presented in this paper is liable to impact society through deployed applications rather than as a standalone method. From an application perspective, our contribution may be summarised as an improvement to the state-of-the-art in high-dimensional Bayesian optimisation over structured input spaces. A stark and topical example of such a problem, at the time of writing, is the search for antiviral drugs for the COVID-19 virus \cite{moonshot}. Indeed, the gravity of the current global crisis underlines the importance of high-dimensional optimisation problems over structured inputs such as drug molecules as well as the need for sample efficiency to expedite the resolution of the crisis. Given that our method may garner use in a range of fields, here, we choose three case studies to illustrate potential positive and negative impacts of our research:

\begin{enumerate}
    \item \textbf{Molecule and Materials Discovery:} Bayesian optimisation methodologies hold great promise for accelerating the discovery of molecules and materials \cite{2018_Song, 2019_Griffiths, 2018_Hase, 2020_Thawani, 2020_Felton, 2020_Cheng}. That being said, the societal effects of novel molecules and materials may range from decreased mortality due to a more diverse set of active drug molecules to a broader array of chemical and biological weapons. On this latter point, as with previous work on high-dimensional Bayesian optimisation \cite{2020_Tripp}, we would hope due to additional demands on scientific infrastructure that our machine learning technology alone would not be sufficient to incite individuals to commence production of weapons.
    \item \textbf{Machine Learning Hyperparameter Tuning:} Machine learning model hyperparameter tuning is a relevant use-case for Bayesian optimisation in the machine learning community \cite{2020_Kandasamy, 2020_Cowen, 2018_Falkner, 2021_Turner}. As with molecules and materials, machine learning models may have positive and negative consequences for society. In this respect, we would again hope that our technology will not stimulate individuals to use their models for nefarious purposes, but rather at worst, will accelerate their ability to do so.
    \item \textbf{Military Applications:} Bayesian optimisation is also used in  robotics and sensor placement systems \cite{2016_Calandra, 2019_Grant} with use-cases for military drones and UAVs. In similar fashion to the previous applications, these technologies may be misused to incite warfare but may also be beneficial for defence and counter-terrorism purposes.
\end{enumerate}

Concern for unfavourable economic impacts of our research due to unemployment may arise in a number of domains \cite{2017_Frey}. Our methodology holds promise to expedite the automation of industrial processes such as mining, reaction optimisation and nuclear power generation, potentially resulting in the loss of jobs for mining professionals, engineers and technicians. This being said, it is important to balance the negative impacts of temporary unemployment against benefits due to climate change mitigation for example, an undoubtedly important long-term consideration for the global economy. Bayesian optimisation is already a core component in self-driving laboratories \cite{2020_macleod} created with the explicit goal of discovering renewable energy materials \cite{2009_MacKay} such as perovskite solar cells \cite{2018_Herbol}. As such, we would hope that over a long time horizon our contribution will be a net force for social good.

\section{Additional Background and Related Work on Deep Metric Learning}\label{App:additional_related_work}

In this section we discuss additional background and related work in deep metric learning. Due to space constraints in the main paper, we target our discussion there towards VAE-based deep metric learning. Here, we provide a short overview of the development of deep metric learning. 

The performance of many machine learning algorithms critically depends on the availability of an informative metric over the input space~\cite{MJ}. The definition of such a measure is far from trivial especially in high-dimensional domains where standard distances tend to convey sub-optimal notions of similarity. As such, the search for the ``right'' metric has gained considerable attention leading to the development of numerous algorithms which according to~\cite{Tut} can be categorised into dimensionality reduction based~\cite{PCA_Jolliffe, LDA_fisher, ANMM_Wang_Zhang, KLMNN_Weinberger_Saul, singh2007dimensionality}, nearest neighbor specific~\cite{LMNN_Wainberger_Saul, NCA_Goldberger, NCMML_NCMC_Mensink, KLMNN_Torresani_Lee}, and information theoretic techniques~\cite{ITML_Davis, DMLMJ_Nguyen, MCML_Globerson_Roweis, DML_eig_Ying_Li, KDA_Mika}. Most of those methods learn a form of a Mahalanobis distance~\cite{schweizer1960statistical} by employing a linear transformation of the input space and then optimising a task specific loss function, e.g., maximising class separation in linear discriminant analysis~\cite{LDA_fisher}, or minimising expected leave-one-out errors in neighborhood component analysis~\cite{NCA_Goldberger}.  

Although early works on metric learning concentrated on linear methods, such models have shown limited separation capability when applied to nonlinear structures like those considered in this paper~\cite{Zabour, Zaboura}. Amongst many works attempting to remedy those limitations, e.g., kernelisation~\cite{Zababeer, torresani2007large, he2014nonlinear}, and localisation~\cite{chen2007nonlinear, hong2011learning, Zaboura, weinberger2009distance}, in this paper, we focus on \emph{deep metric learning methods} which have shown outstanding performance in a variety of fields such as textual entailment classiﬁcation \cite{wu2020}, image retrieval \cite{kim2019}, and reinforcement learning~\cite{zhang2021}. While deep metric learning originally garnered acclaim in classification domains, there is an extensive literature focussed on extending deep metric learning to regression problems \cite{2019_Kim, 2020_Kim, 2019_Kaya}. In addition, there is much work on augmenting triplet and contrastive losses for new problem domains \cite{2017_Wang_angular, 2018_Ge}. As such, the incorporation of deep metric learning methodologies for Bayesian optimisation would appear to be timely.

\end{document}